%% file: tit2013.tex
\documentclass[12pt]{IEEEtran}
\onecolumn

\usepackage[utf8]{inputenc} 
\usepackage[T1]{fontenc}    
\usepackage{url}            
\usepackage{booktabs}       
\usepackage{amsfonts}       
\usepackage{nicefrac}       
\usepackage{microtype}      
\usepackage{xcolor}         

\usepackage{graphicx}
\usepackage{subfigure}
\usepackage{tabularx} 

\usepackage{algorithm,algorithmic}

\usepackage{epsfig,amsmath,amssymb,amstext,amsthm,mathrsfs}
\usepackage{latexsym,graphics,epsf,epsfig,psfrag}
\usepackage{dsfont,color,epstopdf,fixmath}
\usepackage{enumitem}

\usepackage{hhline}

\usepackage{threeparttable}
\usepackage{makecell}

\usepackage{amsfonts}       
\usepackage{microtype}      
\usepackage{subfigure}
\usepackage{mathtools}
\usepackage{amssymb}
\usepackage{amsthm}

\usepackage{cleveref}

\usepackage{multirow}

\newcommand{\hQ}{\widehat Q}
\newcommand{\hV}{\widehat V}
\newcommand{\hpi}{\widehat \pi}

\newcommand{\mcB}{\mathcal B}

\newcommand{\rTheta}{\rm \Theta}

\newcommand{\mP}{\mathbb P}
\newcommand{\mB}{\mathbb B}

\newcommand{\mcH}{\mathcal H}
\newcommand{\mcF}{\mathcal F}
\newcommand{\mcG}{\mathcal G}
\newcommand{\mcN}{\mathcal N}

\newcommand{\mE}{\mathbb E}

\newcommand{\mcs}{\mathcal S}
\newcommand{\mca}{\mathcal A}

\newcommand{\mcx}{\mathcal X}

\newcommand{\mcl}{\mathcal L}
\newcommand{\mcd}{\mathcal D}
\newcommand{\mf}{\mathcal F}
\newcommand{\mcv}{\mathcal V}

\newcommand{\ltwo}[1]{\left\|#1\right\|_2}
\newcommand{\lone}[1]{\left|#1\right|}

\newcommand{\lF}[1]{\left\|#1\right\|_F}

\newcommand{\lsigma}[1]{\left\|#1\right\|_{{\rm\Sigma}({\rm\Theta}_0)^{-1}}}
\newcommand{\llambda}[1]{\left\|#1\right\|_{{\rm\Lambda}({w}_0)^{-1}}}
\newcommand{\lomega}[1]{\left\|#1\right\|_{{\rm\Omega}_N^{-1}}}
\newcommand{\lH}[1]{\left\|#1\right\|_{\mathcal H}}
\newcommand{\lM}[1]{\left\|#1\right\|_{M}}

\newcommand{\lHK}[1]{\left\|#1\right\|_{\mathcal H_{K}}}
\newcommand{\lHKH}[1]{\left\|#1\right\|_{\mathcal H_{K_H}}}

\newcommand{\lsigmab}[1]{\left\|#1\right\|_{{\rm\Sigma}^{-1}}}
\newcommand{\lsigmabp}[1]{\left\|#1\right\|_{{\rm\Sigma}^{\prime-1}}}

\newcommand{\llambdab}[1]{\left\|#1\right\|_{{\rm\Lambda}^{-1}}}
\newcommand{\llambdabh}[1]{\left\|#1\right\|_{{\rm\Lambda}_h^{-1}}}
\newcommand{\llambdabp}[1]{\left\|#1\right\|_{{\rm\Lambda}^{\prime-1}}}

\newcommand{\lomegab}[1]{\left\|#1\right\|_{{\rm\Omega}}}
\newcommand{\lomegabp}[1]{\left\|#1\right\|_{{\rm\Omega}^{\prime}}}

\newcommand{\lupsilon}[1]{\left\|#1\right\|_{{\rm\Upsilon}}}
\newcommand{\lupsilonp}[1]{\left\|#1\right\|_{{\rm\Upsilon}^{\prime}}}

\DeclareMathOperator*{\argmin}{argmin}
\DeclareMathOperator*{\argmax}{argmax}

\newcommand{\mR}{\mathbb{R}}

\newtheorem{theorem}{Theorem}
\newtheorem{corollary}{Corollary}
\newtheorem{lemma}{Lemma}

\newtheorem{assumption}{Assumption}

\newtheorem{definition}{Definition}

\allowdisplaybreaks[4]

\title{Provably Efficient Offline Reinforcement Learning with Trajectory-Wise Reward}

\author{Tengyu Xu \quad Yue Wang \quad Shaofeng Zou\quad Yingbin Liang
	\thanks{ Tengyu Xu and Yingbin Liang are with the Department of Electrical and Computer Engineering, The Ohio State University (email: xu.3260@buckeyemail.osu.edu, liang.889@osu.edu);
 Yue Wang and Shaofeng Zou are with the Department of Electrical Engineering, University at Buffalo, Buffalo, NY 14228 USA (email: ywang294@buffalo.edu, szou3@buffalo.edu).
}
}

\begin{document}

\maketitle
\begin{abstract}
  The remarkable success of reinforcement learning (RL) heavily relies on observing the reward of every visited state-action pair. In many real world applications, however, an agent can observe only a score that represents the quality of the whole trajectory, which is referred to as the {\em trajectory-wise reward}. In such a situation, it is difficult for standard RL methods to well utilize trajectory-wise reward, and large bias and variance errors can be incurred in policy evaluation. In this work, we propose a novel offline RL algorithm, called Pessimistic vAlue iteRaTion with rEward Decomposition (PARTED), which decomposes the trajectory return into per-step proxy rewards via least-squares-based reward redistribution, and then performs pessimistic value iteration based on the learned proxy reward. To ensure the value functions constructed by PARTED are always pessimistic with respect to the optimal ones, we design a new penalty term to offset the uncertainty of the proxy reward. For general episodic MDPs with large state space, we show that PARTED with overparameterized neural network function approximation achieves an $\tilde{\mathcal{O}}(D_{\text{eff}}H^2/\sqrt{N})$ suboptimality, where $H$ is the length of episode, $N$ is the total number of samples, and $D_{\text{eff}}$ is the effective dimension of the neural tangent kernel matrix. To further illustrate the result, we show that PARTED achieves an $\tilde{\mathcal{O}}(dH^3/\sqrt{N})$ suboptimality with linear MDPs, where $d$ is the feature dimension, which matches with that with neural network function approximation, when $D_{\text{eff}}=dH$. To the best of our knowledge, PARTED is the first offline RL algorithm that is provably efficient in general MDP with trajectory-wise reward.
\end{abstract}

\input{introduction}

\input{preliminary}
\input{algorithm}
\input{results}

\section{Conclusion}
In this paper, we propose a novel offline RL algorithm, called PARTED, to handle the episodic RL problem with trajectory-wise rewards. PARTED uses a least-square-based reward redistribution method for reward estimation and incorporates a new penalty term to offset the uncertainty of proxy reward. Under the neural network function approximation, we prove that PARTED achieves an $\tilde{\mathcal{O}}(D_{\text{eff}}H^2/\sqrt{N})$ suboptimality, which matches the order $\tilde{\mathcal{O}}(dH^3/\sqrt{N})$ of linear MDP (that we further establish) when the effective dimension satisfies $D_{\text{eff}} = dH$. To the best of our knowledge, this is the first offline RL algorithm that is provably efficient in general episodic MDP setting with trajectory-wise rewards. As a future direction, it is interesting to incorporate the randomized return decomposition in \cite{ren2021learning} to improve the scalability of PARTED in the long horizon scenario.

\appendices

\input{appendix}

\bibliography{iclr2023_conference}
\bibliographystyle{ieeetr}

\end{document}

%% file: introduction.tex
\section{Introduction}
Reinforcement learning (RL) aims at searching for an optimal policy in an unknown environment \cite{sutton2018reinforcement}. To achieve this goal, an instantaneous reward is typically required at every step so that RL algorithms can maximize the cumulative reward of a Markov Decision Process (MDP). In recent years, RL has achieved remarkable empirical success with a high quality reward function \cite{mnih2015human,levine2016end,silver2017mastering,senior2020improved,maei2011gradient}. However, in many real-world scenarios, instantaneous rewards are hard or impossible to be obtained. For example, in the autonomous driving task \cite{shalev2016safe}, it is very costly and time consuming to score every state-action pair that the agent (car) visits. 
In contrast, it is fairly easy to score the entire trajectory after the agent completing the task \cite{chatterji2021theory}. Therefore, in practice, it becomes more reasonable to adopt trajectory-wise reward schemes, in which only a return signal that represents the quality of the entire trajectory is revealed to the agent in the end. In recent years, trajectory-wise rewards have become prevalent in many real-world applications \cite{gong2019decentralized,olivecrona2017molecular,lin2018efficient,hein2017benchmark,rahmandad2009effects}.

Although trajectory-wise rewards are convenient to be obtained, it is often challenging for standard RL algorithms to utilize such a type of rewards well due to the high bias and variance it can introduce in the policy evaluation process \cite{arjona2019rudder}, which leads to unsatisfactory policy optimization results. To address such an issue, \cite{chatterji2021theory,pacchiano2021dueling} proposed to encode the whole trajectory and search for a non-Markovian trajectory-dependent optimal policy using the contextual bandit method. Although this type of approaches have promising theoretical guarantees, they are difficult to be implemented in practice due to the difficulty of searching the large trajectory-dependent policy space whose dimension increases exponentially with the horizon length. Another type of approaches widely adopted in practice is called {\em reward redistribution}, which learns a reward function by allocating the trajectory-wise reward to every visited state-action pairs based on their contributions \cite{arjona2019rudder,liu2019sequence,gangwani2020learning,ren2021learning,efroni2021reinforcement}. Since the reward function in reward redistribution is typically learned via solving a supervised learning problem, such an approach is sample-efficient and can be integrated into the existing RL frameworks easily. However, most of existing reward redistribution approaches do not have theoretical performance guarantee. So far, only \cite{efroni2021reinforcement} proposes a provably efficient reward redistribution algorithm, but is only applicable to tabular episodic MDP and requires both reward and transition kernel to be horizon-independent. 

Despite the superior performance of the reward redistribution method, all previous algorithms considered only the {\em online} setting, which are not applicable to many critical domains where offline sampling is preferred (or can be required), as interactive data collection could be very costly and risky \cite{shalev2016safe,gottesman2019guidelines}. 
{\em How to design reward redistribution in offline RL for trajectory-wise rewards is an important but fully unexplored problem.} For such a problem, designing reward redistribution algorithms can be hard due to the insufficient sample coverage issue \cite{wang2020statistical} in offline RL. Further challenges can be encountered when we try to design {\em provably} efficient reward distribution algorithms for general MDPs with large state space and horizon-dependent rewards and transition kernels, which has not been studied in online setting.

The goal of this work is to design an offline RL algorithm with reward redistribution for trajectory-wise rewards, which has provable efficiency guarantee for general episodic MDPs.

\subsection{Main Contributions}
In this paper, we consider episodic MDP with possibly infinity state space and horizon-dependent reward function and transition kernel.
The trajectory-wise reward adopts a standard sum-form as considered previously in \cite{han2021off,zheng2018learning,klissarov2020reward,oh2018self,ren2021learning,efroni2021reinforcement}, in which only the summation of rewards over the visited state-action pairs is revealed at the end of each episode.  

We propose a novel Pessimistic vAlue iteRaTion with rEward Decomposition (PARTED) algorithm for offline RL with trajectory-wise rewards, which incorporates a least-square-based reward redistribution into the pessimistic value iteration (PEVI) algorithm \cite{jin2021pessimism,yin2021near,yin2020near,yin2022near,yin2021towards}. Differently from the standard PEVI with instantaneous reward, in which reward and value function can be learned together by solving a single regression problem, in PARTED, reward need to be learned separately from the value function by training a regression model to decompose the trajectory return into per-step proxy rewards. In order to capture the reward and value function for a large state space, we adopt overparameterized neural networks for  function approximation. Moreover, to offset the estimation error of proxy rewards, we design a penalty function by transfering the uncertainty from the covariance matrix of trajectory features to step-wise proxy rewards via an "one-block-hot" vector, which is new in the literature. 

We show that our proposed new penality term ensures that the value functions constructed by PARTED are always pessimistic with respect to the optimal ones. Furthermore, with overparameterized neural network function approximation, we show that PARTED achieves an $\tilde{\mathcal{O}}(D_{\text{eff}}H^2/\sqrt{N})$ suboptimality, where $H$ is the length of episode, $N$ is the total number of samples, and $D_{\text{eff}}$ is the effective dimension of neural tangent kernel matrix. To further illustrate our result, we show that PARTED achieves an $\tilde{\mathcal{O}}(dH^3/\sqrt{N})$ suboptimality in the linear MDP setting, where $d$ is the feature dimension, which matches that in the neural network function approximation setting when $D_{\text{eff}} = dH$. To the best of our knowledge, PARTED is the first-known offline RL algorithm that is provably efficient in general episodic MDPs with trajectory-wise rewards.

\subsection{Related Works}

\textbf{Trajectory-Wise Reward RL. } Policy optimization with trajectory-rewards is extremely difficult. A variety of practical strategies have been proposed to resolve this technical challenge by redistributing trajectory rewards to step-wise rewards. 
RUDDER \cite{arjona2019rudder} trains a return predictor of state-action sequence with LSTM \cite{hochreiter1997long}, and the reward at each horizon is then assigned by the difference between the predications of two adjacent sub-trajectories. Later, \cite{liu2019sequence} improves RUDDER and utilizes a Transformer \cite{vaswani2017attention} for better reward learning. IRCR \cite{gangwani2020learning} assigns the proxy reward of a state-action pair as the normalized value of trajectory returns that contain the correspondingly state-action pair. RRD \cite{ren2021learning} learns a proxy reward function by solving a supervised learning problem together with a Monte-Carlo sampling strategy. 
Although those methods have achieved great empirical success, they all lack overall theoretical performance guarantee.

Differently from empirical studies, existing theoretical works of trajectory-wise reward RL are rare and focus only on the online setting. One line of research assumes trajectory reward being non-Markovian, and thus focuses on searching for a non-Markovian, trajectory-dependent optimal policy. 
\cite{chatterji2021theory} assumes that trajectory-wise reward is a binary signal generated by a logistic classifier with trajectory embedding as the input. In this setting, the policy optimization problem is reduced to a linear contextual bandit problem in which the trajectory embedding is the contextual vector. \cite{pacchiano2021dueling} considers a similar setting as \cite{chatterji2021theory} but assumes only having access to a binary preference score between two trajectories instead of an absolute reward of a trajectory. Another line of research assumes that the trajectory-wise reward is the summation of underlying step-wise Markovian rewards. The goal of this line of work is to search for an optimal Markovian policy.
\cite{cohen2021online} adopted a mirror descent approach so that the summation of rewards alone is sufficient to perform the policy optimization. This approach relies on the on-policy unbiased sampling of trajectory rewards, and can hardly be extended to the offline setting.
\cite{efroni2021reinforcement} proposed to recover the reward by solving a least-squared regression problem that fits the summation of reward estimation toward the trajectory reward. 

To our best knowledge, offline RL with trajectory-wise rewards (where no interaction with the environment is allowed) has not been studied before, and our work develops the first-known algorithm for such a setting with provable sample efficiency guarantee. Further, although our reward redistribution approach applies the least-square based method, which has also been adopted in \cite{efroni2021reinforcement}, our algorithm is designed for general MDPs with possibly infinite state and horizon-dependent rewards and transition kernels, which is very different from that in \cite{efroni2021reinforcement} designed for tabular MDPs with time-independent rewards and transition kernels.

\textbf{Offline RL. }The major challenge in offline RL is the insufficient sample coverage in the pre-collected dataset, which arises from the lack of exploration \cite{wang2020statistical,liu2020provably}. To address such a challenge, two types of algorithms have been studied: (1) regularized approaches, which prevent the policy from visiting states and actions that are less covered by the dataset \cite{dadashi2021offline,fujimoto2019off,fujimoto2019benchmarking,wang2020critic,fujimoto2021minimalist}; (2) pessimistic approach, which penalize the estimated values of the less-covered state-action pairs \cite{buckman2020importance, kumar2020conservative}. So far, a number of provably efficient pessimistic offline RL algorithm have been proposed in both tabular MDP setting \cite{yin2020near,shi2022pessimistic,yan2022efficacy,li2022settling,yin2021optimal,ren2021nearly,xie2021policy,yin2021towards,rashidinejad2021bridging} and linear MDP setting \cite{jin2021pessimism,xie2021bellman,zanette2021provable,wang2020statistical,zanette2021exponential,foster2021offline,yin2022near}.
However, the efficiency of all those works relies on both the availability of instantaneous reward and special structures of MDP, which can hardly be satisfied in practical settings. In this work, we take a first step towards relaxing those two assumptions by proposing PARTED, which is provably efficient in general episodic MDPs with trajectory-wise rewards. 


%% file: preliminary.tex
\section{Preliminary and Problem Formulation}

\subsection{Episodic Markov Decision Process}\label{sc: MDP}

An episodic Markov decision process (MDP) is defined by a tuple $(\mcs, \mca, \mP,r,H)$, where $\mcs$ and $\mca$ are the state and action spaces, $H>0$ is the length of each episode, and
$\mP=\{\mP_h\}_{h\in[H]}$ and $r=\{r_h\}_{h\in[H]}$ are the transition kernel and reward, respectively, where $[n]=\{1,2,\cdots,n\}$ for integer $n\geq 1$. 
We assume $\mcs$ is a measurable space of possibly infinite cardinality and $\mca$ is a finite set. For each $h\in[H]$, $\mP_h(\cdot|s,a)$ denotes the transition probability when action $a$ is taken at state $s$ at timestep $h$, and $r_h(s,a) \in [0,1]$ is a {\bf random} reward that is observed with state-action pair $(s,a)$ at timestep $h$. We denote the mean of the reward as $R_h(s,a)=\mE[r_h(s,a)|s,a]$ for all $(s,a)\in\mcs\times\mca$. For any policy $\pi=\{\pi_h\}_{h\in[H]}$, we define the state value function $V^\pi_h(\cdot):\mcs\rightarrow \mR$ and state-action value function $Q^\pi_h(\cdot):\mcs\times\mca \rightarrow \mR$ at each timestep $h$ as
\begin{flalign*}
	V^\pi_h(s)=\mE_\pi\left[\sum_{t=h}^{H}r_t(s_t,a_t)\Bigg| s_h=s\right],\quad Q^\pi_h(s,a)=\mE_\pi\left[\sum_{t=h}^{H}r_t(s_t,a_t)\Bigg| (s_h,a_h)=(s,a)\right],
\end{flalign*}
where the expectation $\mE_\pi$ is taken with respect to the randomness of the trajectory induced by policy $\pi$, which is obtained by taking action $a_t\sim\pi_t(\cdot|s_t)$ and transiting to the next state $s_{t+1}\sim\mP_t(\cdot|s_t,a_t)$ at timestep $t\in[H]$. At each timestep $h\in[H]$, for any $f: \mcs\rightarrow \mR$, we define the transition operator as $(\mP_hf)(s,a) = \mE\left[f(s_{h+1})|(s_h,a_h)=(s,a)\right]$ and the Bellman operator as $(\mB_hf)(s,a)=R_h(s,a)+(\mP_hf)(s,a)$. For episodic MDP $(\mcs, \mca, \mP,r,H)$, we have
\begin{flalign*}
	Q^\pi_h(s,a)=(\mB_h V^\pi_{h+1})(s,a),\quad V^\pi_h(s)=\langle Q^\pi_h(s,\cdot),\pi_h(\cdot|s)\rangle_\mca, \quad V^\pi_{H+1}(s) = 0,
\end{flalign*}
where $\langle\cdot,\cdot\rangle_\mca$ denotes the inner product over $\mca$. We define the optimal policy $\pi^*$ as the policy that yields the optimal value function, i.e., $V^{\pi^*}_h(s)=\sup_{\pi}V^\pi_h(s)$  for all $s\in\mcs$ and $h\in[H]$. For simplicity, we denote $V^{\pi^*}_h$ and $Q^{\pi^*}_h$ as $V^*_h$ and $Q^*_h$, respectively. The Bellman optimality equation is given as follows
\begin{flalign}\label{eq: 1}
Q^*_h(s,a)=(\mB_h V^*_h)(s,a),\quad V^*_h(s)=\argmax_{a\in\mca} Q^\pi_h(s,\cdot), \quad V^*_{H+1}(s) = 0,
\end{flalign}
The goal of reinforcement learning is to learn the optimal policy $\pi^*$. For any fixed $\pi$, we define the performance metric as 
\begin{flalign*}
	\text{SubOpt}(\pi,s)=V^*_1(s) - V^\pi_1(s),
\end{flalign*}
which is the suboptimality of the policy $\pi$ given the initial state $s_1=s$. 

\subsection{Trajectory-Wise Reward and Offline RL}
In the trajectory-wise reward setting, the transition of the environment is still Markovian and the agent can still observe and interact with the environment instantly as in standard MDPs. However, unlike standard MDPs in which the agent can receive an instantaneous reward $r_h(s,a)$ for every visited state-action pair $x$ at each timestep $h$, in the trajectory-wise reward setting, only a reward that is associated with the whole trajectory can be observed at the end of the episode, i.e., $r(\tau)$ where $\tau=\{(s^\tau_1,a^\tau_1),\cdots,(s^\tau_H,a^\tau_H)\}$ denotes a trajectory and $(s^\tau_h,a^\tau_h)$ is the $h$-th state-action pair in trajectory $\tau$, which is called "trajectory reward" in the sequel. In this work, we consider the setting in which the trajectory reward is the summation of the underlying instantaneous reward in the trajectory of MDP $(\mcs, \mca, \mP,r,H)$, i.e., $r(\tau)=\sum_{h=1}^Hr_h(s^\tau_h,a^\tau_h)$. We denote the mean of the trajectory reward as $R(\tau)=\mE[r(\tau)|\tau] = \sum_{h=1}^H R_h(s^\tau_h,a^\tau_h)$. Such a sum-form reward has been commonly considered in both theoretical \cite{efroni2021reinforcement} and empirical studies \cite{han2021off,zheng2018learning,klissarov2020reward,oh2018self,ren2021learning}. It models the situations where the agent's goal is captured by a certain metric with additive properties, e.g., the energy cost of a car during driving, the click rate of advertisements during a time interval, or the distance of a robot's running. Such a form of reward can be more general than the standard RL feedback and is expected to be more common in practical scenarios.
Note that RL problems under trajectory-wise rewards is very challenging,
as traditional policy optimization approach typically fails due the obscured feedback received from the environment, which causes large value function evaluation error \cite{han2021off}.

We consider the offline RL setting, in which a learner has access only to a pre-collected dataset $\mcd$ consisting of $N$ trajectories $\{\tau_i, r(\tau_i)\}_{i,h=1}^{N,H}$ rolled out from some possibly unknown behavior policy $\mu$, where $\tau_i$ and $r(\tau_i)$ are the $i$-th trajectory and the observed trajectory reward of $\tau_i$, respectively. Given this batch data $\mcd$ with only trajectory-wise rewards and a target accuracy $\epsilon$, our goal is to find a policy $\pi$ such that $\text{SubOpt}(\pi,s)\leq \epsilon$ for all $s\in\mcs$.

\subsection{Overparameterized Neural Network}\label{subsc: overpnn}
In this paper, we consider the function approximation setting, in which the state-action value function is approximated by a two-layer neural network. To simplify the notation, we denote $\mcx=\mcs\times\mca$ and view it as a subset of $\mR^d$. We further assign a feature vector $x\in\mcx$ to represent a state-action pair $(s,a)$. Without loss of generality, we assume that $\ltwo{x}=1$ for all $x\in\mcx$. We also allow $x=0$ to represent a null state-action pair. 
We now define a two-layer neural network $f(\cdot,b,w):\mcx\rightarrow\mR$ with $2m$ neurons and weights $(b,w)$ as
\begin{flalign}\label{eq: 2}
	f(x;b,w)=\frac{1}{\sqrt{2m}}\sum_{r=1}^{2m}b_j\cdot\sigma(w_r^\top x),\qquad \forall x\in\mcx,
\end{flalign}
where $\sigma(\cdot):\mR\rightarrow\mR$ is the activation function, $b_r\in\mR$ and $w_r\in\mR^d$ for all $r\in[2m]$, and $b=(b_1,\cdots,b_{2m})^\top\in\mR^{2m}$ and $w=(w^\top_1,\cdots,w^\top_{2m})^\top\in\mR^{2md}$. We make the following assumption for $\sigma(\cdot)$, which can be satisfied by a number of activation functions such as ReLU and $\tanh(\cdot)$.
\begin{assumption}\label{ass}
	For all $x\in\mcx$, we have $\lone{\sigma^\prime(x)}\leq C_\sigma<+\infty$ and $\sigma^\prime(0)=0$.
\end{assumption}
\vspace{-0.2cm}
We initialize $b$ and $w$ via a symmetric initialization scheme \cite{gao2019convergence,bai2019beyond}: for any $1\leq r\leq m$ we set $b_{0,r}\sim \text{Unif}(\{-1,1\})$ and $w_{0,r}\sim N(0,I_d/d)$, where $I_d$ is the identity matrix in $\mR^d$, and for any $m+1\leq r\leq 2m$, we set $b_{0,r}=-b_{0,r-m}$ and $w_{0,r}=w_{0,r-m}$. Under such an initialization, the initial neural network is a zero function, i.e. $f(x;b_0,w_0)=0$ for all $x\in\mcx$, where $b_0=[b_{0,1},\cdots,b_{0,2m}]^\top$ and $w_0=[w_{0,1}^\top,\cdots,w_{0,2m}^\top]^\top$ are initialization parameters. During training, we fix the value of $b$ at its initial value and only optimize $w$. To simplify the notation, we denote $f(x;b,w)$ as $f(x;w)$ and $\nabla_wf(x,w)$ as $\phi(x,w)$. 

{\bf Notations.} We use $\widetilde{\mathcal{O}}(X)$ to refer to a quality that is upper bounded by $X$, up to poly-log factors of $d, H, N, m$ and $(1/\delta)$. Furthermore, we use $\mathcal{O}(X)$ to refer to a quantity that is upper bounded by $X$ up to constant multiplicative factors. We use $I_d$ as the identity matrix in dimension $d$. Similarly, we denote by $\mathbf{0}_d\in\mR^d$ as the vector whose components are zeros. For any square matrix $M$, we let $\ltwo{M}$ denote the operator norm of $M$. Finally, for any positive definite matrix $M\in\mR^{d\times d}$ and any vector $x\in\mR^d$, we define $\lM{x} = \sqrt{x^\top M x} $.

%% file: algorithm.tex
\section{Algorithm}\label{sc: alg}
In this section, we propose a Pessimistic vAlue iteRaTion with rEward Decomposition (PARTED) algorithm based on the neural network function approximation. PARTED shares a similar structure as that of pessimistic value iteration (PEVI) \cite{jin2021pessimism,xie2021policy,yin2020near}, but has a very different design due to trajectory-wise rewards. In PEVI, a pessimistic estimator of the value function is constructed from the dataset $\mcd$ and the Bellman optimality equation is then iterated based such an estimator. Since instantaneous rewards are available in PEVI, given a function class $\mcG$, PEVI constructs an estimated Bellman backup of value function $(\widehat{\mB}_h\widehat{V}_{h+1})$ by solving the following regression problem for all $h\in[H]$ in the backward direction:
\begin{flalign}
	 (\widehat{\mB}_h\widehat{V}_{h+1}) = \argmin_{g_h\in\mcG} L^h_{\text{PEVI}}(g_h) =  \sum_{\tau\in\mcd}\left(r_h(x^\tau_h) + \widehat{V}_{h+1}(s^\tau_{h+1}) - g_h(x^\tau_h)\right)^2+\lambda \cdot \text{Reg}(g_h).\label{eq: 200}
\end{flalign}
In \cref{eq: 200}, $\widehat{V}_{h+1}(\cdot)$ is the pessimistic estimator of optimal value function constructed for horizon $h+1$, $\lambda>0$ is a regularization parameter and $\text{Reg}(\cdot)$ is the regularization function. The optimal state-action value function can then be estimated as $\widehat{Q}_h(\cdot) = \min\{(\widehat{\mB}_h\widehat{V}_{h+1})(\cdot)-{\rm\Gamma}_h(\cdot),H \}^+$, where $-{\rm\Gamma}_h$ is a negative penalty used to offset the uncertainty in $(\widehat{\mB}_h\widehat{V}_{h+1})(\cdot)$ and guarantee the pessimism of $\widehat{Q}_h$.

However, in PARTED (see \Cref{alg1}) designed for trajectory-wise rewards, since instantaneous reward $r_h(\cdot)$ is not available, we can no longer obtain $(\widehat{\mB}_h\widehat{V}_{h+1})$ in the same way as PEVI by solving the regression problem in \cref{eq: 200}. To overcome such an issue, in PARTED, we construct two estimators $\widehat{r}_h$ and $(\widehat{\mP}_h\widehat{V}_{h+1})$ for instantaneous reward ${r}_h$ and transition value function $({\mP}_h\widehat{V}_{h+1})$, respectively. The estimated Bellman backup can then be formulated as $(\widehat{\mB}_h\widehat{V}_{h+1})(\cdot) = \widehat{r}_h(\cdot) + (\widehat{\mP}_h\widehat{V}_{h+1})(\cdot)$.

{\bf Reward Redistribution. }In order to estimate the instantaneous rewards from the trajectory-wise reward, we use a neural network $f(\cdot,\theta_h)$ given in \cref{eq: 2} to represent per-step mean reward $R_h(\cdot)$ for all $h\in[H]$, where $\theta_h\in\mR^{2md}$ is the parameter. We further assume, for simplicity, that all the neural networks share the same initial weights denoted by $\theta_0\in\mR^{2md}$. We define the following loss function $L_r(\cdot):\mR^{2mdH}\rightarrow\mR$ for reward redistribution as
\begin{flalign}\label{eq: 7}
	\textstyle L_r({\rm\Theta})=\sum_{\tau\in\mcd}\left[ \sum_{h=1}^{H}f(x^\tau_h,\theta_h) - r(\tau) \right]^2 + \lambda_1\cdot\sum_{h=1}^{H}\ltwo{\theta_h-\theta_0}^2,
\end{flalign}
where ${\rm\Theta}=[\theta_1^\top,\cdots,\theta_H^\top]^\top\in\mR^{2mdH}$ and $\lambda_1>0$ is a regularization parameter. Then, the per-step proxy reward $\widehat{R}_h(\cdot)$ is obtained by solving the following optimization problem
\begin{flalign}\label{eq: 3}
	\widehat{R}_h(\cdot) = f(\cdot,\widehat{\theta}_h), \quad \text{where}\quad \widehat{\rm\Theta}=\argmin_{{\rm\Theta}\in\mR^{2mdH}}L_r( {\rm\Theta} )\,\,\,\text{and}\,\,\, \widehat{\rm\Theta}=[\widehat{\theta}_1^\top,\cdots,\widehat{\theta}_H^\top]^\top.
\end{flalign}

{\bf Transition Value Function Estimation. }Similarly, we use $H$ neural networks  given in \cref{eq: 2} with parameter $\{w_h\}_{h\in[H]}$ to estimate $\{(\mP_h\widehat{V}_{h+1})(\cdot)\}_{h\in[H]}$, where $w_h\in\mR^{2md}$ is the parameter of the $h$-th network. Specifically, for each $h\in[H]$, we define the loss function $L^h_v(w_h)$: $\mR^{2md}\rightarrow\mR$ as
\begin{flalign}
	\textstyle L^h_v(w_h) = \sum_{\tau\in\mcd}\left(\widehat{V}_{h+1}(s^\tau_{h+1}) - f(x^\tau_h,w_h)\right)^2+\lambda_2\cdot \ltwo{w_h-w_0}^2,\label{eq: 74}
\end{flalign}
where $\lambda_2>0$ is a regularization parameter and $w_0$ is the initialization shared by all neural networks. The estimated transition value function $(\widehat{\mP}_h\widehat{V}_{h+1})(\cdot): \mcx\rightarrow\mR$ can be obtained by solving the following optimization problem
\begin{flalign}\label{eq: 4}
	\textstyle (\widehat{\mP}_h\widehat{V}_{h+1})(\cdot) = f(\cdot,\widehat{w}_h),\quad\text{where}\quad \widehat{w}_h=\argmin_{w_h\in\mR^{2md}}L^h_v(w_h).
\end{flalign}

\begin{algorithm}[tb]
	\caption{Neural Pessimistic Value Iteration with Reward Decomposition (PARTED)}
	\label{alg1}
	\begin{algorithmic}
		\STATE {\bfseries Input:} Dataset $\mcd=\{\tau_i,r(\tau_i)\}_{i,h=1}^{N,H}$
		\STATE {\bfseries Initialization:} Set $\widehat{V}_{H+1}$ as zero function
		\STATE Obtain $\widehat{R}_h$ and $\widehat{\rm\Theta}$ according to \cref{eq: 3}
		\FOR{$h=H,H-1,\cdot,1$}
		\STATE Obtain $\widehat{\mP}_h\widehat{V}_{h+1}$ and $\widehat{w}_h$ according to \cref{eq: 4}
		\STATE Obtain ${\rm\Gamma}_h(\cdot,\widehat{\rm\Theta}, \widehat{w}_h )$ according to \cref{eq: 5}
		\STATE Obtain $\widehat{Q}_h(\cdot)$ and $\widehat{V}_h(\cdot)$ according to \cref{eq: 198} and let $\hpi_h(\cdot|s)=\argmax_{\pi_h}\langle \hQ_h(s,\cdot), \pi_h(\cdot|s) \rangle$
		\ENDFOR
	\end{algorithmic}
\end{algorithm}
{\bf Penality Term Construction. }It remains to construct the penalty term $\rm{\Gamma}_h$ to offset the uncertainties in $\widehat{R}_h$ and $(\widehat{\mP}_h V_{h+1})$. First consider the penalty of $\widehat{R}_h(\cdot)$ for each $h\in[H]$. For any $\tau \in\mcd$ and ${\rm\Theta}\in\mR^{2mdH}$, we define a trajectory feature ${\rm\Phi}(\tau,{\rm\Theta})=[\phi(x^\tau_1,\theta_1)^\top,\cdots,\phi(x^\tau_H,\theta_H)^\top]^\top$.
Based on ${\rm\Phi}(\tau,{\rm\Theta})$, the trajectory feature covariance matrix ${\rm\Sigma}({\rm\Theta})\in\mR^{2mdH\times2mdH}$ is then defined as
\begin{flalign*}
	\textstyle{\rm\Sigma}({\rm\Theta}) = \lambda_1\cdot I_{2mdH} + \sum_{\tau\in\mcd}{\rm\Phi}(\tau,{\rm\Theta}){\rm\Phi}(\tau,{\rm\Theta})^\top.
\end{flalign*}
We also define an "one-block-hot" vector  ${\rm\Phi}_h(x,{\rm\Theta})=[\mathbf{0}^\top_{2md},\cdots,\phi(x,\theta_h)^\top,\cdots,\mathbf{0}^\top_{2md}]^\top$ for all $x\in\mcx$, where ${\rm\Phi}_h(x,{\rm\Theta})\in\mR^{2mdH}$ is a vector in which $[{\rm\Phi}_h(x,{\rm\Theta})]_{2md(h-1)+1:2mdh}=\phi(x,\theta_h)$ and the rest entries are zero. 
The penalty term of reward for a given ${\rm\Theta}\in\mR^{2mdH}$ is defined as:
\begin{flalign}\label{eq: 10}
	b_{r,h}(x,{\rm\Theta})=\left[{\rm\Phi}_h(x,{\rm\Theta})^\top {\rm\Sigma}^{-1}({\rm\Theta}) {\rm\Phi}_h(x,{\rm\Theta})\right]^{1/2},\quad\forall x\in\mcx.
\end{flalign}
Note that the reward penalty term $b_{r,h}(x,{\rm\Theta})$ is new and first proposed in this work. By constructing $b_{r,h}(x,{\rm\Theta})$ in this way, we can capture the effect of uncertainty caused by solving the trajectory-wise regression problem in \cref{eq: 7},  which is contained in the covariance matrix ${\rm\Sigma}({\rm\Theta})$, on the proxy reward $f(\cdot,\widehat{\theta}_h)$ at each step $h\in[H]$, via the "one-block-hot" vector ${\rm\Phi}_h(\cdot,{\rm\Theta})$.

Next, we consider the penalty of $(\widehat{\mP}_h \widehat{V}_{h+1})(\cdot)$ for each $h\in[H]$. We define the per-step feature covariance matrix ${\rm\Lambda}_h(w_h)\in\mR^{2md\times2md}$ as
\begin{flalign*}
	{\rm\Lambda}_h(w) = \lambda_2\cdot I_{2md} + \sum_{\tau\in\mcd}\phi(x^\tau_h,w)\phi(x^\tau_h,w)^\top.
\end{flalign*}
Then, the penality term of $(\widehat{\mP}_h \widehat{V}_{h+1})(\cdot)$ for a given $w\in\mR^{2md}$ is defined as
\begin{flalign}\label{eq: 11}
	b_{v,h}(x,w)=\left[\phi(x,w)^\top {\rm\Lambda}_h(w)^{-1} \phi(x,w)^\top \right]^{1/2},\quad\forall x\in\mcx.
\end{flalign}
Finally, combining \cref{eq: 10,eq: 11}, the penalty term for $\widehat{\mB}_h\widehat{V}_{h+1}(\cdot)$ is constructed as
\begin{flalign}\label{eq: 5}
	\textstyle{\rm\Gamma}_h(x,{\rm\Theta},w) = \beta_1 b_{r,h}(x,{\rm\Theta}) + \beta_2 b_{v,h}(x,w),
\end{flalign}
where $\beta_1,\beta_2>0$ are parameters. The estimator of $Q_h(\cdot)$ and $V_h(\cdot)$ can then be obtained as
\begin{flalign}
	\textstyle\widehat{Q}_h(\cdot) = \min\{ \widehat{R}_h(\cdot)+ (\widehat{\mP}_h\widehat{V}_{h+1})(\cdot) - {\rm\Gamma}_h(\cdot,\widehat{\rm\Theta}, \widehat{w}_h ), H \}^{+},\quad \widehat{V}_h(\cdot) =  \argmax_{a\in\mca}\hQ_h(\cdot,\cdot).\label{eq: 198}
\end{flalign}
Furthermore, for any $h\in[H]$, we denote $\mcv_h(x,R_{\beta_1}, R_{\beta_2}, \lambda_1, \lambda_2 )$ as the class of functions that takes the form $\overline{V}_h(\cdot)= \max_{a\in\mca}\overline{Q}_h(\cdot,a)$, where
\begin{flalign*}
\textstyle\overline{Q}_h(x )=& \min\{ \langle \phi(x,\theta_0), \theta - \theta_0 \rangle+  \langle \phi(x,w_0), w- w_0 \rangle \nonumber\\
&\textstyle- \beta_1\cdot \sqrt{{\rm\Phi}_h(x,\theta_0)^\top {\rm\Sigma}^{-1}{\rm\Phi}_h(x,\theta_0) } - \beta_2\cdot \sqrt{\phi(x,w_0)^\top{\rm\Lambda}^{-1}\phi(x,w_0)  } , H\}^{+},
\end{flalign*}
in which $\ltwo{\theta-\theta_0}\leq H\sqrt{N/\lambda_1}$, $\ltwo{w-w_0}\leq H\sqrt{N/\lambda_2}$, $\beta_1\in[0, R_{\beta_1}]$, $\beta_2\in [0, R_{\beta_2}]$, $\ltwo{\rm\Sigma}\geq \lambda_1$ and $\ltwo{\rm\Lambda}\geq \lambda_2 $. To this end, for any $\epsilon>0$, we define $\mcN_{\epsilon,h}^v$ as the $\epsilon-$covering number of $\mcv_h(x,R_{\beta_1}, R_{\beta_2}, \lambda_1, \lambda_2 )$ with respect to the $\ell_\infty-$norm on $\mcx$, and we let $\mcN_{\epsilon}^v=\max_{h\in[H]}\{\mcN_{\epsilon,h}^v\}$.

%% file: results.tex
\section{Main Results}\label{sc: main}

\subsection{Suboptimality of PARTED for General MDPs}\label{sc: subop}

In the overparameterized scheme, the neural network width $2m$ is considered to be much larger than the number of trajectories $N$ and horizon length $H$. Under such a scheme, the training process of neural networks can be captured by the framework of neural tangent kernel (NTK) \cite{jacot2018neural}. Specifically, conditioning on the realization of $w_0$, we define a kernel $K(x,x^\prime):\mcx\times\mcx\rightarrow\mR$ as
\begin{flalign*}
\textstyle K(x,x^\prime)=\langle \phi(x,w_0),\phi(x^\prime,w_0)\rangle= \frac{1}{2m}\sum_{r=1}^{2m} \sigma^\prime(w^{\top}_{0,r}x)\sigma^\prime(w^{\top}_{0,r}x^\prime)x^\top x^\prime, \quad \forall (x,x^\prime)\in \mcx\times\mcx,
\end{flalign*}
where $\sigma^\prime(\cdot)$ is the derivative of the action function $\sigma(\cdot)$. It can be shown that $f(\cdot,w)$ is close to its linearization at $w_0$ when $m$ is sufficiently large and $w$ is not too far away from $w_0$, i.e.,
\begin{flalign*}
\textstyle f(x,w)\approx f_0(x,w) = f(x,w_0) + \langle \phi(x,w_0), w-w_0\rangle = \langle \phi(x,w_0), w-w_0\rangle,\quad \forall x\in\mcx.
\end{flalign*}
Note that $f_0(x,w)$ belongs to a reproducing kernel Hilbert space (RKHS) with kernel $K(\cdot,\cdot)$. 
Similarly, consider the sum of $H$ neural networks $f(\tau,{\rm\Theta}) = \sum_{h=1}^{H}f(x^\tau_h, \theta_h)$ with the same initialization $\theta_0$ for each neural network, where $\tau=[x_1^\top,\cdots,x_H^\top]^\top$ and ${\rm\Theta}=[\theta_1^\top,\cdots,\theta_H^\top]^\top$. If $\theta_h$ is not too far away from $\theta_0$ for all $h\in[H]$ and $m$ is sufficiently large, it can be shown that the dynamics of $f(\tau,{\rm\Theta})$ belong to a RKHS with kernel $K_H$ defined as $K_H(\tau,\tau^\prime) = \sum_{h=1}^{H}K(x_h,x^\prime_h)$. We further define $\mcH_{K}$ and $\mcH_{K_H}$ as the RKHS induced by $K(\cdot,\cdot)$ and $K_H(\cdot,\cdot)$, respectively.

Based on the kernel $K(\cdot,\cdot)$ and $K_H(\cdot,\cdot)$, we define the Gram matrix $K_r, K_{v,h}\in\mR^{N\times N}$ as
\begin{flalign}
\textstyle K_r=[K_H(\tau_i,\tau_j)]_{i,j\in[N]},\quad\text{and}\quad K_{v,h}=[K(x^{\tau_i}_h,x^{\tau_j}_h)]_{i,j\in[N]}.\nonumber
\end{flalign}

We further define a function class as follows
\begin{flalign}
	\textstyle\mf_{B_1,B_2}=\left\{ f_\ell(x) = \int_{\mR^d}\sigma^\prime(w^\top x)\cdot x^\top\ell(w)dp(w): \sup_w\ltwo{\ell(w)}\leq B_1, \sup_w\frac{\ltwo{\ell(w)}}{p(w)}\leq B_2 \right\},\nonumber
\end{flalign}
where $\ell: \mR^d \rightarrow \mR^d$ is a mapping, $B_1,B_2$ are positive constants, and $p$ is the density of $N(0,I_d/d)$. We then make the following assumption regarding the expressive power of the above function class.
\begin{assumption}\label{ass1}
	We assume that there exist $a_1,a_2,A_1,A_2>0$ such that $R_h(\cdot)\in \mf_{a_1,a_2}$ and $(\mP_hf)(\cdot)\in \mf_{A_1,A_2}$ for any $f(\cdot):\mcx\rightarrow[0,H]$.
\end{assumption}
\Cref{ass1} ensures that both $R_h(\cdot)$ and $(\mP_h \widehat{V}_{h+1})(\cdot)$ can be captured by an infinite width neural network.
Note that \Cref{ass1} is mild since $\mf_{B_1,B_2}$ is an expressive function class as shown in Lemma C.1 of \cite{gao2019convergence}. Similar assumptions have also been adopted in many previous works that consider neural network function approximation \cite{yang2020function,cai2019neural,wang2019neural,xu2021crpo,qiu2021reward}.
Additionally, we assme that the data collection process explores the state-action space and trajectory space well. Note that similar assumptions have also been adopted in \cite{duan2020minimax,yin2022near,jin2021pessimism}.

\begin{assumption}[Well-Explored Dataset]\label{ass2}
	Suppose the $N$ trajectories in dataset $\mcd$ are independently and identically induced by a fixed behaviour policy $\mu$. There exist absolute constants $C_{\sigma}>0$ and $C_{\varsigma}>0$ such that
	\begin{flalign*}
		\textstyle\lambda_{\min}(\overline{M}({\rm\Theta}_0))\geq C_{\sigma}\quad\text{and}\quad\lambda_{\min}(\overline{m}_h({w}_0))\geq C_{\varsigma}\quad \forall h\in[H],
	\end{flalign*}
	where 
	\begin{flalign*}
		\textstyle\overline{M}({\rm\Theta}_0) = \mE_{\mu}\left[ {\rm\Phi}(\tau,{\rm\Theta}_0){\rm\Phi}(\tau,{\rm\Theta}_0)^\top \right]\quad\text{and}\quad \overline{m}_h({w}_0) = \mE_{\mu}\left[\phi(x^\tau_h,w_0)\phi(x^\tau_h,w_0)^\top \right].
	\end{flalign*}
\end{assumption}
We can now present the suboptimality of the policy $\widehat{\pi}$ obtained via \Cref{alg1}.
\begin{theorem}\label{thm1}
	Consider \Cref{alg1}. Suppose \Cref{ass}-\ref{ass2} hold. Let $\lambda_1 = \lambda_2 = 1 + 1/N$, $\beta_1= R_{\beta_1}$ and $\beta_2 = R_{\beta_2}$, in which $R_{\beta_1}$ and $R_{\beta_2}$ satisfy
	\begin{flalign*}
	&\textstyle R_{\beta_1} \ge H\left(4 a^2_2\lambda_1/d +   2\log\det\left(I+K_r/\lambda_1 \right) + 10\log(NH^2) \right)^{1/2}, \nonumber\\
	&\textstyle R_{\beta_2}\geq H\left(8 A^2_2\lambda_2/d +   4\max_{h\in[H]}\{\log\det\left(I+ K_{v,h}/\lambda_2\right)\} + 6C_\epsilon + 16\log(NH^2 \mcN_\epsilon^v)\right)^{1/2},
	\end{flalign*}
	where $\epsilon = \sqrt{\lambda_2 C_\epsilon}H/(2NC_\phi)$, $C_\epsilon\geq 1$ is an adjustable parameter, and $C_\phi>0$ is an absolute constant.
	In addition, let $m$ be sufficiently large. Then, with probability at least $1-(N^2H^4)^{-1}$, we have
	\begin{flalign}
	\textstyle \text{\rm SubOpt}(\hpi,s) \leq \widetilde{\mathcal{O}}\left(\frac{H \max \{\beta_1, \beta_2\}}{\sqrt{N}}\right) + \varepsilon_1,\nonumber
	\end{flalign}
	where
	\begin{flalign*}
		\textstyle\varepsilon_1 = \max\{\beta_1 H^{5/3}, \beta_2 H^{7/6} \} \widetilde{\mathcal{O}}\left(   \frac{N^{1/12}}{m^{1/12}} \right) +  \widetilde{\mathcal{O}}\left( \frac{H^{17/6} N^{5/3}}{m^{1/6}} \right).
	\end{flalign*}
\end{theorem}
\Cref{thm1} shows that \Cref{alg1} can find an $\epsilon$-optimal policy with $\widetilde{\mathcal{O}}(H^2\max\{\beta_1,\beta_2\}^2/\epsilon^2)$ episodes of offline data in the trajectory-wise reward setting up to a function approximation error $\varepsilon_1$, which vanishes as the neural network width $2m$ increases. Note that the dependence of $\varepsilon_1$ on the network width, which is $\mathcal{O}(m^{-1/12})$, matches that of the approximation error in the previous work of value iteration algorithm with neural network function approximation \cite{yang2020function}.

{\bf Discussion of Proof of \Cref{thm1}. }Comparing to the analysis of PEVI for linear MDP with instantaneous reward, which has been extensively studied in offline RL \cite{jin2021pessimism,yin2022near,yin2021near}, our analysis needs to address the following two new challenges:
(1) In instantaneous reward setting, both $ R_h(\cdot)$ and $(\mP_h \widehat{V}_{h+1})(\cdot)$ can be learned together by solving a single regression problem in per-step scale. However, in our trajectory-wise reward setting, $R_h(\cdot)$ and $(\mP_h \widehat{V}_{h+1} )(\cdot)$ need to be learned separately by solving two regression problems (\cref{eq: 3,eq: 4}) in different scales, i.e., \cref{eq: 3} is in trajectory scale and \cref{eq: 4} is in per-step scale. In order to apply union concentrations to bound the Bellman estimation error $|(\mB_h\widehat{V}_h)(\cdot) - (\widehat{\mB}_h\widehat{V}_h)(\cdot)|$, we need to develop new techniques to handle the mismatch between \cref{eq: 3,eq: 4} in terms of scale.
(2) In linear MDP, both $R_h(\cdot)$ and $(\mP_h \widehat{V}_{h+1})(\cdot)$ can be captured exactly by linear functions. However, in the more general MDP that we consider, we need to develop new analysis to bound the estimation error that caused by the insufficient expressive power of neural networks in order to characterize the optimality of $\widehat{\theta}_h$ and $\widehat{w}_h$ in \cref{eq: 3,eq: 4}, respectively.

To obtain a more concrete suboptimality bound for \Cref{alg1}, we impose an assumption on the spectral structure of kernels $K_H$ and $K$.
\begin{assumption}[Finite Spectrum NTK \cite{yang2020function}]\label{ass4}
	Conditioned on the randomness of $(b_0, w_0)$, let $T_{K_H}$ and $T_{K}$ be the integral operator induced by $K_H$ and $K$ (see \Cref{sc: rkhs} for definition of $T_{K_H}$ and $T_K$), respectively, and let $\{\omega_j\}_{j\geq 1}$ and $\{\upsilon_j\}_{j\geq 1}$ be eigenvalues of $T_{K_H}$ and $T_{K}$, respectively. We have $\omega_j=0$ for all $j\geq D_1+1$ and $\upsilon_j=0$ for all $\upsilon_j\geq D_2 + 1$, where $D_1, D_2$ are positive integers.
\end{assumption}
\Cref{ass4} implies that $\mcH_{K_H}$ and $\mcH_K$ are $D_1$-dimensional and $D_2$-dimensional, respectively. For concrete examples of neural networks that satisfy \Cref{ass4}, please refer to Section B.3 in \cite{yang2020function}. Note that such an assumption is in parallel to the "effective dimension" assumption in \cite{zhou2020neural,valko2013finite}.
\begin{corollary}\label{corollary1}
	Consider \Cref{alg1}. Suppose \Cref{ass}-\ref{ass4} hold. Let $\lambda_1=\lambda_2=1+1/N$, ${\beta_1} = \tilde{\mathcal{O}}(H D_1)$ and ${\beta_2} = \tilde{\mathcal{O}} (H\max\{D_1, D_2\})$. Then, with probability at least $1-(N^{2}H^{4})^{-1}$, we have
	\begin{equation}
	\textstyle\text{\rm SubOpt}(\hpi,s) = \widetilde{\mathcal{O}}\left(D_{\text{eff}}H^2/\sqrt{N}\right) + \varepsilon_2,\nonumber
	\end{equation}
	where $D_{\text{eff}}=\max\{D_1, D_2\}$ denotes the effective dimension and
	\begin{flalign*}
		\textstyle\varepsilon_2 = \max\left\{\sqrt{H}, \max\{D_1, D_2\}, \frac{H^{5/3}N^{19/12}}{m^{1/12}} \right\} \widetilde{\mathcal{O}} \left(   \frac{H^{13/6}N^{1/12} }{m^{1/12}} \right).
	\end{flalign*}
\end{corollary}
\Cref{corollary1} states that when $\beta_1$ and $\beta_2$ are chosen properly according to the dimension of $\mcH_{K_H}$ and $\mcH_{K}$, the suboptimality of the policy $\widehat{\pi}$ incurred by \Cref{alg1} converges to an $\epsilon$-optimal policy with $\widetilde{\mathcal{O}}(D^2_{\text{eff}}H^4/\epsilon^2)$ episodes of offline data up to a function approximation error $\varepsilon_2$. 

\subsection{Suboptimality of PARTED under Linear MDPs}
In this section, we briefly illustrate our result by instantiating PARTED to simpler linear MDPs with trajectory-wise rewards. We further provide a detailed treatment of the linear MDP setting with trajectory-wise rewards in \Cref{sc: linearMDP}.

With an abuse of notation, we define the linear MDP as follows.
\begin{definition}[Linear MDP \cite{jin2020provably,yang2019sample}]\label{def1}
	We say an episodic MDP $(\mcs, \mca, \mP, r, H)$ is a linear MDP with a known feature map $\phi(\cdot): \mcx \rightarrow \mR^d$ if there exist an unknown vector $w^*_h(s) \in \mR^d$ over $\mcs$ and an unknown vector $\theta^*_h\in \mR^d$ such that
	\begin{flalign}
	\textstyle \mP_h(s^\prime|s,a) = \langle \phi(s,a), w^*_h(s^\prime)\rangle,\quad R_h(s,a) = \langle \phi(s,a), \theta^*_h \rangle, \label{eq: 171}
	\end{flalign}
	for all $(s,a,s^\prime)\in \mcs\times\mca\times\mcs$ at each step $h\in[H]$. Here we assume $\ltwo{\phi(x)}\leq 1$ for all $x\in\mcx$ and $\max\{ \ltwo{w^*_h(\mcs)}, \ltwo{\theta^*_h} \}\leq \sqrt{d}$ at each step $h\in[H]$, where $\ltwo{w^*_h(\mcs)}= \int_{\mcs}\ltwo{w^*_h(s)}ds$.
\end{definition}
In linear MDPs, it has been shown that both reward $R_h(\cdot)$ and transition value function $(\mP_h\widehat{V}_{h+1})(\cdot)$ are linear functions with respect to $\phi(\cdot)$ \cite{agarwal2019reinforcement,jin2020provably}. Thus, with only trajectory-wise rewards, we can construct the proxy reward $\widehat{R}_h(\cdot)$ and estimated transition value function $(\widehat{\mP}_h\widehat{V}_{h+1})(\cdot)$ by solving two linear regression problems that take similar forms as \cref{eq: 7} and \cref{eq: 74} with feature ${\rm\Phi}(\tau) = [\phi(x^\tau_1)^\top,\cdots,\phi(x^\tau_h)]^\top$ and $\phi(x)$, respectively. For a detailed description of the algorithm, please see \Cref{alg2} in \Cref{sc: prelinearMDP}. The following theorem characterizes the suboptimality of \Cref{alg2}.
\begin{theorem}[Informal]\label{thm2.1}
	Consider PARTED for linear MDP in \Cref{alg2}. Under appropriate hyperparameter setting and dataset coverage assumption, we have $\text{\rm SubOpt}(\hpi,s) \leq \widetilde{\mathcal{O}}( dH^3/\sqrt{N})$ holds with high probability.
\end{theorem}
Note that linear function with feature ${\rm\Phi}(\tau)$ and $\phi(x)$ belongs to RKHS with kernel $K^\prime_H(\tau,\tau^\prime) = \langle{\rm\Phi}(\tau),{\rm\Phi}(\tau^\prime)\rangle$ and $K^\prime(x,x^\prime) = \langle \phi(x), \phi(x^\prime) \rangle$, respectively. Thus, $\mcH_{K^\prime_H}$ is $dH$-dimensional and $\mcH_{K^\prime}$ is $d$-dimensional. 
The suboptimality of PARTED for linear MDP in \Cref{thm2.1} will match that in \Cref{corollary1} if we let $D_1=dH$ and $D_2 =d$ (which implies $D_{\text{eff}} = dH$), where the dynamic of neural networks can be approximately captured by RKHSs $\mcH_{K_H}$ and $\mcH_K$ defined in \Cref{sc: subop}.

To highlight why trajectory-wise reward RL is more challenging than instantaneous reward RL, we observe that \Cref{thm2.1} with {\em trajectory-wise rewards} has an additional dependence on the horizon $H$, compared to the suboptimality $\widetilde{\mathcal{O}}(dH^2/\sqrt{N})$ \cite[Corollary 4.5]{jin2021pessimism} of PEVI for linear MDP with {\em instantaneous rewards}. This additional dependence on $H$ is caused by the reward redistribution process, in which PARTED needs to solve a {trajectory-level} regression problem with feature ${\rm\Phi}(\tau)\in\mR^{dH}$, which inevitably introduces large uncertainty in the regression solution used to construct the per-step proxy reward.



%% file: appendix.tex
\section{Proof Flow of \Cref{thm1}}
In this section, we present the main proof flow of \Cref{thm1}.
We first decompose the suboptimality $\text{SubOpt}(\pi,s)$, and then present the two main results of \Cref{lemma2} and \Cref{lemma8} to bound the evaluation error and summation of penality terms, respectively. The detailed proof of \Cref{lemma2} and \Cref{lemma8} can be found at \Cref{sc: pflemma2} and \Cref{pfpenaltysummation}.

We define the evaluation error at each step $h\in[H]$ as
\begin{flalign}
	\delta_h(s,a) = (\mB_h \widehat{V}_{h+1})(s,a) - \widehat{Q}_h(s,a),\label{eq: 123}
\end{flalign}
where $\mB_h$ is the Bellman operator defined in \Cref{sc: MDP} and $\hV_{h}$ and $\hQ_h$ are estimation of state- and state-action value functions, respectively. To proceed the proof, we first decompose the suboptimality into three parts as follows via the standard technique (see Section A in \cite{jin2021pessimism}).
\begin{flalign}\label{eq: 6}
\text{\em SubOpt}(\pi,s) &= -\sum_{h=1}^{H}\mE_\pi\left[\delta_h(s_h,a_h)\big|s_1=s\right] + \sum_{h=1}^{H}\mE_{\pi^*}\left[\delta_h(s_h,a_h)\big|s_1=s\right] \nonumber\\
&\quad +\sum_{h=1}^{H}\mE_{\pi^*}\left[ \langle  \widehat{Q}_h(s_h,\cdot), \pi^*_h(\cdot|s_h)-\widehat{\pi}_h(\cdot|s_h) \rangle \big|s_1=s\right].
\end{flalign}
In \Cref{alg1}, the output policy at each horizon $\hpi_h$ is greedy with respect to the estimated Q-value $\hQ_h$. Thus, we have
\begin{flalign*}
	\langle  \widehat{Q}_h(s_h,\cdot), \pi^*_h(\cdot|s_h)-\widehat{\pi}_h(\cdot|s_h) \rangle \leq 0, \quad \forall h\in[H],\quad \forall s_h\in\mcs.
\end{flalign*}
According to \cref{eq: 6}, we have the following holds for the suboptimality of $\hpi=\{\hpi_h\}_{h=1}^H$
\begin{flalign}
	\text{\rm SubOpt}(\hpi,s) &= -\sum_{h=1}^{H}\mE_{\hpi}\left[\delta_h(s_h,a_h)\big|s_1=s\right] + \sum_{h=1}^{H}\mE_{\pi^*}\left[\delta_h(s_h,a_h)\big|s_1=s\right].\label{eq: 128}
\end{flalign}
In the following lemma, we provide the first main technical result for the proof, which bounds the evaluation error $\delta_h(s,a)$. Recall that we use $\mcx$ to represent the joint state-action space $\mcs\times\mca$ and use $x$ to represent a state action pair $(s,a)$. 
\begin{lemma}\label{lemma2}
	Let $\lambda_1$, $\lambda_2=1+1/N$. Suppose \Cref{ass1} holds. With probability at least $1-\mathcal{O}(N^{-2}H^{-4})$, it holds for all $h\in[H]$ and $x\in\mcx$ that
	\begin{flalign*}
		-\varepsilon_b \leq \delta_h(x) \leq  2\left[\beta_1\cdot b_{r,h}(x,\widehat{\rm\Theta}) + \beta_2\cdot b_{v,h}(x,\widehat{w}_h) + \varepsilon_b\right], \quad\forall x\in\mcx,\quad \forall h\in[H],
	\end{flalign*}
	where
	\begin{flalign*}
	\varepsilon_b &= \max\{\beta_1 H^{2/3}, \beta_2 H^{1/6} \} \mathcal{O}\left(   \frac{N^{1/12} (\log m)^{1/4}}{m^{1/12}} \right) +  \mathcal{O}\left( \frac{H^{17/6} N^{5/3}\sqrt{\log(N^2H^5m)}}{m^{1/6}} \right),\nonumber\\
	\beta_1 &= H\left(\frac{4 a^2_2\lambda_1}{d} +   2\log\det\left(I+\frac{K^r_N}{\lambda_1}\right) + 10\log(NH^2) \right)^{1/2},\nonumber\\
	\beta_2 &= H\left(\frac{8 A^2_2\lambda_2}{d} +   4\max\left\{\log\det\left(I+\frac{K^v_{N,h}}{\lambda_2}\right)\right\} + 6C_\epsilon + 16\log(NH^2 \mcN_\epsilon^v) \right)^{1/2},\nonumber\\
	\epsilon &= \sqrt{\lambda_2 C_\epsilon}H/(2NC_\phi),\,\,\text{where}\,\, C_\epsilon\geq 1.
	\end{flalign*}
\end{lemma}
\begin{proof}
	The main technical development of the proof lies in handling the uncertainty caused by redistributing the trajectory-wise reward via solving a trajectory-level regression problem and analyzing the dynamics of neural network optimization.
	The detailed proof is provided in \Cref{sc: pflemma2}.
\end{proof}
Applying \Cref{lemma2} to \cref{eq: 128} yields
\begin{flalign}
\text{\rm SubOpt}(\hpi,s) &= -\sum_{h=1}^{H}\mE_{\hpi}\left[\delta_h(s_h,a_h)\big|s_1=s\right] + \sum_{h=1}^{H}\mE_{\pi^*}\left[\delta_h(s_h,a_h)\big|s_1=s\right]\nonumber\\
&\leq 3H\varepsilon_b + 2\beta_1\cdot \sum_{h=1}^{H}b_{r,h}(x,\widehat{\rm\Theta}) + 2\beta_2\cdot\sum_{h=1}^{H} b_{v,h}(x,\widehat{w}_h).\label{eq: 129}
\end{flalign}
The following lemma captures the second main technical result for the proof, which bounds the summation of the penalty terms $\beta_1\cdot \sum_{h=1}^{H}b_{r,h}(x,\widehat{\rm\Theta})+\beta_2\cdot\sum_{h=1}^{H} b_{v,h}(x,\widehat{w}_h)$.
\begin{lemma}\label{lemma8}
	Suppose \Cref{ass1}\&\ref{ass2} hold. We have the following holds with probability $1-\mathcal{O}(N^{-2}H^{-4})$
	\begin{flalign}
	&\beta_1\cdot\sum_{h=1}^{H}b_{r,h}(x,\widehat{\rm\Theta}) + \beta_2\cdot \sum_{h=1}^{H}b_{v,h}(x,\widehat{w}_h)  \nonumber\\
	&\quad \leq  \left( \frac{\beta_1}{\sqrt{C_\sigma}} + \frac{\beta_2}{\sqrt{C_\varsigma}} \right)\frac{\sqrt{2} H C_\phi}{\sqrt{N}} + \max\{\beta_1 H^{5/3},\beta_2 H^{7/6}\}\cdot \mathcal{O}\left(   \frac{ N^{1/12} (\log m)^{1/4}}{m^{1/12}} \right).\nonumber
	\end{flalign}
\end{lemma}
\begin{proof}
	The proof develops new analysis to characterize the summation of the penality term $b_{r,h}$ constructed by trajectory features, which is unique in the trajectory-wise reward setting. The detailed proof is provided in \Cref{pfpenaltysummation}.
\end{proof}
Applying \Cref{lemma8} to \cref{eq: 129}, we have
\begin{flalign}
&\text{\rm SubOpt}(\hpi,s) \nonumber\\
&\quad \leq 3H\varepsilon_b + \left( \frac{\beta_1}{\sqrt{C_\sigma}} + \frac{\beta_2}{\sqrt{C_\varsigma}} \right)\frac{2\sqrt{2} H C_\phi}{\sqrt{N}} + \max\{\beta_1 H^{5/3},\beta_2 H^{7/6}\}\cdot \mathcal{O}\left(   \frac{ N^{1/12} (\log m)^{1/4}}{m^{1/12}} \right)\nonumber\\
&\quad \leq 4H\varepsilon_b+ \left( \frac{\beta_1}{\sqrt{C_\sigma}} + \frac{\beta_2}{\sqrt{C_\varsigma}} \right)\frac{2\sqrt{2} H C_\phi}{\sqrt{N}},\label{eq: 150}
\end{flalign}
which completes the proof.


\section{Proof of \Cref{corollary1}}\label{pfcorollary1}
To provide a concrete bound for $\text{\rm SubOpt}(\hpi,s)$ defined in \cref{eq: 150}, we first need to bound the penalty coefficients $\beta_1$, $\beta_2$ under \Cref{ass4}. Recalling the properties of $\beta_1,\beta_2$ in \Cref{thm1}, we have
\begin{flalign}
&H\left(\frac{4 a^2_2\lambda_1}{d} +   2\log\det\left(I+\frac{K^r_N}{\lambda_1}\right) + 10\log(NH^2) \right)^{1/2}\leq R_{\beta_1} = \beta_1,\label{eq: 161}\\
&H\left(\frac{8 A^2_2\lambda_2}{d} +   4\max_{h\in[H]}\left\{\log\det\left(I+\frac{K^v_{N,h}}{\lambda_2}\right)\right\} + 6C_\epsilon + 16\log(NH^2 \mcN_\epsilon^v) \right)^{1/2}\leq R_{\beta_2} = \beta_2.\label{eq: 162}
\end{flalign}
Recall that we use $\mcx$ to represent the joint state-action space $\mcs\times\mca$ and use $x$ to represent a state action pair $(s,a)$.
We define the maximal information gain associated with RHKS with kernels $K^r_N$ and $K^v_{N,h}$ as follows
\begin{flalign}
{\rm\Gamma}_{K^r_N}(N,\lambda_1) &= \sup_{\mcd\subset \mcd_\tau}\{1/2\cdot \log\det(I_{2dmH} + \lambda_1^{-1}\cdot K^r_N) \},\label{eq: 163}\\
{\rm\Gamma}_{K^v_{N,h}}(N,\lambda_2) &= \sup_{\mcd\subset \mcd_x}\{1/2\cdot \log\det(I_{2dm} + \lambda_2^{-1}\cdot K^v_{N,h}) \},\label{eq: 164}
\end{flalign}
where $\mcd_x$ and $\mcd_\tau$ are discrete subsets of state-action pair $x\in\mcx$ and trajectory $\tau\in \mcx\times\cdots\times \mcx$ with cardinality no more than $N$, respectively. Applying \Cref{lemma10} in \Cref{sc: rkhs} and \Cref{ass4}, we have
\begin{equation}
{\rm\Gamma}_{K^r_N}(N,\lambda_1) \leq C_{K_1}\cdot D_1\cdot \log N\quad\text{and}\quad {\rm\Gamma}_{K^v_{N,h}}(N,\lambda_2) \leq C_{K_2}\cdot D_2 \cdot \log N,\label{eq: 165}
\end{equation}
where $C_{K_1}$, $C_{K_2}$ are absolute constants. Recall that $\mcN^v_{\epsilon,h}$ is the cardinality of the function class. Next, we proceed to bound the term $ \mcN_\epsilon^v = \max_{h\in[H]}\{ \mcN_{\epsilon,h}^v \}$.
\begin{flalign*}
\mcv_h(x,&\,R_\theta, R_w, R_{\beta_1}, R_{\beta_2}, \lambda_1, \lambda_2 )=\Big\{ \max_{a\in\mca}\{\overline{Q}_h(s,a,\theta,w,\beta_1,\beta_2,{\rm\Sigma}, {\rm\Lambda} )\}:\mcs\rightarrow [0,H]\,\, \nonumber\\
&\text{with}\,\, \ltwo{\theta-\theta_0}\leq R_\theta, \ltwo{w-w_0}\leq R_w, \beta_1\in[0, R_{\beta_1}], \beta_2\in [0, R_{\beta_2}], \ltwo{\rm\Sigma}\geq \lambda_1, \ltwo{\rm\Lambda}\geq \lambda_2  \Big\},
\end{flalign*}
where $R_{\theta} = H\sqrt{N/\lambda_1}$, $R_{w} = H\sqrt{N/\lambda_2}$, and 
\begin{flalign*}
\overline{Q}_h&(x,\theta,w,\beta_1,\beta_2,{\rm\Sigma}, {\rm\Lambda}) = \min\{ \langle \phi(x,\theta_0), \theta - \theta_0 \rangle +  \langle \phi(x,w_0), w- w_0 \rangle \nonumber\\
&- \beta_1\cdot \sqrt{{\rm\Phi}_h(x,\theta_0)^\top {\rm\Sigma}^{-1}{\rm\Phi}_h(x,\theta_0) } - \beta_2\cdot \sqrt{\phi(x,w_0)^\top{\rm\Lambda}^{-1}\phi(x,w_0)  } , H \}^{+}.
\end{flalign*}
Note that
\begin{flalign}
	&\lone{\max_{a\in\mca}\{\overline{Q}_h(s,a,\theta,w,\beta_1,\beta_2,{\rm\Sigma}, {\rm\Lambda} )\} - \max_{a\in\mca}\{\overline{Q}_h(s,a,\theta^\prime,w^\prime,\beta_1^\prime,\beta_2^\prime,{\rm\Sigma}^\prime, {\rm\Lambda}^\prime )\}}\nonumber\\
	&\leq \max_{a\in\mca}\lone{\overline{Q}_h(s,a,\theta,w,\beta_1,\beta_2,{\rm\Sigma}, {\rm\Lambda} ) - \overline{Q}_h(s,a,\theta^\prime,w^\prime,\beta_1^\prime,\beta_2^\prime,{\rm\Sigma}^\prime, {\rm\Lambda}^\prime )}\nonumber\\
	&\overset{(i)}{\leq} \max_{a\in\mca}\lone{\langle \phi(x,\theta_0), \theta - \theta^\prime  \rangle} +   \max_{a\in\mca}\lone{\langle \phi(x,w_0), w- w^\prime \rangle } \nonumber\\
	&\quad +  \max_{a\in\mca}\lone{(\beta_1- \beta^\prime_1)\cdot \sqrt{{\rm\Phi}_h(x,\theta_0)^\top {\rm\Sigma}^{-1}{\rm\Phi}_h(x,\theta_0) } }\nonumber\\
	&\quad +  \max_{a\in\mca}\lone{\beta^\prime_1\cdot \left[\sqrt{ {\rm\Phi}_h(x,\theta_0)^\top {\rm\Sigma}^{-1}{\rm\Phi}_h(x,\theta_0) } - \sqrt{{\rm\Phi}_h(x,\theta_0)^\top {\rm\Sigma}^{\prime-1}{\rm\Phi}_h(x,\theta_0) }\right]  }\nonumber\\
	&\quad +  \max_{a\in\mca}\lone{(\beta_2- \beta^\prime_2)\cdot \sqrt{\phi(x,w_0)^\top{\rm\Lambda}^{-1}\phi(x,w_0)  }}\nonumber\\
	&\quad +  \max_{a\in\mca}\lone{\beta^\prime_2\cdot \left[\sqrt{\phi(x,w_0)^\top{\rm\Lambda}^{-1}\phi(x,w_0)  } -  \sqrt{\phi(x,w_0)^\top{\rm\Lambda}^{\prime-1}\phi(x,w_0) }\right] }\nonumber\\
	&\overset{(ii)}{\leq} \max_{a\in\mca}\lone{\langle {\rm\Phi}_h(x,{\rm\Theta}_0), {\rm\Theta} - {\rm\Theta}^\prime  \rangle} +   \max_{a\in\mca}\lone{\langle \phi(x,w_0), w- w^\prime \rangle } + \frac{C_\phi}{\sqrt{\lambda_1}} \lone{\beta_1-\beta^\prime_1} + \frac{C_\phi}{\sqrt{\lambda_2}}\lone{\beta_2-\beta^\prime_2}\nonumber\\
	&\quad + R_{\beta_1} \max_{a\in\mca} \lone{ \lsigmab{{\rm\Phi}_h(x,\theta_0)} - \lsigmabp{{\rm\Phi}_h(x,\theta_0)} } \nonumber\\
	&\quad + R_{\beta_2} \max_{a\in\mca} \lone{ \llambdab{\phi(x,w_0)} - \llambdabp{\phi(x,w_0)}  },\label{eq: 151}
\end{flalign}
where $(i)$ follows from contractive properties of operators $\min\{ \cdot, H \}^{+}$ and $\max_{a\in\mca}\{\cdot \}$ and the triangle inequality, and $(ii)$ follows from the fact that $\ltwo{\phi(x,w_0)}, \ltwo{{\rm\Phi}_h(x,\theta_0)} \leq C_\phi$. 

Following arguments similar to those in the proof of Corollaries 4.8, Corollaries 4.4 and Section D.1 in \cite{yang2020function}, we have the followings hold for terms in the right hand side of \cref{eq: 151}
\begin{flalign}
	&\lone{\langle {\rm\Phi}_h(x,{\rm\Theta}_0), {\rm\Theta} - {\rm\Theta}^\prime  \rangle} = \lone{g_1(x) - g_2(x)}\,\,\text{where}\,\, \lHKH{g_i}\leq R_g = 2H\sqrt{{\rm\Gamma}_{K^r_N}(N,\lambda_1)} \,\,\forall i\in\{ 1, 2\},  \label{eq: 168}\\
	&\lone{\langle \phi(x,w_0), w- w^\prime \rangle } = \lone{h_1(x) - h_2(x)}\,\,\text{where}\,\, \lHK{h_i}\leq R_h = 2H\sqrt{{\rm\Gamma}_{K^v_{N,h}}(N,\lambda_2)} \,\, \forall i\in\{ 1, 2\},\label{eq: 169}\\
	& \lone{ \lsigmab{{\rm\Phi}_h(x,\theta_0)} - \lsigmabp{{\rm\Phi}_h(x,\theta_0)} } = \lone{ \lomegab{ {\rm\Psi}(x)} - \lomegabp{{\rm\Psi}(x)} },\nonumber\\
	& \lone{ \llambdab{\phi(x,w_0)} - \llambdabp{\phi(x,w_0)}  } = \lone{ \lupsilon{\psi(x)} - \lupsilonp{\psi(x)}},\nonumber
\end{flalign}
where $g_1(\cdot), g_2(\cdot)$ are two functions in RKHS $\mcH_{K_H}$, $h_1(\cdot), h_2(\cdot)$ are two functions in RKHS $\mcH_{K}$, ${\rm\Psi}(\cdot)$ and $\psi(\cdot)$ are feature mappings of RKHSs $\mcH_{K_H}$ and $\mcH_{K}$, respectively, ${\rm\Omega}, {\rm\Omega}^\prime: \mcH_{K_H}\rightarrow \mcH_{K_H}$ are self-adjoint operators with eigenvalues bounded in $[0,1/\lambda_1]$, and ${\rm\Upsilon}, {\rm\Upsilon}^\prime: \mcH_{K}\rightarrow\mcH_{K}$ are self-adjoint operators with eigenvalues bounded in $[0,1/\lambda_2]$. We define the following two function classes
\begin{flalign}
	\mcF_1 &= \{ \lomegab{ {\rm\Psi}(\cdot) }: \ltwo{\rm\Omega}\leq 1/\lambda_1 \},\label{eq: 152}
\end{flalign}
and
\begin{flalign}
\mcF_2 &= \{ \lupsilon{\psi(\cdot)}: \ltwo{\rm\Upsilon}\leq 1/\lambda_2 \}.\label{eq: 153}
\end{flalign}
For any $\epsilon > 0$, we denote $\mcN(\epsilon, \mcH, R)$ as the $\epsilon$-covering of $\{ f\in\mcH: \lH{f}\leq R \}$, denote $\mcN(\epsilon, \mcF_1, \lambda_1)$ as the $\epsilon$-covering number of $\mcF_1$ in \cref{eq: 152}, denote $\mcN(\epsilon, \mcF_2, \lambda_2)$ as the $\epsilon$-covering number of $\mcF_2$ in \cref{eq: 153}, and denote $\mcN(\epsilon, R)$ as the $\epsilon$-covering number of the interval $[0,R]$ with respect to the Euclidean distance. Note that \cref{eq: 151} implies
\begin{flalign}
	\mcN_{\epsilon,h}^v &\leq \mcN(\epsilon/6, \mcH_{K^H_m}, R_g)\cdot \mcN(\epsilon/6, \mcH_{K_m}, R_h) \cdot \mcN(\epsilon/(6C_\phi), R_{\beta_1}) \cdot \mcN(\epsilon/(6C_\phi), R_{\beta_2})\nonumber\\
	&\quad \cdot \mcN(\epsilon/(6R_{\beta_1}), \mcF_1, \lambda_1) \cdot \mcN(\epsilon/(6R_{\beta_2}), \mcF_2, \lambda_2).\label{eq: 154}
\end{flalign}
Based on Corollary 4.1.13 in \cite{vershynin2018high}, we have the followings hold for $\mcN(\epsilon/(6C_\phi), R_{\beta_1})$ and $\mcN(\epsilon/(6C_\phi), R_{\beta_2})$, respectively
\begin{flalign}
	 \mcN(\epsilon/(6C_\phi), R_{\beta_1})  \leq 1 + 12C_\phi R_{\beta_1}/\epsilon\,\,\text{and}\,\, \mcN(\epsilon/(6C_\phi), R_{\beta_2})  \leq 1 + 12C_\phi R_{\beta_2}/\epsilon.\label{eq: 155}
\end{flalign}
Moreover, as shown in Lemma D.2 and Lemma D.3 in \cite{yang2020function}, under the finite spectrum NTK assumption in \Cref{ass4}, we have the followings hold
\begin{flalign}
	 \log\mcN(\epsilon/6, \mcH_{K^H_m}, R_g) &\leq C_1\cdot D_1 \cdot [\log(6R_g/\epsilon) + C_2],\label{eq: 156}\\
	  \log\mcN(\epsilon/6, \mcH_{K_m}, R_h) &\leq C_3\cdot D_2 \cdot [\log(6R_h/\epsilon) + C_4],\label{eq: 157}\\
	  \log\mcN(\epsilon/(6R_{\beta_1}), \mcF_1, \lambda_1) &\leq C_5\cdot D_1^2\cdot[\log(6R_{\beta_1}/\epsilon) + C_6],\label{eq: 158}\\
	  \log\mcN(\epsilon/(6R_{\beta_2}), \mcF_2, \lambda_2) &\leq C_7\cdot D_2^2\cdot[\log(6R_{\beta_2}/\epsilon) + C_8].\label{eq: 159}
\end{flalign}
where $C_i$ ($i\in\{1,\cdots,8\}$) are absolute constants that do not rely on $N$, $H$ or $\epsilon$. Then, substituting \cref{eq: 155}-(\ref{eq: 159}) into \cref{eq: 154}, we have
\begin{flalign}
	\log\mcN_{\epsilon,h}^v &\leq \mcN(\epsilon/6, \mcH_{K^H_m}, R_g)+ \mcN(\epsilon/6, \mcH_{K_m}, R_h) + \mcN(\epsilon/(6C_\phi), R_{\beta_1}) + \mcN(\epsilon/(6C_\phi), R_{\beta_2})\nonumber\\
	&\quad + \mcN(\epsilon/(6R_{\beta_1}), \mcF_1, \lambda_1) + \mcN(\epsilon/(6R_{\beta_2}), \mcF_2, \lambda_2)\nonumber\\
	&\leq \log(1 + 12C_\phi R_{\beta_1}/\epsilon) + \log(1 + 12C_\phi R_{\beta_2}/\epsilon) + C_1D_1[\log(6R_g/\epsilon) + C_2] \nonumber\\
	&\quad + C_3D_2[\log(6R_h/\epsilon) + C_4] + C_5D_1^2[\log(6R_{\beta_1}/\epsilon) + C_6] + C_7 D_2^2[\log(6R_{\beta_2}/\epsilon) + C_8],\label{eq: 160}
\end{flalign}
We next proceed to show that there exists an absolute constant $R_{\beta_1}>0$ such that \cref{eq: 161} holds. Substituting \cref{eq: 165} to \cref{eq: 161}, we can obtain
\begin{flalign}
	\text{L.H.S of \cref{eq: 161}}\leq H\left(\frac{4 a^2_2\lambda_1}{d} +   4C_{K_1}D_1\log N + 10\log(NH^2) \right)^{1/2}.\nonumber
\end{flalign}
If we let
\begin{flalign}
	R_{\beta_1} = C_{\beta_1} H \sqrt{D_1 \log(NH^2)},\label{eq: 167}
\end{flalign}
in which $C_{\beta_1}$ is a sufficiently large constant, then we have the following holds
\begin{flalign}
\text{L.H.S of \cref{eq: 161}} \leq R_{\beta_1}.\nonumber
\end{flalign}
Note that \cref{eq: 160} directly implies that
\begin{flalign}
\log\mcN_{\epsilon}^v &= \max_{h\in[H]}\{ \log\mcN_{\epsilon,h}^v \} \nonumber\\
&\leq \log(1 + 12C_\phi R_{\beta_1}/\epsilon) + \log(1 + 12C_\phi R_{\beta_2}/\epsilon) + C_1D_1[\log(6R_g/\epsilon) + C_2] \nonumber\\
&\quad + C_3D_2[\log(6R_h/\epsilon) + C_4] + C_5D_1^2[\log(6R_{\beta_1}/\epsilon) + C_6] + C_7 D_2^2[\log(6R_{\beta_2}/\epsilon) + C_8]\nonumber\\
&\overset{(i)}{\leq} C^\prime_1D^2_1\log(R_{\beta_1}/\epsilon) + C^\prime_2D^2_2\log(R_{\beta_2}/\epsilon) + C^\prime_3 D_1 \log(H\sqrt{D_1}/\epsilon) + C^\prime_4 D_2 \log(H\sqrt{D_2}/\epsilon)\nonumber\\
&\overset{(ii)}{\leq} C^{\prime\prime}_1D^2_1\log(NH^2\sqrt{D_1}/\epsilon) + C^\prime_2D^2_2\log(R_{\beta_2}/\epsilon) + C^\prime_3 D_1 \log(H\sqrt{D_1}/\epsilon) \nonumber\\
&\quad+ C^\prime_4 D_2 \log(H\sqrt{D_2}/\epsilon),\label{eq: 166}
\end{flalign}
where in $(i)$ we let $C^\prime_1, C^\prime_2, C^\prime_3$ and $C^\prime_4$ be sufficiently large absolute constants, in $(ii)$ we use \cref{eq: 167} and let $C^{\prime\prime}_1$ be sufficiently large.
Then, we proceed to show that there exists an absolute constant $R_{\beta_2}>0$ such that \cref{eq: 162} holds. Using \cref{eq: 165} and \cref{eq: 166}, the left hand side of \cref{eq: 162} can be bounded as follows
\begin{flalign}
	&\text{L.H.S of \cref{eq: 162}}\nonumber\\
	&\quad \leq  H\Big(\frac{8 A^2_2\lambda_2}{d} +   8C_{K_2}\cdot D_2 \cdot \log N + 20\log(NH^2)  + 6C_\epsilon + 16 \log\mcN_{\epsilon}^v  \Big)^{1/2}\nonumber\\
	&\quad \overset{(i)}{\leq} H C_{\beta_2,1} \sqrt{D_2\log(NH^2)} + HC_{\beta_2,2}\sqrt{ \log\mcN_{\epsilon}^v} + HC_{\beta_2,3} \sqrt{C_\epsilon} \nonumber\\
	&\quad \leq H C_{\beta_2,1} \sqrt{D_2\log(NH^2)}  + HC_{\beta_2,2} \Big[ D_1\sqrt{C^{\prime\prime}_1\log(NH^2\sqrt{D_1}/\epsilon)} + D_2\sqrt{C^\prime_2\log(R_{\beta_2}/\epsilon)} \nonumber\\
	&\quad \quad + \sqrt{C^\prime_3 D_1 \log(H\sqrt{D_1}/\epsilon)} + \sqrt{C^\prime_4 D_2 \log(H\sqrt{D_2}/\epsilon)} \Big] + HC_{\beta_2,3} \sqrt{C_\epsilon}.\nonumber
\end{flalign}
where $(i)$ follows from the fact that $\sqrt{a + b}\leq \sqrt{a} + \sqrt{b}$ and $C_{\beta_2,1}$, $C_{\beta_2,2}$ and $C_{\beta_2,3}$ are sufficiently large constants. Clearly, if we let
\begin{flalign}
	R_{\beta_2} = C_{\beta_2} H\max\{D_1, D_2\} \log(NH^2\max\{D_1, D_2\}/\epsilon),\label{eq: 170}
\end{flalign}
where $C_{\beta_2}$ is a sufficiently large absolute constant, then we have
\begin{flalign}
		\text{L.H.S of \cref{eq: 162}}\leq R_{\beta_2}.
\end{flalign}
Finally, substituting the value of $R_{\beta_1}$ in \cref{eq: 167} and value of $R_{\beta_2}$ in \cref{eq: 170} into \cref{eq: 150} and letting $C_\epsilon = \max\{D_1, D_2\}^2$ (which implies $\epsilon = \sqrt{\lambda_2} \max\{D_1, D_2\} H/(2NC_\phi)$), we have
\begin{flalign}
&\text{\rm SubOpt}(\hpi,s) \nonumber\\
&\quad \leq \left( \frac{\beta_1}{\sqrt{C_\sigma}} + \frac{\beta_2}{\sqrt{C_\varsigma}} \right)\frac{2\sqrt{2} H C_\phi}{\sqrt{N}} + 4H\varepsilon_b\nonumber\\
&\quad \leq  \max\{R_{\beta_1},R_{\beta_2}\} \mathcal{O}\left(\frac{H}{\sqrt{N}}\right) +  \max\{R_{\beta_1} H^{5/3}, R_{\beta_2} H^{7/6} \} \mathcal{O}\left(   \frac{N^{1/12} (\log m)^{1/4}}{m^{1/12}} \right) \nonumber\\
&\quad\quad+ \mathcal{O}\left( \frac{H^{23/6} N^{5/3}\sqrt{\log(N^2H^5m)}}{m^{1/6}} \right)\nonumber\\
&\quad\leq \mathcal{O}\left(\frac{H^2\max\{D_1, D_2\}}{\sqrt{N}}\log\left(\frac{NH^2\max\{D_1, D_2\}}{\epsilon}\right)\right)\nonumber\\
&\quad\quad + \max\left\{\sqrt{H}, \max\{D_1, D_2\} \right\}\mathcal{O}\left(   \frac{H^{13/6}N^{1/12} (\log m)^{1/4}}{m^{1/12}}  \log\left( N^2H^2  \right) \right) \nonumber\\
&\quad\quad+ \mathcal{O}\left( \frac{H^{23/6} N^{5/3}\sqrt{\log(N^2H^5m)}}{m^{1/6}} \right)\nonumber\\
&\quad\overset{(i)}{\leq} \mathcal{O}\left(\frac{H^2\max\{D_1, D_2\}}{\sqrt{N}}\log\left(2C_\phi N^2 H\right)\right)\nonumber\\
&\quad\quad + \max\left\{\sqrt{H}, \max\{D_1, D_2\}, \frac{H^{5/3}N^{19/12}}{m^{1/12}} \right\}\mathcal{O}\left(   \frac{H^{13/6}N^{1/12} (\log m)^{1/4}}{m^{1/12}}  \log\left( N^2H^5m  \right) \right),\nonumber
\end{flalign}
where $(i)$ follows from the definition of $\epsilon$ and the fact that $\lambda_2\geq 1$.

\section{Linear MDP with Trajectory-Wise Reward}\label{sc: linearMDP}
In this section, we present the full details of our study on the offline RL in the linear MDP setting with trajectory-wise rewards. 

\subsection{Linear MDP and Algorithm}\label{sc: prelinearMDP}
We define the linear MDP \cite{jin2020provably} as follows, where the transition kernel and expected reward function are linear in a feature map. We use $\mcx$ to represent the joint state-action space $\mcs\times\mca$ and use $x$ to represent a state action pair. 

\begin{definition}[Linear MDP]
	We say an episodic MDP $(\mcs, \mca, \mP, r, H)$ is a linear MDP with a known feature map $\phi(\cdot): \mcx \rightarrow \mR^d$ if there exist an unknown vector $w^*_h(s) \in \mR^d$ over $\mcs$ and an unknown vector $\theta^*_h\in \mR^d$ such that
	\begin{flalign}
		\mP_h(s^\prime|s,a) = \langle \phi(s,a), w^*_h(s^\prime)\rangle,\quad R_h(s,a) = \langle \phi(s,a), \theta^*_h \rangle, \label{eq: 199}
	\end{flalign}
	for all $(s,a,s^\prime)\in \mcs\times\mca\times\mcs$ at each step $h\in[H]$. Here we assume $\ltwo{\phi(x)}\leq 1$ for all $x\in\mcx$ and $\max\{ \ltwo{w^*_h(\mcs)}, \ltwo{\theta^*_h} \}\leq \sqrt{d}$ at each step $h\in[H]$, where with an abuse of notation, we define $\ltwo{w^*_h(\mcs)}= \int_{\mcs}\ltwo{w^*_h(s)}ds$.
\end{definition}

\begin{algorithm}[tb]
	\caption{Linear Pessimistic Value Iteration with Reward Decomposition (PARTED)}
	\label{alg2}
	\begin{algorithmic}
		\STATE {\bfseries Input:} Dataset $\mcd=\{\tau_i,r(\tau_i)\}_{i,h=1}^{N,H}$
		\STATE {\bfseries Initialization:} Set $\widehat{V}_{H+1}$ as zero function
		\STATE Obtain $\widehat{R}_h$ and $\widehat{\rm\Theta}$ according to \cref{eq: 173}
		\FOR{$h=H,H-1,\cdot,1$}
		\STATE Obtain $\widehat{\mP}_h\widehat{V}_{h+1}$ and $\widehat{w}_h$ according to \cref{eq: 175}
		\STATE Obtain ${\rm\Gamma}_h(\cdot )$ according to \cref{eq: 178}
		\STATE $\widehat{Q}_h(\cdot) = \min\{ \widehat{R}_h(\cdot)+ \widehat{\mP}_h\widehat{V}_{h+1}(\cdot) - {\rm\Gamma}_h(\cdot), H-h+1 \}^{+}$
		\STATE $\hpi_h(\cdot|s)=\argmax_{\pi_h}\langle \hQ_h(s,\cdot), \pi_h(\cdot|s) \rangle$
		\STATE $\widehat{V}_h(\cdot) = \langle \hQ_h(\cdot,\cdot),\hpi_h(\cdot|\cdot) \rangle_\mca$
		\ENDFOR
	\end{algorithmic}
\end{algorithm}

We present our PARTED algorithm for linear MDPs with trajectory-wise rewards in \Cref{alg2}. Note that \Cref{alg2} shares a structure similar to that of \Cref{alg1}. Specifically, we estimate each $R_h(\cdot)$ for all $h\in[H]$ using a linear function $\langle \phi(s,a), \theta_h\rangle$, where $\theta_h\in\mR^d$ is a learnable parameter. We define the vector ${\rm\Theta} = [\theta_1^\top,\cdots,\theta^\top_H]\in\mR^{dH}$ and the loss function $L_r:\mR^{dH}\rightarrow\mR$ for reward learning as
\begin{flalign}\label{eq: 172}
L_r({\rm\Theta})=\sum_{\tau\in\mcd}\left[ \sum_{h=1}^{H} \langle \phi(x^\tau_h), \theta_h\rangle - r(\tau) \right]^2 + \lambda_1\cdot\sum_{h=1}^{H}\ltwo{\theta_h-\theta_0}^2,
\end{flalign}
where $\lambda_1>0$ is a regularization parameter. We then define $\widehat{R}_h(\cdot)$ as
\begin{flalign}\label{eq: 173}
\widehat{R}_h(\cdot) = \langle \phi(\cdot), \widehat{\theta}_h \rangle, \quad \text{where}\quad \widehat{\rm\Theta}=\argmin_{{\rm\Theta}\in\mR^{2dmH}}L_r( {\rm\Theta} )\,\,\,\text{and}\,\,\, \widehat{\rm\Theta}=[\widehat{\theta}_1^\top,\cdots,\widehat{\theta}_H^\top]^\top.
\end{flalign}
Similarly, we also use linear function $\langle \phi(s,a), w_h \rangle$ to estimate transition value functions $\{(\mP_h\widehat{V}_{h+1})(\cdot,\cdot)\}_{h\in[H]}$ for all $h\in[H]$, where $w_h\in\mR^{d}$ is a learnable parameter. For each $h\in[H]$, we define the loss function $L^h_v(w_h)$: $\mR^{d}\rightarrow\mR$ as
\begin{flalign}
L^h_v(w_h) = \sum_{\tau\in\mcd}\left(\widehat{V}_{h+1}(s^\tau_{h+1}) - \langle \phi(x^\tau_h), w_h \rangle \right)^2+\lambda_2\cdot \ltwo{w_h-w_0}^2,\label{eq: 174}
\end{flalign}
where $\lambda_2>0$ is a regularization parameter. We then define $(\widehat{\mP}_h\widehat{V}_{h+1})(\cdot): \mcx\rightarrow\mR$ as
\begin{flalign}\label{eq: 175}
(\widehat{\mP}_h\widehat{V}_{h+1})(\cdot) = \langle \phi(\cdot), \widehat{w}_h \rangle,\quad\text{where}\quad \widehat{w}_h=\argmin_{w_h\in\mR^{d}}L^h_v(w_h).
\end{flalign}
It remains to construct the penalty term $\rm{\Gamma}_h$. We first consider the penalty term that is used to offset the uncertainty raised from estimating the reward $R_h(\cdot)$ for each $h\in[H]$. We define the vectors  ${\rm\Phi}_h(x)=[\mathbf{0}^\top_d,\cdots,\phi(x)^\top,\cdots,\mathbf{0}^\top_d]^\top\in\mR^{dH}$ and ${\rm\Phi}(\tau)=[\phi(x^\tau_1),\cdots,\phi(x^\tau_H)]\in\mR^{dH}$, where ${\rm\Phi}_h(x)\in\mR^{dH}$ is a vector in which $[{\rm\Phi}_h(x)]_{d(h-1)+1:dh}=\phi(x)$ and the rest entries are all zero. We define a matrix ${\rm\Sigma}({\rm\Theta})\in\mR^{dH\times dH}$ as
\begin{flalign*}
{\rm\Sigma} = \lambda_1\cdot I_{dH} + \sum_{\tau\in\mcd}{\rm\Phi}(\tau){\rm\Phi}(\tau)^\top.
\end{flalign*}
The penalty term $b_{r,h}$ of the estimated reward is then defined as
\begin{flalign}\label{eq: 176}
b_{r,h}(x)=\left[{\rm\Phi}_h(x)^\top {\rm\Sigma}^{-1} {\rm\Phi}_h(x)\right]^{1/2},\quad\forall x\in\mcx.
\end{flalign}
Next, we consider the penalty term that is used to offset the uncertainty raised from estimating the transition value function $(\mP_h \widehat{V}_{h+1})(\cdot)$ for each $h\in[H]$. We define a matrix ${\rm\Lambda}_h\in\mR^{d\times d}$ as
\begin{flalign*}
{\rm\Lambda}_h = \lambda_2\cdot I_{d} + \sum_{\tau\in\mcd}\phi(x^\tau_h)\phi(x^\tau_h)^\top.
\end{flalign*}
The penality term $b_{v,h}$ of the estimated transition value function is then defined as
\begin{flalign}\label{eq: 177}
b_{v,h}(x)=\left[\phi(x)^\top {\rm\Lambda}_h^{-1} \phi(x)^\top \right]^{1/2},\quad\forall x\in\mcx.
\end{flalign}
Finally, the penalty term for the estimated Bellman operation $\widehat{\mB}_h\widehat{V}_{h+1}(\cdot)$ is obtained as
\begin{flalign}\label{eq: 178}
{\rm\Gamma}_h(x) = \beta_1 b_{r,h}(x) + \beta_2 b_{v,h}(x),
\end{flalign}
where $\beta_1,\beta_2>0$ are constant factors. 

\subsection{Main Result}\label{sc: mainlinearMDP}
We consider the following dataset coverage assumption so that we can explicitly bound the suboptimality of \Cref{alg2}. Note that the following assumption has also been considered in \cite{jin2021pessimism}.
\begin{assumption}[Well-Explored Dataset]\label{ass5}
	Suppose the $N$ trajectories in dataset $\mcd$ are independent and identically induced by a fixed behaviour policy $\mu$. There exist absolute constants $C_{\sigma}>0$ and $C_{\varsigma}>0$ such that
	\begin{flalign*}
	\lambda_{\min}(\overline{M}({\rm\Theta}_0))\geq C_{\sigma}\quad\text{and}\quad\lambda_{\min}(\overline{m}_h({w}_0))\geq C_{\varsigma}\quad \forall h\in[H],
	\end{flalign*}
	where 
	\begin{flalign*}
	\overline{M} = \mE_{\mu}\left[ {\rm\Phi}(\tau){\rm\Phi}(\tau)^\top \right]\quad\text{and}\quad \overline{m}_h({w}_0) = \mE_{\mu}\left[\phi(x^\tau_h)\phi(x^\tau_h)^\top \right].
	\end{flalign*}
\end{assumption}

We provide a formal statement of \Cref{thm2.1} as follows, which characterizes the suboptimality of \Cref{alg2}.
\begin{theorem}[Formal Statement of \Cref{thm2.1}]\label{thm2}
	Consider \Cref{alg2}. Let $\lambda_1 = \lambda_2 = 1$ and $\beta_1=\mathcal{O}(H\sqrt{dH\log(N/\delta)})$ and $\beta_2 = \mathcal{O}(dH^2\sqrt{\log(dH^3N^{5/2}/\delta)})$. Then, with probability at least $1-\delta$, we have
	\begin{flalign}
	&\text{\rm SubOpt}(\hpi,s) \leq \mathcal{O}\left( \frac{dH^3}{\sqrt{N}} \sqrt{\log\left(\frac{dH^3N^{5/2}}{\delta} \right)} \right).\nonumber
	\end{flalign}
\end{theorem}

\subsection{Proof Flow of \Cref{thm2.1}}
In this section, we present the main proof flow of \Cref{thm2.1}. Our main development is \Cref{lemma11}, the proof of which is presented in \Cref{pflinearmdp}. 

Recalling the suboptimality of $\hpi=\{\hpi_h\}_{h=1}^H$ in \cref{eq: 128}, we have
	\begin{flalign}
	\text{\rm SubOpt}(\hpi,s) &= -\sum_{h=1}^{H}\mE_{\hpi}\left[\delta_h(s_h,a_h)\big|s_1=s\right] + \sum_{h=1}^{H}\mE_{\pi^*}\left[\delta_h(s_h,a_h)\big|s_1=s\right],\nonumber
	\end{flalign}
	where $\delta_h(\cdot)$ is the evaluation error defined as
	\begin{flalign}
	\delta_h(s,a) = (\mB_h \widehat{V}_{h+1})(s,a) - \widehat{Q}_h(s,a).\nonumber
	\end{flalign}
	To characterize the suboptimality $\text{\rm SubOpt}(\hpi,s)$, we provide the following lemma to bound $\delta_h(\cdot)$ in the linear MDP setting.
	\begin{lemma}\label{lemma11}
		Let $\lambda_1$, $\lambda_2=1$, and let $\beta_1 = C_{\beta_1} H\sqrt{dH\log(N/\delta)}$ and $\beta_2 = C_{\beta_2} dH^2\sqrt{\log(dH^3N^{5/2}/\delta)}$, where $C_{\beta_1}, C_{\beta_2}$ are two absolute constants. Suppose \Cref{ass1} holds. With probability at least $1-\delta/2$, it holds for all $h\in[H]$ and $(s,a)\in\mcs\times\mca$ that
		\begin{flalign*}
		0 \leq \delta_h(x) \leq  2\left[\beta_1\cdot b_{r,h}(x) + \beta_2\cdot b_{v,h}(x)\right], \quad\forall x\in\mcx,\quad \forall h\in[H].
		\end{flalign*}
	\end{lemma}
	\begin{proof}
		The main technical development here lies in handling additional challenges caused by the reward redistribution of trajectory-wise rewards, which are not present in linear MDPs with instantaneous rewards \cite{jin2021pessimism}.
		The detailed proof is provided in \Cref{pflinearmdp}.
	\end{proof}
Applying \Cref{lemma11} to \cref{eq: 128}, we can obtain
\begin{flalign}
\text{\rm SubOpt}(\hpi,s) \leq 2\beta_1\cdot \sum_{h=1}^{H}b_{r,h}(x) + 2\beta_2\cdot\sum_{h=1}^{H} b_{v,h}(x).\label{eq: 195}
\end{flalign}
Then, following steps similar to those in \Cref{pfpenaltysummation}, we have the followings hold with probability at least $1-\delta/2$
\begin{flalign}
	b_{r,h}(x) \leq \frac{C^\prime}{\sqrt{N}}\quad\text{and}\quad b_{v,h}(x)\leq \frac{C^{\prime\prime}}{\sqrt{N}},\label{eq: 196}
\end{flalign}
where $C^\prime$ and $C^{\prime\prime}$ are absolute constants dependent only on $C_\sigma$, $C_\varsigma$ and $\log(1/\delta)$. Then, substituting \cref{eq: 196} into \cref{eq: 195}, we have the following holds with probability $1-\delta$
\begin{flalign}
	\text{\rm SubOpt}(\hpi,s) &\leq 2C_{\beta_1} H^2\sqrt{dH\log(N/\delta)} \cdot \frac{C^\prime}{\sqrt{N}} + 2C_{\beta_2} dH^3\sqrt{\log(dH^3N^{5/2}/\delta)}\cdot\frac{C^{\prime\prime}}{\sqrt{N}}\nonumber\\
	&\leq \mathcal{O}\left( \frac{dH^3}{\sqrt{N}} \sqrt{\log\left(\frac{dH^3N^{5/2}}{\delta} \right)} \right),\nonumber
\end{flalign}
which completes the proof.

\section{Proof of \Cref{lemma2}}\label{sc: pflemma2}
Recall that in \Cref{subsc: overpnn} we let $(b_0, w_0)$ be the initial value of network parameters obtained via the symmetric initialization scheme, which makes $f(\cdot;w_0)$ a zero function. We denote $(\widehat{\mB}_h\widehat{V}_{h+1})(\cdot) = \widehat{R}_h(\cdot) + (\widehat{\mP}_h\hV_{h+1})(\cdot)$ as the estimator of Bellman operator $({\mB}_h\widehat{V}_{h+1})(\cdot) = {R}_h(\cdot) + ({\mP}_h\hV_{h+1})(\cdot)$. To prove \Cref{lemma2}, we show that $(\widehat{\mB}_h\widehat{V}_{h+1})(\cdot) - \beta_1b_{r,h}(\cdot,\widehat{\rm\Theta}) - \beta_2b_{v,h}(\cdot,\widehat{w})$ is approximately a pessimistic estimator of $({\mB}_h\widehat{V}_{h+1})(\cdot)$ up to a function approximation error. We consider $m$ to be sufficiently large such that $m\geq NH^2$.

\subsection{Uncertainty of Estimated Reward $\widehat{R}_h(\cdot)$}
In this step, we aim to bound the estimation error $\lone{\widehat{R}_h(\cdot) - R_h(\cdot)}$.
Since $\widehat{\rm\Theta}$ is the global minimizer of the loss function $L_r$ defined in \cref{eq: 7}, we have
\begin{flalign}
	L_r(\widehat{\rm\Theta})&=\sum_{\tau\in\mcd}\left[ \sum_{h=1}^{H}f(x^\tau_h,\widehat{\theta}_h) - r(\tau) \right]^2 + \lambda_1\cdot\sum_{h=1}^{H}\ltwo{\widehat{\theta}_h-\theta_0}^2\nonumber\\
	&\leq L_r({\rm\Theta}_0) = \sum_{\tau\in\mcd}\left[ \sum_{h=1}^{H}f(x^\tau_h,\theta_0) - r(\tau) \right]^2 \overset{(i)}{=} \sum_{\tau\in\mcd}\left[r(\tau)\right]^2 \overset{(ii)}{\leq} NH^2,\label{eq: 8}
\end{flalign}
where $(i)$ follows from the fact that $f(x,\theta_0)=0$ for all $x\in\mcx$ and $(ii)$ follows from the fact that $r(\tau)\leq H$ for any trajectory $\tau$ and we have total $N$ trajectories in the offlline sample set $\mcd$. We define the vector ${\rm\Theta}_0=[\theta_0^\top,\cdots,\theta_0^\top]^\top\in\mR^{2mdH}$. Note that \cref{eq: 8} implies
\begin{flalign}
	\ltwo{\widehat{\theta}_h-\theta_0}^2\leq \ltwo{\widehat{\rm\Theta} - {\rm\Theta}_0}^2 = \sum_{h=1}^{H}\ltwo{\widehat{\theta}_h-\theta_0}^2 \leq NH^2/\lambda_1,\quad\forall h\in[H].\label{eq: 45}
\end{flalign}
Hence, each $\widehat{\theta}_h$ belongs to the Euclidean ball $\mcB_\theta=\{\theta\in\mR^{2md}: \ltwo{\theta-\theta_0}\leq H\sqrt{N/\lambda_1} \}$. 

Since the radius of $\mcB_\theta$ does not depend on $m$, when $m$ is sufficient large it can be shown that $f(\cdot,\theta)$ is close to its linearization at $\theta_0$, i.e.,
\begin{flalign*}
	f(\cdot,\theta)\approx \langle \phi(\cdot,\theta_0), \theta-\theta_0\rangle,\quad\forall \theta\in\mcB_\theta,
\end{flalign*}
where $\phi(\cdot,\theta)=\nabla_\theta f(\cdot,\theta)$. Furthermore, according to \Cref{ass1}, there exists a function $\ell_{a_1,a_2}:\mR^d\rightarrow\mR^d$ such that the mean of the true reward function $R_h(\cdot)=\mE[r_h(\cdot)]$ satisfies
\begin{flalign}\label{eq: 9}
	R_h(x) = \int_{\mR^d}\sigma^\prime(\theta^\top x) \cdot x^\top \ell_r(\theta)dp(\theta),
\end{flalign}
where $\sup_\theta\ltwo{\ell_r(\theta)}\leq a_1$, $\sup_\theta(\ltwo{\ell_r(\theta)}/p(\theta))\leq a_2$ and $p$ is the density of the distribution $N(0,I_d/d)$. We then proceed to bound the difference between $\widehat{R}_h(\cdot)$ and $R_h(\cdot)$.

{\bf Step I.} In the first step, we show that with high probability the mean of the true reward $R_h(\cdot)$ can be well-approximated by a linear function with the feature vector $\phi(\cdot,\theta_0)$. \Cref{lemma3} in \Cref{sc: supplemma} implies that that $R_h(\cdot)$ in \cref{eq: 9} can be well-approximated by a finite-width neural network, i.e., with probability at least $1-N^{-2}H^{-4}$ over the randomness of initialization $\theta_0$, for all $h\in[H]$, there exists a function $\widetilde{R}_h(\cdot):\mcx\rightarrow\mR$ satisfying
\begin{flalign}
	\sup_{x\in\mcx}\lone{\widetilde{R}_h(x) - R_h(x)}\leq \frac{2(L_\sigma a_2 + C^2_\sigma a^2_2)\sqrt{\log(N^2H^5)}}{\sqrt{m}}, \label{eq: 47}
\end{flalign}
where $\widetilde{R}_h(\cdot)$ can be written as
\begin{flalign*}
	\widetilde{R}_h(x) = \frac{1}{\sqrt{m}}\sum_{r=1}^{m}\sigma^\prime(\theta^{\top}_{0,r} x)\cdot x^\top \ell_r,
\end{flalign*}
where $\ltwo{\ell_r}\leq a_2/\sqrt{dm}$ for all $r\in[m]$ and $\theta_0=[\theta_{0,1},\cdots,\theta_{0,m}]$ is generated via the symmetric initialization scheme. We next proceed to show that there exists a vector $\tilde{\theta}_h\in\mR^{2md}$ such that $\widetilde{R}_h(\cdot) = \langle \phi(\cdot,\theta_0), \tilde{\theta}_h -\theta_0\rangle$. Let $\tilde{\theta}_h=[\tilde{\theta}^{\top}_{h,1},\cdots.\tilde{\theta}^{\top}_{h,2m}]^\top$, in which $\tilde{\theta}^{\top}_{h,r}=\theta_{0,r} + b_{0,r}\cdot \ell_r/\sqrt{2}$ for all $r\in\{1,\cdots,m\}$ and $\tilde{\theta}^{\top}_{h,r}=\theta_{0,r} + b_{0,r}\cdot \ell_{r-m}/\sqrt{2}$ for all $r\in\{m+1,\cdots,2m\}$. Then, we have
\begin{flalign}
	\widetilde{R}_h(x) &= \frac{1}{\sqrt{2m}}\sum_{r=1}^{m}\sqrt{2}(b_{0,r})^2\cdot \sigma^\prime(\theta^{\top}_{0,r}x)\cdot x^\top\ell_r\nonumber\\
	&=\frac{1}{\sqrt{2m}}\sum_{r=1}^{m} \frac{1}{\sqrt{2}} (b_{0,r})^2\cdot \sigma^\prime(\theta^{\top}_{0,r}x)\cdot x^\top\ell_r + \frac{1}{\sqrt{2m}}\sum_{r=1}^{m}\frac{1}{\sqrt{2}}(b_{0,r})^2\cdot \sigma^\prime(\theta^{\top}_{0,r}x)\cdot x^\top\ell_{r-m}\nonumber\\
	&=\frac{1}{\sqrt{2m}}\sum_{r=1}^{2m} b_{0,r}\cdot \sigma^\prime(\theta^{\top}_{0,r}x)\cdot x^\top(\tilde{\theta}_{h,r}-\theta_{0,r})\nonumber\\
	&=\phi(x,\theta_0)^\top(\tilde{\theta}_h - \theta_0).\label{eq: 79}
\end{flalign}
Thus, the true mean reward $R_h(\cdot)$ is approximately linear with the feature $\phi(\cdot,\theta_0)$. Since $\tilde{\theta}_{h,r} - \theta_{0,r} =  b_{0,r}\cdot \ell_r/\sqrt{2}$ or $ b_{0,r}\cdot \ell_{r-m}/\sqrt{2}$, we have 
\begin{flalign*}
	\ltwo{\tilde{\theta}_h - \theta_0}\leq a_2\sqrt{2dm}.
\end{flalign*}

{\bf Step II.} In this step, we show that $ \widehat{R}_h(\cdot)$ learned by neural network in \Cref{alg1} can be well-approximated by its counterpart learned by a linear function with feature $\phi(\cdot,\theta_0)$.

Consider the following least-square loss function 
\begin{flalign*}
	\bar{L}_r({\rm\Theta}) &= \sum_{\tau\in\mcd}\left[\sum_{h=1}^{H} \langle \phi(x^\tau_h,\theta_0), \theta_h-\theta_0  \rangle - r(\tau) \right]^2 + \lambda_1\cdot\sum_{h=1}^{H}\ltwo{\theta_h-\theta_0}^2\nonumber\\
	&= \sum_{\tau\in\mcd}\left[ \langle {\rm\Phi}(\tau,{\rm\Theta}_0), {\rm\Theta}-{\rm\Theta}_0  \rangle - r(\tau) \right]^2 + \lambda_1 \cdot \ltwo{{\rm\Theta}-{\rm\Theta}_0}^2.
\end{flalign*}
The global minimizer of $\bar{L}_r({\rm\Theta})$ is defined as
\begin{flalign}\label{eq: 16}
 \overline{\rm\Theta}=\argmin_{{\rm\Theta}\in\mR^{2dmH}}\bar{L}_r(\theta)\,\,\,\text{and}\,\,\,  \overline{\rm\Theta}=[\bar{\theta}_1^\top,\cdots,\bar{\theta}_H^\top]^\top.
\end{flalign}
We define $\overline{R}_h(\cdot)=\langle \phi(\cdot,\theta_0), \bar{\theta}_h - \theta_0 \rangle$ for all $h\in[H]$. We then proceed to bound the term $\lone{\widehat{R}_h(x) - \overline{R}_h(x)}$ as follows
\begin{flalign*}
	\lone{\widehat{R}_h(x) - \overline{R}_h(x)} & = \lone{f(x,\widehat{\theta}_h) - \langle \phi(x,\theta_0), \bar{\theta}_h - \theta_0 \rangle}\nonumber\\
	& = \lone{f(x,\widehat{\theta}_h) - \langle {\rm\Phi}_h(x,{\rm\Theta}_0), \overline{\rTheta} - \rTheta_0 \rangle}\nonumber\\
	& = \lone{f(x,\widehat{\theta}_h) - \langle {\rm\Phi}_h(x,{\rm\Theta}_0), \widehat{\rm\Theta} - {\rm\Theta}_0 \rangle + \langle {\rm\Phi}_h(x,{\rm\Theta}_0), \widehat{\rm\Theta} - \overline{\rm\Theta} \rangle}\nonumber\\
	&\leq \lone{f(x,\widehat{\theta}_h) - \langle {\rm\Phi}_h(x,{\rm\Theta}_0), \widehat{\rTheta} - \rTheta_0 \rangle} + \lone{\langle {\rm\Phi}_h(x,{\rm\Theta}_0), \widehat{\rTheta} - \overline{\rTheta} \rangle}\nonumber\\
	&= \lone{f(x,\widehat{\theta}_h) - \langle {\phi}_h(x,{\theta}_0), \widehat{\theta}_h - \theta_0 \rangle} + \lone{\langle {\rm\Phi}_h(x,{\rm\Theta}_0), \widehat{\rTheta} - \overline{\rTheta} \rangle}\nonumber\\
	&\leq  {\lone{f(x,\widehat{\theta}_h) - \langle {\phi}_h(x,{\theta}_0), \widehat{\theta}_h - \theta_0 \rangle}} +  {\ltwo{{\rm\Phi}_h(x,{\rm\Theta}_0)} \ltwo{\widehat{\rTheta} - \overline{\rTheta}}}\nonumber\\
	&=  \underbrace{\lone{f(x,\widehat{\theta}_h) - \langle {\phi}_h(x,{\theta}_0), \widehat{\theta}_h - \theta_0 \rangle}}_{(i)} +  \underbrace{\ltwo{{\phi}(x,{\theta}_0)} \ltwo{\widehat{\rTheta} - \overline{\rTheta}}}_{(ii)}.
\end{flalign*}
According to \Cref{lemma4} and the fact that $\ltwo{\widehat{\theta}_h-\theta_0}\leq H\sqrt{N/\lambda_1}$, we have the followings hold with probability at least $1-N^{-2}H^{-4}$
\begin{flalign}
	(i)&\leq \mathcal{O}\left(C_\phi \left(\frac{N^2H^4}{\lambda^2_1\sqrt{m}}\right)^{1/3}\sqrt{\log m}\right),\label{eq: 28}\\
	(ii)&\leq  C_\phi \ltwo{\widehat{\rTheta} - \overline{\rTheta}}.\label{eq: 29}
\end{flalign}
We then proceed to bound the term $\ltwo{\widehat{\rTheta} - \overline{\rTheta}}$. Consider the minimization problem defined in \cref{eq: 3} and \cref{eq: 16}. By the first order optimality condition, we have
\begin{flalign}
	\lambda_1\left( \widehat{\rTheta} - \rTheta_0 \right) &= \sum_{\tau\in\mcd}\left( r(\tau) -  \sum_{h=1}^{H}f(x^\tau_h,\hat{\theta}_h) \right){\rm\Phi}(\tau,\widehat{\rTheta})\label{eq: 17}\\
	\lambda_1\left( \overline{\rTheta} - \rTheta_0 \right) &= \sum_{\tau\in\mcd}\left( r(\tau) -  \langle {\rm\Phi} (\tau,\rTheta_0), \overline{\rTheta}-\rTheta_0 \rangle \right){\rm\Phi}(\tau,\rTheta_0)\label{eq: 18}.
\end{flalign}
Note that \cref{eq: 18} implies
\begin{flalign}\label{eq: 19}
	{\rm\Sigma}({\rm\Theta}_0)\left( \overline{\rTheta} - \rTheta_0 \right) = \sum_{\tau\in\mcd}r(\tau){\rm\Phi}(\tau,\rTheta_0).
\end{flalign}
Adding the term $\sum_{\tau\in\mcd} \langle {\rm\Phi} (\tau,{\rm\Theta}_0), \widehat{\rm\Theta}-{\rm\Theta}_0 \rangle {\rm\Phi} (\tau,\rTheta_0)$ on both sides of \cref{eq: 17} yields
\begin{flalign}\label{eq: 20}
	{\rm\Sigma}({\rm\Theta}_0)\left( \widehat{\rTheta} - \rTheta_0 \right) &= \sum_{\tau\in\mcd}r(\tau){\rm\Phi}(\tau,\widehat{\rm\Theta}) \nonumber\\
	& \quad + \sum_{\tau\in\mcd} \left[ \langle {\rm\Phi} (\tau,{\rm\Theta}_0), \widehat{\rm\Theta}-{\rm\Theta}_0 \rangle {\rm\Phi} (\tau,{\rm\Theta}_0) - \left(\sum_{h=1}^{H}f(x^\tau_h,\hat{\theta}_h) \right) {\rm\Phi}(\tau,\widehat{\rTheta}) \right].
\end{flalign}
Then, subtracting \cref{eq: 19} from \cref{eq: 20}, we can obtain
\begin{flalign}
	{\rm\Sigma}({\rm\Theta}_0)(\widehat{\rTheta} - \overline{\rTheta})&=\sum_{\tau\in\mcd}r(\tau) \left({\rm\Phi}(\tau,\widehat{\rm\Theta}) - {\rm\Phi}(\tau,{\rm\Theta}_0)\right)\nonumber\\
	&\quad + \sum_{\tau\in\mcd} \left[ \langle {\rm\Phi} (\tau,{\rm\Theta}_0), \widehat{\rm\Theta}-{\rm\Theta}_0 \rangle {\rm\Phi} (\tau,{\rm\Theta}_0) - \left(\sum_{h=1}^{H}f(x^\tau_h,\hat{\theta}_h) \right) {\rm\Phi}(\tau,\widehat{\rTheta}) \right],\label{eq: 21}
\end{flalign}
which implies
\begin{flalign}
\ltwo{{\rm\Sigma}({\rm\Theta}_0)(\widehat{\rTheta} - \overline{\rTheta})}&\leq \sum_{\tau\in\mcd}r(\tau) \sqrt{\sum_{h=1}^{H}\ltwo{{\phi} (x^\tau_h,{\theta}_0) - {\phi} (x^\tau_h,\widehat{\theta}_h)}^2} \nonumber\\
&\quad + \sum_{\tau\in\mcd}  \ltwo{\langle {\rm\Phi} (\tau,{\rm\Theta}_0), \widehat{\rm\Theta}-{\rm\Theta}_0 \rangle {\rm\Phi} (\tau,{\rm\Theta}_0) - \left(\sum_{h=1}^{H}f(x^\tau_h,\hat{\theta}_h) \right) {\rm\Phi}(\tau,\widehat{\rTheta})}. \label{eq: 27}
\end{flalign}
To bound the term $\ltwo{\langle {\rm\Phi} (\tau,{\rm\Theta}_0), \widehat{\rm\Theta}-{\rm\Theta}_0 \rangle {\rm\Phi} (\tau,{\rm\Theta}_0) - \left(\sum_{h=1}^{H}f(x^\tau_h,\hat{\theta}_h) \right) {\rm\Phi}(\tau,\widehat{\rTheta})}$, we proceed as follows
\begin{flalign*}
	&\langle {\rm\Phi} (\tau,{\rm\Theta}_0), \widehat{\rm\Theta}-{\rm\Theta}_0 \rangle {\rm\Phi} (\tau,{\rm\Theta}_0) - \left(\sum_{h=1}^{H}f(x^\tau_h,\hat{\theta}_h) \right) {\rm\Phi}(\tau,\widehat{\rTheta})\nonumber\\
	&=\langle {\rm\Phi} (\tau,{\rm\Theta}_0), \widehat{\rm\Theta}-{\rm\Theta}_0 \rangle ({\rm\Phi} (\tau,{\rm\Theta}_0) - {\rm\Phi} (\tau,\widehat{\rm\Theta})) - \left( \langle {\rm\Phi} (\tau,{\rm\Theta}_0), \widehat{\rm\Theta}-{\rm\Theta}_0 \rangle   - \sum_{h=1}^{H}f(x^\tau_h,\hat{\theta}_h) \right) {\rm\Phi}(\tau,\widehat{\rTheta})\nonumber\\
	&=\langle {\rm\Phi} (\tau,{\rm\Theta}_0), \widehat{\rm\Theta}-{\rm\Theta}_0 \rangle ({\rm\Phi} (\tau,{\rm\Theta}_0) - {\rm\Phi} (\tau,\widehat{\rm\Theta})) - \left[ \sum_{h=1}^{H} \left(\langle {\phi} (x^\tau_h,{\theta}_0), \widehat{\theta}_h-{\theta}^0_h \rangle   - f(x^\tau_h,\hat{\theta}_h) \right)\right] {\rm\Phi}(\tau,\widehat{\rTheta}),
\end{flalign*}
which implies
\begin{flalign}
	&\ltwo{\langle {\rm\Phi} (\tau,{\rm\Theta}_0), \widehat{\rm\Theta}-{\rm\Theta}_0 \rangle {\rm\Phi} (\tau,{\rm\Theta}_0) - \left(\sum_{h=1}^{H}f(x^\tau_h,\hat{\theta}_h) \right) {\rm\Phi}(\tau,\widehat{\rTheta})}\nonumber\\
	&\leq \ltwo{{\rm\Phi} (\tau,{\rm\Theta}_0)} \ltwo{\widehat{\rm\Theta}-{\rm\Theta}_0} \ltwo{{\rm\Phi} (\tau,{\rm\Theta}_0) - {\rm\Phi} (\tau,\widehat{\rm\Theta})} \nonumber\\
	&\quad + \left[ \sum_{h=1}^{H}  \lone{\langle {\phi} (x^\tau_h,{\theta}_0), \widehat{\theta}_h-{\theta}^0_h \rangle - f(x^\tau_h,\hat{\theta}_h)} \right] \ltwo{{\rm\Phi}(\tau,\widehat{\rTheta})},\nonumber\\
	&= \sqrt{\sum_{h=1}^{H}\ltwo{\phi(x^\tau_h,{\theta}^0_h)}^2} \ltwo{\widehat{\rm\Theta}-{\rm\Theta}_0} \sqrt{\sum_{h=1}^{H}\ltwo{{\phi} (x^\tau_h,{\theta}_0) - {\phi} (x^\tau_h,\widehat{\theta}_h)}^2} \nonumber\\
	&\quad + \left[ \sum_{h=1}^{H}  \lone{\langle {\phi} (x^\tau_h,{\theta}_0), \widehat{\theta}_h-{\theta}^0_h \rangle - f(x^\tau_h,\hat{\theta}_h)} \right] \sqrt{\sum_{h=1}^{H}\ltwo{\phi(x^\tau_h,\hat{\theta}_h)}^2},\label{eq: 22}
\end{flalign}
where the last equality follows from the fact that $\ltwo{{\rm\Phi}(\tau,{\rTheta})}^2 = \sum_{h=1}^{H}\ltwo{\phi(x^\tau_h,{\theta}_h)}^2$ for any ${\rm\Theta}\in\mR^{2mdH}$. According to \Cref{lemma4} and the fact that $\ltwo{\widehat{\theta}_h-\theta_0}\leq H\sqrt{N/\lambda_1}$, we have the followings hold with probability at least $1-N^{-2}H^{-4}$ for all $h\in[H]$ and $\tau\in\mcd$
\begin{flalign}
	&\ltwo{\phi(x^\tau_h,{\theta}_0)} \leq C_\phi\quad\text{and}\quad	\ltwo{\phi(x^\tau_h,\hat{\theta}_h)} \leq C_\phi,\label{eq: 23}\\
	&\ltwo{{\phi} (x^\tau_h,{\theta}_0) - {\phi} (x^\tau_h,\widehat{\theta}_h)}\leq \mathcal{O}\left( C_\phi \left( \frac{H\sqrt{N/\lambda_1}}{\sqrt{m}}\right)^{1/3}\sqrt{\log m} \right),\label{eq: 24}\\
	&\lone{\langle {\phi} (x^\tau_h,{\theta}_0), \widehat{\theta}_h-{\theta}_0 \rangle - f(x^\tau_h,\hat{\theta}_h)}\leq \mathcal{O}\left(C_\phi \left(\frac{H^4N^2/\lambda^2_1}{\sqrt{m}}\right)^{1/3}\sqrt{\log m}\right).\label{eq: 25}
\end{flalign}
Substituting \cref{eq: 23}, \cref{eq: 24} and \cref{eq: 25} into \cref{eq: 22}, we have
\begin{flalign}
	&\ltwo{\langle {\rm\Phi} (\tau,{\rm\Theta}_0), \widehat{\rm\Theta}-{\rm\Theta}_0 \rangle {\rm\Phi} (\tau,{\rm\Theta}_0) - \left(\sum_{h=1}^{H}f(x^\tau_h,\hat{\theta}_h) \right) {\rm\Phi}(\tau,\widehat{\rTheta})}\nonumber\\
	&\leq (H^2\sqrt{N/\lambda_1}) \mathcal{O}\left( C^2_\phi \left( \frac{H\sqrt{N/\lambda_1}}{\sqrt{m}}\right)^{1/3}\sqrt{\log m} \right) + \mathcal{O}\left(C^2_\phi H^{3/2} \left(\frac{H^4N^2/\lambda^2_1}{\sqrt{m}}\right)^{1/3}\sqrt{\log m}\right)\nonumber\\
	&\leq  \mathcal{O}\left( \frac{C^2_\phi H^{17/6} N^{2/3}\sqrt{\log(m)}}{m^{1/6}\lambda_1^{2/3}} \right).\label{eq: 26}
\end{flalign} 
Then, substituting \cref{eq: 26} into \cref{eq: 27}, we have the following holds with probability at least $1-N^{-2}H^{-4}$
\begin{flalign*}
	&\ltwo{{\rm\Sigma}({\rm\Theta}_0)(\widehat{\rTheta} - \overline{\rTheta})}\nonumber\\
	&\leq NH \cdot  \sqrt{H}\cdot \mathcal{O}\left( C_\phi \left( \frac{H\sqrt{N/\lambda_1}}{\sqrt{m}}\right)^{1/3}\sqrt{\log m} \right) +  N\cdot\mathcal{O}\left( \frac{C^2_\phi H^{17/6} N^{2/3}\sqrt{\log(m)}}{m^{1/6}\lambda_1^{2/3}} \right)\nonumber\\
	&\leq \mathcal{O}\left( \frac{C^2_\phi H^{17/6} N^{5/3}\sqrt{\log(m)}}{m^{1/6}\lambda_1^{2/3}} \right),
\end{flalign*}
where we use the fact that $r(\tau)\leq H$. We then proceed to bound the term $	\ltwo{\widehat{\rm\Theta}-{\rm\Theta}_0}$ as follows
\begin{flalign}
	\ltwo{\widehat{\rm\Theta}-\overline{\rm\Theta}}&= \ltwo{{\rm\Sigma}^{-1}({\rm\Theta}_0){\rm\Sigma}({\rm\Theta}_0)(\widehat{\rm\Theta}-\overline{\rm\Theta})}\nonumber\\
	&\leq \ltwo{{\rm\Sigma}^{-1}({\rm\Theta}_0)}\ltwo{{\rm\Sigma}({\rm\Theta}_0)(\widehat{\rm\Theta}-\overline{\rm\Theta})}\nonumber\\
	&\leq \lambda_1^{-1} \ltwo{{\rm\Sigma}({\rm\Theta}_0)(\widehat{\rm\Theta}-\overline{\rm\Theta})}\nonumber\\
	&\leq \mathcal{O}\left( \frac{C^2_\phi H^{17/6} N^{5/3}\sqrt{\log(m)}}{m^{1/6}\lambda_1^{5/3}} \right).\label{eq: 30}
\end{flalign}
Substituting \cref{eq: 30} into \cref{eq: 29}, we can bound $(ii)$ as follows
\begin{flalign}
	(ii) \leq \mathcal{O}\left( \frac{C^3_\phi H^{17/6} N^{5/3}\sqrt{\log(m)}}{m^{1/6}\lambda_1^{5/3}} \right).\label{eq: 31}
\end{flalign}
Taking summation of the upper bounds of $(i)$ in \cref{eq: 28} and $(ii)$ in \cref{eq: 31}, we have
\begin{flalign}
	\lone{\widehat{R}_h(x) - \overline{R}_h(x)} &\leq (i) + (ii) \nonumber\\
	&\leq \mathcal{O}\left(C_\phi \left(\frac{N^2H^4}{\lambda^2_1\sqrt{m}}\right)^{1/3}\sqrt{\log m}\right) + \mathcal{O}\left( \frac{C^3_\phi H^{17/6} N^{5/3}\sqrt{\log(m)}}{m^{1/6}\lambda_1^{5/3}} \right)\nonumber\\
	&\leq \mathcal{O}\left( \frac{C^3_\phi H^{17/6} N^{5/3}\sqrt{\log(m)}}{m^{1/6}\lambda_1^{5/3}} \right).\label{eq: 39}
\end{flalign}

\textbf{Step III.} In this step, we show that the bonus term $b_{r,h}(\cdot,\widehat{\rm\Theta})$ in \Cref{alg1} can be well approximated by $b_{r,h}(\cdot,{\rm\Theta}_0)$. According to the definition of $b_{r,h}(\cdot,{\rm\Theta})$, we have
\begin{flalign}
&\lone{ b_{r,h}(x,\widehat{\rm\Theta}) - b_{r,h}(x,{\rm\Theta}_0)} \nonumber\\
&\quad =\lone{\left[{\rm\Phi}_h(x,\widehat{\rm\Theta})^\top {\rm\Sigma}^{-1}(\widehat{\rm\Theta}) {\rm\Phi}_h(x,\widehat{\rm\Theta})\right]^{1/2} - \left[{\rm\Phi}_h(x,{\rm\Theta}_0)^\top {\rm\Sigma}^{-1}({\rm\Theta}_0) {\rm\Phi}_h(x,{\rm\Theta}_0)\right]^{1/2}}\nonumber\\
&\quad \leq \lone{ {\rm\Phi}_h(x,\widehat{\rm\Theta})^\top {\rm\Sigma}^{-1}(\widehat{\rm\Theta}) {\rm\Phi}_h(x,\widehat{\rm\Theta}) - {\rm\Phi}_h(x,{\rm\Theta}_0)^\top {\rm\Sigma}^{-1}({\rm\Theta}_0) {\rm\Phi}_h(x,{\rm\Theta}_0)}^{1/2},\label{eq: 32}
\end{flalign}
where the last inequality follows from the fact that $\lone{\sqrt{x}-\sqrt{y}}\leq \sqrt{\lone{x-y}}$. We then proceed to bound the term $\lone{ {\rm\Phi}_h(x,\widehat{\rm\Theta})^\top {\rm\Sigma}^{-1}(\widehat{\rm\Theta}) {\rm\Phi}_h(x,\widehat{\rm\Theta}) - {\rm\Phi}_h(x,{\rm\Theta}_0)^\top {\rm\Sigma}^{-1}({\rm\Theta}_0) {\rm\Phi}_h(x,{\rm\Theta}_0)}$ as follows
\begin{flalign}
&\lone{{\rm\Phi}_h(x,\widehat{\rm\Theta})^\top {\rm\Sigma}^{-1}(\widehat{\rm\Theta}) {\rm\Phi}_h(x,\widehat{\rm\Theta}) - {\rm\Phi}_h(x,{\rm\Theta}_0)^\top {\rm\Sigma}^{-1}({\rm\Theta}_0) {\rm\Phi}_h(x,{\rm\Theta}_0)}\nonumber\\
&\quad=\lone{[{\rm\Phi}_h(x,\widehat{\rm\Theta}) - {\rm\Phi}_h(x,{\rm\Theta}_0)] ^\top {\rm\Sigma}^{-1}(\widehat{\rm\Theta}) {\rm\Phi}_h(x,\widehat{\rm\Theta})} + \lone{{\rm\Phi}_h(x,{\rm\Theta}_0)^\top ({\rm\Sigma}^{-1}(\widehat{\rm\Theta}) - {\rm\Sigma}^{-1}({\rm\Theta}_0)) {\rm\Phi}_h(x,\widehat{\rm\Theta})}\nonumber\\
&\quad\quad + \lone{{\rm\Phi}_h(x,{\rm\Theta}_0)^\top {\rm\Sigma}^{-1}({\rm\Theta}_0) ({\rm\Phi}_h(x,\widehat{\rm\Theta}) - {\rm\Phi}_h(x,{\rm\Theta}_0))}\nonumber\\
&\quad\leq[{\rm\Phi}_h(x,\widehat{\rm\Theta}) - {\rm\Phi}_h(x,{\rm\Theta}_0)] ^\top {\rm\Sigma}^{-1}(\widehat{\rm\Theta}) {\rm\Phi}_h(x,\widehat{\rm\Theta}) + \lone{{\rm\Phi}_h(x,{\rm\Theta}_0)^\top ({\rm\Sigma}^{-1}(\widehat{\rm\Theta}) - {\rm\Sigma}^{-1}({\rm\Theta}_0)) {\rm\Phi}_h(x,\widehat{\rm\Theta})}\nonumber\\
&\quad\quad + \lone{{\rm\Phi}_h(x,{\rm\Theta}_0)^\top {\rm\Sigma}^{-1}({\rm\Theta}_0) ({\rm\Phi}_h(x,\widehat{\rm\Theta}) - {\rm\Phi}_h(x,{\rm\Theta}_0))}\nonumber\\
&\quad\leq \ltwo{{\rm\Phi}_h(x,\widehat{\rm\Theta}) - {\rm\Phi}_h(x,{\rm\Theta}_0)} \ltwo{{\rm\Sigma}^{-1}(\widehat{\rm\Theta})} \ltwo{{\rm\Phi}_h(x,\widehat{\rm\Theta})} \nonumber\\
&\quad\quad+ \ltwo{{\rm\Phi}_h(x,{\rm\Theta}_0)} \ltwo{{\rm\Sigma}^{-1}(\widehat{\rm\Theta}) - {\rm\Sigma}^{-1}({\rm\Theta}_0)} \ltwo{{\rm\Phi}_h(x,\widehat{\rm\Theta})}\nonumber\\
&\quad\quad + \ltwo{{\rm\Phi}_h(x,{\rm\Theta}_0)} \ltwo{{\rm\Sigma}^{-1}({\rm\Theta}_0)} \ltwo{{\rm\Phi}_h(x,\widehat{\rm\Theta}) - {\rm\Phi}_h(x,{\rm\Theta}_0)}\nonumber\\
&\quad = \ltwo{{\phi}(x,\widehat{\theta}_h) - {\phi}(x,{\theta}_0)} \ltwo{{\rm\Sigma}^{-1}(\widehat{\rm\Theta})} \ltwo{{\phi}(x,\widehat{\theta}_h)} \nonumber\\
&\quad\quad+ \ltwo{{\phi}(x,{\theta}_0)} \ltwo{{\rm\Sigma}^{-1}(\widehat{\rm\Theta}) ({\rm\Sigma}(\widehat{\rm\Theta}) - {\rm\Sigma}({\rm\Theta}_0)) {\rm\Sigma}^{-1}({\rm\Theta}_0)} \ltwo{{\phi}(x,\widehat{\theta}_h)}\nonumber\\
&\quad\quad + \ltwo{{\phi}(x,{\theta}_0)} \ltwo{{\rm\Sigma}^{-1}({\rm\Theta}_0)} \ltwo{{\phi}(x,\widehat{\theta}_h) - {\phi}(x,{\theta}_0)}\nonumber\\
&\quad \leq \ltwo{{\phi}(x,\widehat{\theta}_h) - {\phi}(x,{\theta}_0)} \ltwo{{\rm\Sigma}^{-1}(\widehat{\rm\Theta})} \ltwo{{\phi}(x,\widehat{\theta}_h)} \nonumber\\
&\quad\quad+ \ltwo{{\phi}(x,{\theta}_0)} \ltwo{{\rm\Sigma}^{-1}(\widehat{\rm\Theta})} \ltwo{{\rm\Sigma}(\widehat{\rm\Theta}) - {\rm\Sigma}({\rm\Theta}_0)} \ltwo{{\rm\Sigma}^{-1}({\rm\Theta}_0)} \ltwo{{\phi}(x,\widehat{\theta}_h)}\nonumber\\
&\quad\quad + \ltwo{{\phi}(x,{\theta}_0)} \ltwo{{\rm\Sigma}^{-1}({\rm\Theta}_0)} \ltwo{{\phi}(x,\widehat{\theta}_h) - {\phi}(x,{\theta}_0)}\nonumber\\
&\quad \leq \frac{1}{\lambda_1}\ltwo{{\phi}(x,\widehat{\theta}_h) - {\phi}(x,{\theta}_0)}  \ltwo{{\phi}(x,\widehat{\theta}_h)} + \frac{1}{\lambda^2_1} \ltwo{{\phi}(x,{\theta}_0)} \ltwo{{\rm\Sigma}(\widehat{\rm\Theta}) - {\rm\Sigma}({\rm\Theta}_0)} \ltwo{{\phi}(x,\widehat{\theta}_h)}\nonumber\\
&\quad\quad + \frac{1}{\lambda_1}\ltwo{{\phi}(x,{\theta}_0)} \ltwo{{\phi}(x,\widehat{\theta}_h) - {\phi}(x,{\theta}_0)},\label{eq: 33}
\end{flalign}
where the last inequality follows from the fact that $\ltwo{{\rm\Sigma}({\rm\Theta})}\geq \lambda_1$ for any ${\rm\Theta}\in\mR^{2mdH}$. For ${\rm\Sigma}(\widehat{\rm\Theta}) - {\rm\Sigma}({\rm\Theta}_0)$, we have
\begin{flalign*}
{\rm\Sigma}(\widehat{\rm\Theta}) - {\rm\Sigma}({\rm\Theta}_0) &= \sum_{\tau\in\mcd}\left[{\rm\Phi}(\tau,\widehat{\rm\Theta}){\rm\Phi}(\tau,\widehat{\rm\Theta})^\top - {\rm\Phi}(\tau,{\rm\Theta}_0){\rm\Phi}(\tau,{\rm\Theta}_0)^\top\right]\nonumber\\
&=\sum_{\tau\in\mcd} \left[{\rm\Phi}(\tau,\widehat{\rm\Theta})({\rm\Phi}(\tau,\widehat{\rm\Theta}) - {\rm\Phi}(\tau,{\rm\Theta}_0))^\top + ({\rm\Phi}(\tau,\widehat{\rm\Theta}) - {\rm\Phi}(\tau,{\rm\Theta}_0)){\rm\Phi}(\tau,{\rm\Theta}_0)^\top\right],
\end{flalign*}
which implies
\begin{flalign}
&\ltwo{	{\rm\Sigma}(\widehat{\rm\Theta}) - {\rm\Sigma}({\rm\Theta}_0)}\nonumber\\
&\quad \leq \sum_{\tau\in\mcd} \left[\ltwo{{\rm\Phi}(\tau,\widehat{\rm\Theta})} \ltwo{{\rm\Phi}(\tau,\widehat{\rm\Theta}) - {\rm\Phi}(\tau,{\rm\Theta}_0)} + \ltwo{{\rm\Phi}(\tau,\widehat{\rm\Theta}) - {\rm\Phi}(\tau,{\rm\Theta}_0)} \ltwo{{\rm\Phi}(\tau,{\rm\Theta}_0)} \right].\label{eq: 36}
\end{flalign}
By definition of ${\rm\Phi}(\tau,{\rm\Theta})$, we have the followings hold for any ${\rm\Theta}, \widetilde{\rm\Theta}\in\mR^{2mdH}$
\begin{flalign}
\ltwo{{\rm\Phi}(\tau,\widehat{\rm\Theta})} &= \sqrt{\sum_{h\in[H]} \ltwo{{\phi}(x^\tau_h,\widehat{\theta}_h)}^2},\label{eq: 34}\\
\ltwo{{\rm\Phi}(\tau,\widehat{\rm\Theta}) - {\rm\Phi}(\tau,{\rm\Theta}_0) } &= \sqrt{\sum_{h\in[H]} \ltwo{{\phi}(x^\tau_h,{\theta}_h) - {\phi}(x^\tau_h,\tilde{\theta}_h) }^2}.\label{eq: 35}
\end{flalign}
Applying \Cref{lemma4} to \cref{eq: 34} and \cref{eq: 35}, we have the followings hold with probability at least $1-N^{-2}H^{-4}$
\begin{flalign*}
\ltwo{{\rm\Phi}(\tau,\widehat{\rm\Theta})} & \leq C_\phi\sqrt{H},\nonumber\\
\ltwo{{\rm\Phi}(\tau,\widehat{\rm\Theta}) - {\rm\Phi}(\tau,{\rm\Theta}_0) } &\leq \mathcal{O}\left(   \frac{C_\phi H^{5/6} N^{1/6} \sqrt{\log m} }{m^{1/6}\lambda_1^{1/6}} \right).
\end{flalign*}
Substituting the above two inequalities into \cref{eq: 36} yields
\begin{flalign}
\ltwo{	{\rm\Sigma}(\widehat{\rm\Theta}) - {\rm\Sigma}({\rm\Theta}_0)} \leq \mathcal{O}\left(   \frac{C^2_\phi H^{4/3} N^{1/6} \sqrt{\log m} }{m^{1/6}\lambda_1^{1/6}} \right).\label{eq: 37}
\end{flalign}
Finally, combining \cref{eq: 37} and \cref{eq: 13} and \cref{eq: 14} in \Cref{lemma4}, we can bound the right hand side of \cref{eq: 33} as
\begin{flalign}
&\lone{{\rm\Phi}_h(x,\widehat{\rm\Theta})^\top {\rm\Sigma}^{-1}(\widehat{\rm\Theta}) {\rm\Phi}_h(x,\widehat{\rm\Theta}) - {\rm\Phi}_h(x,\overline{\rm\Theta})^\top {\rm\Sigma}^{-1}(\overline{\rm\Theta}) {\rm\Phi}_h(x,\overline{\rm\Theta})}\nonumber\\
&\quad \leq  \frac{1}{\lambda_1}\ltwo{{\phi}(x,\widehat{\theta}_h) - {\phi}(x,{\theta}_0)}  \ltwo{{\phi}(x,\widehat{\theta}_h)} + \frac{1}{\lambda^2_1} \ltwo{{\phi}(x,{\theta}_0)} \ltwo{{\rm\Sigma}(\widehat{\rm\Theta}) - {\rm\Sigma}({\rm\Theta}_0)} \ltwo{{\phi}(x,\widehat{\theta}_h)}\nonumber\\
&\quad\quad + \frac{1}{\lambda_1}\ltwo{{\phi}(x,{\theta}_0)} \ltwo{{\phi}(x,\widehat{\theta}_h) - {\phi}(x,{\theta}_0)}\nonumber\\
&\quad \leq  \mathcal{O}\left(  \frac{C^2_\phi H^{1/3} N^{1/6} \sqrt{\log m} }{ m^{1/6} \lambda_1^{7/6}}  \right) + \mathcal{O}\left(   \frac{C^4_\phi H^{4/3} N^{1/6} \sqrt{\log m} }{m^{1/6}\lambda_1^{13/6}} \right)\nonumber\\
&\quad = \mathcal{O}\left(   \frac{C^4_\phi H^{4/3} N^{1/6} \sqrt{\log m} }{m^{1/6}\lambda_1^{13/6}} \right).\label{eq: 38}
\end{flalign}
Substituting \cref{eq: 38} into \cref{eq: 32}, we have the following holds with probability at least $1-N^{-2}H^{-4}$
\begin{flalign}
&\lone{ b_{r,h}(x,\widehat{\rm\Theta}) - b_{r,h}(x,{\rm\Theta}_0)} \nonumber\\
&\leq \lone{ {\rm\Phi}_h(x,\widehat{\rm\Theta})^\top {\rm\Sigma}^{-1}(\widehat{\rm\Theta}) {\rm\Phi}_h(x,\widehat{\rm\Theta}) - {\rm\Phi}_h(x,{\rm\Theta}_0)^\top {\rm\Sigma}^{-1}({\rm\Theta}_0) {\rm\Phi}_h(x,{\rm\Theta}_0)}^{1/2}\nonumber\\
&\leq \mathcal{O}\left(   \frac{C^2_\phi H^{2/3} N^{1/12} (\log m)^{1/4}}{m^{1/12}\lambda_1^{13/12}} \right).\label{eq: 109}
\end{flalign}

{\bf Step IV.} In Steps I and II, we show that the mean of the real reward $R_h(\cdot)$ can be well approximated by a linear function $\widetilde{R}_h(\cdot)$ with feature $\phi(\cdot,\theta_0)$ and our learned reward $\widehat{R}_h(\cdot)$ can be well approximated by a linear function $\overline{R}_h(\cdot)$ with feature $\phi(\cdot,\theta_0)$. In this step, we want to show that the reward estimation error $\lone{\widehat{R}_h(\cdot)-R_h(\cdot)}$ is approximately $\beta_1\cdot b_{r.h}(x,{\rm\Theta}_0)$ with an approximately chosen $\beta_1$.

Recall that $\widetilde{R}_h(\cdot)=\langle \phi(\cdot, \theta_0), \tilde{\theta}_h-\theta_0 \rangle $ and $ \overline{R}_h(\cdot) = \langle \phi(\cdot,\theta_0), \bar{\theta}_h - \theta_0 \rangle$. Considering the difference between $\overline{R}_h(\cdot)$ and $\widetilde{R}_h(\cdot)$, we have
\begin{flalign}
	\overline{R}_h(x) - \widetilde{R}_h(x) = \langle \phi(x,\theta_0), \bar{\theta}_h - \tilde{\theta}_h \rangle = \langle {\rm\Phi}_h(x,{\rm\Theta}_0), \overline{\rm\Theta} - \widetilde{\rm\Theta} \rangle,\label{eq: 73}
\end{flalign}
where the last equality follows from the definition of ${\rm\Phi}_h(\cdot,{\rm\Theta})$. By \cref{eq: 19}, we have
\begin{flalign}
	\overline{{\rm\Theta}} - {\rm\Theta}_0 = {\rm\Sigma}({\rm\Theta}_0)^{-1}\sum_{\tau\in\mcd} r(\tau){\rm\Phi}(\tau,{\rm\Theta}_0).\label{eq: 40}
\end{flalign}
By the definition of ${\rm\Sigma}({\rm\Theta})$, we have
\begin{flalign}
	\widetilde{\rm\Theta} - {\rm\Theta}_0 = {\rm\Sigma}({\rm\Theta}_0)^{-1}\left[ \lambda_1\left( \widetilde{\rm\Theta} - {\rm\Theta}_0 \right) + \left(\sum_{\tau\in\mcd}{\rm\Phi}(\tau,{\rm\Theta}_0){\rm\Phi}(\tau,{\rm\Theta}_0)^\top\right)\left( \widetilde{\rm\Theta} - {\rm\Theta}_0 \right)  \right].\label{eq: 41}
\end{flalign}
Subtracting \cref{eq: 41} from \cref{eq: 40}, we have
\begin{flalign}
	\overline{{\rm\Theta}} - \widetilde{\rm\Theta} = -\lambda_1 {\rm\Sigma}({\rm\Theta}_0)^{-1}\left( \widetilde{\rm\Theta} - {\rm\Theta}_0 \right) + {\rm\Sigma}({\rm\Theta}_0)^{-1} \sum_{\tau\in\mcd} {\rm\Phi}(\tau,{\rm\Theta}_0)\left[r(\tau) - \langle {\rm\Phi}(\tau,{\rm\Theta}_0), \widetilde{\rm\Theta} - {\rm\Theta}_0 \rangle \right].\label{eq: 42}
\end{flalign}
Taking inter product of both sides of \cref{eq: 42} with vector ${\rm\Phi}_h(x,{\rm\Theta}_0)$ and using the fact that $R(\tau)=\sum_{h\in[H]}R_h(x^\tau_h)$ and $\langle {\rm\Phi}(\tau,{\rm\Theta}_0), \widetilde{\rm\Theta} - {\rm\Theta}_0 \rangle = \sum_{h\in[H]} \langle \phi(x^\tau_h, \theta_0), \tilde{\theta}_h - \theta_0 \rangle$, we have
\begin{flalign}
	&\langle {\rm\Phi}_h(x,{\rm\Theta}_0), \overline{\rm\Theta} - \widetilde{\rm\Theta} \rangle\nonumber\\
	&\quad=-\lambda_1 {\rm\Phi}_h(x,{\rm\Theta}_0)^\top{\rm\Sigma}({\rm\Theta}_0)^{-1/2}{\rm\Sigma}({\rm\Theta}_0)^{-1/2}\left( \widetilde{\rm\Theta} - {\rm\Theta}_0 \right) \nonumber\\
	&\quad\quad+ {\rm\Phi}_h(x,{\rm\Theta}_0)^\top{\rm\Sigma}({\rm\Theta}_0)^{-1/2}{\rm\Sigma}({\rm\Theta}_0)^{-1/2} \left(\sum_{\tau\in\mcd} {\rm\Phi}(\tau,{\rm\Theta}_0)\left(r(\tau) - R(\tau)\right) \right)\nonumber\\
	&\quad\quad+ {\rm\Phi}_h(x,{\rm\Theta}_0)^\top{\rm\Sigma}({\rm\Theta}_0)^{-1/2}{\rm\Sigma}({\rm\Theta}_0)^{-1/2} \left(\sum_{\tau\in\mcd} {\rm\Phi}(\tau,{\rm\Theta}_0)\left[ \sum_{h\in[H]} \left(R_h(x^\tau_h) - \langle \phi(x^\tau_h, \theta_0), \tilde{\theta}_h - \theta_0 \rangle \right) \right]\right).\label{eq: 43}
\end{flalign}
Recall that $\widetilde{R}_h(x^\tau_h) = \langle \phi(x^\tau_h, \theta_0), \tilde{\theta}_h - \theta_0 \rangle$, and \cref{eq: 43} implies that
\begin{flalign}
	&\lone{\langle {\rm\Phi}_h(x,{\rm\Theta}_0), \overline{\rm\Theta} - \widetilde{\rm\Theta} \rangle} \nonumber\\
	&\leq \sqrt{\lambda_1} \ltwo{{\rm\Phi}_h(x,{\rm\Theta}_0)^\top{\rm\Sigma}({\rm\Theta}_0)^{-1/2}} \ltwo{\widetilde{\rm\Theta} - {\rm\Theta}_0}  \nonumber\\
	&\quad +  \ltwo{ {\rm\Phi}_h(x,{\rm\Theta}_0)^\top{\rm\Sigma}({\rm\Theta}_0)^{-1/2}} \lsigma{\sum_{\tau\in\mcd} {\rm\Phi}(\tau,{\rm\Theta}_0)\varepsilon(\tau)} \nonumber\\
	&\quad + \frac{1}{\sqrt{\lambda_1}} \ltwo{ {\rm\Phi}_h(x,{\rm\Theta}_0)^\top{\rm\Sigma}({\rm\Theta}_0)^{-1/2}} \left(\sum_{\tau\in\mcd} \ltwo{{\rm\Phi}(\tau,{\rm\Theta}_0)} \sum_{h\in[H]} \lone{R_h(x^\tau_h) - \widetilde{R}_h(x^\tau_h)  } \right),\label{eq: 44}
\end{flalign}
where we denote $\epsilon(\tau) = r(\tau) - R(\tau)$ and use the fact that $\ltwo{{\rm\Sigma}({\rm\Theta})^{-1/2}}\leq 1/\sqrt{\lambda_1}$ for any ${\rm\Theta}\in\mR^{2mdH}$. By the definition of $\widetilde{\rm\Theta}$ in Step I, we have
\begin{flalign}
	\ltwo{\widetilde{\rm\Theta} - {\rm\Theta}} &= \sqrt{\sum_{h\in[H],r\in[m]}\ltwo{ \tilde{\theta}_{h,r} - \theta^0_{h,r} }^2} \nonumber\\
	&= \sqrt{\sum_{h\in[H],r\in[m]}\ltwo{ \ell_{r} }^2} \nonumber\\
	&\leq r_2\sqrt{H/d}.\label{eq: 46}
\end{flalign}
By \Cref{lemma4} and \cref{eq: 47}, we have the followings hold with probability at least $1-N^{-2}H^{-4}$
\begin{flalign}
	\ltwo{{\rm\Phi}(\tau,{\rm\Theta}_0)} &\leq C_\phi\sqrt{H},\label{eq: 49}\\
	\lone{R_h(x^\tau_h) -  \widetilde{R}_h(x^\tau_h) } &\leq \frac{2(L_\sigma a_2 + C^2_\sigma a^2_2)\sqrt{\log N^2H^5}}{\sqrt{m}}.\label{eq: 48}
\end{flalign}

Substituting \cref{eq: 46}, \cref{eq: 49} and \cref{eq: 48} into \cref{eq: 44} and using the fact that $b_{r,h}(x,{\rm\Theta}_0) = \ltwo{{\rm\Phi}_h(x,{\rm\Theta}_0)^\top{\rm\Sigma}({\rm\Theta}_0)^{-1/2}}$, we have
\begin{flalign}
&\lone{\langle {\rm\Phi}_h(x,{\rm\Theta}_0), \overline{\rm\Theta} - \widetilde{\rm\Theta} \rangle} \nonumber\\
&\leq \left(a_2\sqrt{\frac{\lambda_1H}{d}} + \frac{2(L_\sigma a_2 + C^2_\sigma a^2_2)C_\phi N H^{3/2}\sqrt{\log HN}}{\sqrt{\lambda_1m}}+ \lsigma{\sum_{\tau\in\mcd} {\rm\Phi}(\tau,{\rm\Theta}_0)\varepsilon(\tau)} \right) b_{r,h}(x,{\rm\Theta}_0)  \label{eq: 50}.
\end{flalign}
Given that the events in \cref{eq: 49} and \cref{eq: 48} occur, applying \cref{eq: 70} in \Cref{lemma7}, we have the following holds with probability at least $1-N^{-2}H^{-4}$
\begin{flalign}
	&\lsigma{\sum_{\tau\in\mcd} {\rm\Phi}(\tau,{\rm\Theta}_0)\varepsilon(\tau)}^2\nonumber\\
	&\quad\leq H^2  \log\det(I+K^r_N/\lambda_1) + H^2 N(\lambda_1-1) + 4H^2\log(NH^2),\label{eq: 71}
\end{flalign}
where $K^r_N\in\mR^{N\times N}$ is the Gram matrix defined as
\begin{flalign*}
	K^r_N = [K_H(\tau_i, \tau_j)]_{i,j\in[N]}\in\mR^{N\times N}.
\end{flalign*}
Combining \cref{eq: 50} and \cref{eq: 71} and letting $\lambda_1 = 1+N^{-1}$ and $m$ be sufficiently large such that
\begin{flalign*}
	\frac{2(L_\sigma a_2 + C^2_\sigma a^2_2)C_\phi N H^{3/2}\sqrt{\log HN}}{\sqrt{\lambda_1m}}\leq  a_2\sqrt{\frac{\lambda_1H}{d}},
\end{flalign*}
we have the following holds with probability at least $1-N^{-2}H^{-2}$
\begin{flalign}
&\lone{\langle {\rm\Phi}_h(x,{\rm\Theta}_0), \overline{\rm\Theta} - \widetilde{\rm\Theta} \rangle} \nonumber\\
&\leq \left( 2a_2\sqrt{\frac{\lambda_1H}{d}} + \sqrt{H^2  \log\det\left(I+\frac{K^r_N}{\lambda_1}\right) + H^2+ 4H^2\log(NH^2)} \right) b_{r,h}(x,{\rm\Theta}_0)\nonumber\\
&\leq \underbrace{H\left(\frac{4 a^2_2\lambda_1}{d} +   2\log\det\left(I+\frac{K^r_N}{\lambda_1}\right) + 10\log(NH^2) \right)^{1/2}}_{\beta_1} b_{r,h}(x,{\rm\Theta}_0),\label{eq: 72}
\end{flalign}
where in the last inequality we use the fact that $a+b\leq\sqrt{2(a^2+b^2)}$.
Substituting \cref{eq: 72} into \cref{eq: 73}, we have the following holds with probability at least $1-N^{-2}H^{-4}$
\begin{flalign}
	\lone{\overline{R}_h(x) - \widetilde{R}_h(x)}\leq \beta_1\cdot b_{r,h}(x,{\rm\Theta}_0),
\end{flalign}
where
\begin{flalign*}
	\beta_1 = H\left(\frac{4 a^2_2\lambda_1}{d} +   2\log\det\left(I+\frac{K^r_N}{\lambda_1}\right) + 10\log(NH^2) \right)^{1/2}.
\end{flalign*}
Next, we proceed to bound the reward estimation error $\lone{R_h(x) - \widehat{R}_h(x)}$. By the triangle inequality, we have
\begin{flalign}
	\lone{R_h(x) - \widehat{R}_h(x)} &= \lone{R_h(x) - \widetilde{R}_h(x) + \widetilde{R}_h(x) - \overline{R}_h(x) + \overline{R}_h(x) - \widehat{R}_h(x)}\nonumber\\
	&\leq \lone{R_h(x) - \widetilde{R}_h(x)} + \lone{\widetilde{R}_h(x) - \overline{R}_h(x)} + \lone{\overline{R}_h(x) - \widehat{R}_h(x)}\nonumber\\
	&\overset{(i)}{\leq} \frac{2(L_\sigma a_2 + C^2_\sigma a^2_2)\sqrt{\log(HN)}}{\sqrt{m}} +  \mathcal{O}\left( \frac{C^3_\phi H^{17/6} N^{5/3}\sqrt{\log(m)}}{m^{1/6}\lambda_1^{5/3}} \right)\nonumber\\
	&\quad +  \beta_1\cdot b_{r,h}(x,{\rm\Theta}_0)\nonumber\\
	&\overset{(ii)}{\leq} \mathcal{O}\left( \frac{H^{17/6} N^{5/3}\sqrt{\log(m)}}{m^{1/6}} \right) +  \beta_1\cdot b_{r,h}(x,{\rm\Theta}_0).\label{eq: 119}
\end{flalign}
where $(i)$ follows from \cref{eq: 47} and \cref{eq: 39} and $(ii)$ follows from the fact that $\lambda_1=1+1/N$ and $L_\sigma,C_\sigma,a_2, C_\phi = \mathcal{O}(1)$.

\subsection{Uncertainty of Estimated Transition Value Function $({\widehat{\mP}_h\widehat{V}_{h+1}   })(\cdot)$}
In this subsection, we aim to bound the estimation error of the transition value function $\lone{({\widehat{\mP}_h\widehat{V}})(\cdot) - ({{\mP}_h\widehat{V}})(\cdot)}$.
For each $h\in[H]$, since $\widehat{w}_h$ is the global minimizer of the loss function $L^h_v(w_h)$ defined in \cref{eq: 74}, we have
\begin{flalign}
L^h_v(\widehat{w}_h) &= \sum_{\tau\in\mcd}\left(\widehat{V}_{h+1}(s^\tau_{h+1}) - f(x^\tau_h,\widehat{w}_h)\right)^2+\lambda_2\cdot \ltwo{\widehat{w}_h-w_0}^2\nonumber\\
&\leq L^h_v(\widehat{w}_0) = \sum_{\tau\in\mcd}\left(\widehat{V}_{h+1}(s^\tau_{h+1}) - f(x^\tau_h,{w}_0)\right)^2 \overset{(i)}{=} \sum_{\tau\in\mcd}\left(\widehat{V}_{h+1}(s^\tau_{h+1})\right)^2 \overset{(ii)}{\leq} NH^2,\label{eq: 75}
\end{flalign}
where $(i)$ follows from the fact that $f(x,w_0)=0$ for all $x\in\mcx$ and $(ii)$ follows from the fact that $\widehat{V}_{h}(s)\leq H$ for any $h\in[H]$, $s\in\mcs$, and $\lone{\mcd}=N$. Note that \cref{eq: 75} implies
\begin{flalign}
\ltwo{\widehat{w}_h-w_0}^2\leq NH^2/\lambda_2,\quad\forall h\in[H].\label{eq: 76}
\end{flalign}
Hence, each $\widehat{w}_h$ belongs to the Euclidean ball $\mcB_w=\{w\in\mR^{2md}: \ltwo{w-w_0}\leq H\sqrt{N/\lambda_2} \}$, where $\lambda_2$ does not depend on the network width $m$. Since the radius of $\mcB_w$ does not depend on $m$, when $m$ is sufficient large, it can be shown that $f(\cdot,w)$ is close to its linearization at $w_0$, i.e.,
\begin{flalign*}
f(\cdot,w)\approx \langle \phi(\cdot,w_0), w-w_0\rangle,\quad\forall w\in\mcB_\theta,
\end{flalign*}
where $\phi(\cdot,w)=\nabla_w f(\cdot,w)$. Furthermore, according to \Cref{ass1}, there exists a function $\ell_{A_1,A_2}:\mR^d\rightarrow\mR^d$ such that $(\mP_h \widehat{V}_{h+1}  )(\cdot)$ satisfies
\begin{flalign}\label{eq: 77}
(\mP_h \widehat{V}_{h+1}  )(x) = \int_{\mR^d}\sigma^\prime(\theta^\top x) \cdot x^\top \ell_v(w)dp(w),
\end{flalign}
where $\sup_w\ltwo{\ell_v(w)}\leq A_1$, $\sup_w(\ltwo{\ell_v(w)}/p(w))\leq A_2$ and $p$ is the density of $N(0,I_d/d)$. We then proceed to bound the difference between $({\widehat{\mP}_h\widehat{V}})(\cdot)$ and $({{\mP}_h\widehat{V}})(\cdot)$.

{\bf Step I.} In the first step, we show that the transition value function $({\mP}_h\widehat{V}_{h+1})(\cdot)$ can be well-approximated by a linear function with the feature vector $\phi(\cdot,\theta_0)$. \Cref{lemma3} in \Cref{sc: supplemma} implies that with probability at least $1-N^{-2}H^{-4}$ over the randomness of initialization $w_0$, for all $h\in[H]$, there exists a function $(\widetilde{\mP}_h\widehat{V}_{h+1})(\cdot):\mcx\rightarrow\mR$ satisfying
\begin{flalign}
\sup_{x\in\mcx}\lone{(\widetilde{\mP}_h\widehat{V}_{h+1})(x) - ({\mP}_h\widehat{V}_{h+1})(x)}\leq \frac{2(L_\sigma A_2 + C^2_\sigma A^2_2)\sqrt{\log(N^2H^5)}}{\sqrt{m}}, \label{eq: 78}
\end{flalign}
where $(\widetilde{\mP}_h\widehat{V}_{h+1})(\cdot)$ is a finite-width neural network which can be written as
\begin{flalign*}
(\widetilde{\mP}_h\widehat{V}_{h+1})(x) = \frac{1}{\sqrt{m}}\sum_{r=1}^{m}\sigma^\prime(w^{\top}_{0,r} x)\cdot x^\top \ell^v_r,
\end{flalign*}
where $\ltwo{\ell^v_r}\leq A_2/\sqrt{dm}$ for all $r\in[m]$ and $w_0=[w_{0,1},\cdots,w_{0,m}]$ is generated via the symmetric initialization scheme. Following steps similar to those in \cref{eq: 79}, we can show that there exists a vector $\tilde{w}_h\in\mR^{2md}$ such that
\begin{flalign*}
	(\widetilde{\mP}_h\widehat{V}_{h+1})(\cdot) = \langle \phi(\cdot,w_0), \tilde{w}_h -w_0\rangle,
\end{flalign*}
where $\tilde{w}_h=[\tilde{w}^{\top}_{h,1},\cdots.\tilde{w}^{\top}_{h,2m}]^\top$, in which $\tilde{w}^{\top}_{h,r}=w_{0,r} + b_{0,r}\cdot \ell^v_r/\sqrt{2}$ for all $r\in\{1,\cdots,m\}$ and $\tilde{w}^{\top}_{h,r}=w_{0,r} + b_{0,r}\cdot \ell^v_{r-m}/\sqrt{2}$ for all $r\in\{m+1,\cdots,2m\}$. Moreover, since $\tilde{w}_{h,r} - w_{0,r} =  b_{0,r}\cdot \ell^v_r/\sqrt{2}$ or $ b_{0,r}\cdot \ell^v_{r-m}/\sqrt{2}$, we have 
\begin{flalign*}
\ltwo{\tilde{w}_h - w_0}\leq A_2\sqrt{2dm}.
\end{flalign*}

{\bf Step II.} In the second step, we show that with high probability, the estimation of the transition value function $ (\widehat{\mP}_h\widehat{V}_{h+1})(\cdot)$ in \Cref{alg1} can be well-approximated by its counterpart learned with a linear function with the feature $\phi(\cdot,\theta_0)$.

Consider the following least-square loss function
\begin{flalign}
\bar{L}^h_v(w_h) &= \sum_{\tau\in\mcd}\left(\widehat{V}_{h+1}(s^\tau_{h+1}) - \langle \phi(x^\tau_h,w_0), w_h-w_0  \rangle \right)^2+\lambda_2\cdot \ltwo{w_h-w_0}^2.
\end{flalign}
The global minimizer of $\bar{L}^h_v(w_h)$ is defined as
\begin{flalign}\label{eq: 80}
\overline{w}_h=\argmin_{w\in\mR^{2dm}}\bar{L}^h_v(w).
\end{flalign}
We define $(\overline{\mP}_h\widehat{V}_{h+1})(\cdot)=\langle \phi(\cdot,w_0), \overline{w}_h - w_0 \rangle$ for all $h\in[H]$. Then, in a manner similar to the construction of $\widehat{Q}_h(\cdot)$ in \Cref{alg1}, we combine $\overline{R}_h(\cdot)$ in \cref{eq: 16}, $b_{r,h}(\cdot, {\rm\Theta}_0)$, $\overline{\mP}_h\widehat{V}_{h+1}(\cdot)$ and $b_{v,h}(\cdot,w_0)$ to construct $\overline{Q}_h(\cdot):\mcx\rightarrow \mR$ as
\begin{flalign}
	\overline{Q}_h(\cdot) = \min\{ \overline{R}_h(\cdot)+ (\overline{\mP}_h\widehat{V}_{h+1})(\cdot) - \beta_1\cdot b_{r,h}(\cdot, {\rm\Theta}_0) - \beta_2\cdot b_{v,h}(\cdot,w_0) , H \}^{+}.\label{eq: 102}
\end{flalign}
Moreover, we define the estimated optimal state value function as
\begin{flalign}
	\overline{V}_h(\cdot) = \max_{a\in\mca} \overline{Q}_h(\cdot,a).\label{eq: 103}
\end{flalign}
We then proceed to bound the estimation error $\lone{(\widehat{\mP}_h\widehat{V}_{h+1})(x) - \overline{\mP}_h\widehat{V}_{h+1})(x)}$ as follows
\begin{flalign}
\lone{(\widehat{\mP}_h\widehat{V}_{h+1})(x) - \overline{\mP}_h\widehat{V}_{h+1})(x)} & = \lone{f(x,\widehat{w}_h) - \langle \phi(x,w_0), \overline{w}_h - w_0 \rangle}\nonumber\\
& = \lone{f(x,\widehat{w}_h) - \langle \phi(x,w_0), \widehat{w}_h - w_0 \rangle - \langle \phi(\cdot,w_0), \widehat{w}_h - \overline{w}_h \rangle}\nonumber\\
& \leq \lone{f(x,\widehat{w}_h) - \langle \phi(x,w_0), \widehat{w}_h - w_0 \rangle} + \lone{\langle \phi(\cdot,w_0), \widehat{w}_h - \overline{w}_h \rangle}\nonumber\\
& \leq \underbrace{\lone{f(x,\widehat{w}_h) - \langle \phi(\cdot,w_0), \widehat{w}_h - w_0 \rangle}}_{(i)} + \underbrace{\ltwo{ \phi(x,w_0)}\ltwo{\widehat{w}_h - \overline{w}_h}}_{(ii)}.\nonumber
\end{flalign}
We then bound the term $(i)$ and term $(ii)$ in the above inequality. According to \Cref{lemma4} and the fact that $\ltwo{\widehat{w}_h-w_0}\leq H\sqrt{N/\lambda_2}$, we have the followings hold with probability at least $1-N^{-2}H^{-4}$
\begin{flalign}
(i)&\leq \mathcal{O}\left(C_\phi \left(\frac{N^2H^4}{\lambda^2_2\sqrt{m}}\right)^{1/3}\sqrt{\log m}\right),\label{eq: 81}\\
(ii)&\leq  C_\phi \ltwo{\widehat{w}_h - \overline{w}_h}.\label{eq: 82}
\end{flalign}
We then proceed to bound $\ltwo{\widehat{w}_h - \overline{w}_h}$. Consider the minimization problem defined in \cref{eq: 3} and \cref{eq: 16}. By the first order optimality condition, we have
\begin{flalign}
\lambda_2\left( \widehat{w}_h - w_0 \right) &= \sum_{\tau\in\mcd}\left( \widehat{V}_{h+1}(s^\tau_{h+1}) -  f(x^\tau_h,\widehat{w}_h) \right){\phi}(x^\tau_h,\widehat{w}_h)\label{eq: 83}\\
\lambda_2\left( \overline{w}_h - w_0 \right) &= \sum_{\tau\in\mcd}\left( \widehat{V}_{h+1}(s^\tau_{h+1}) -  \langle \phi(x^\tau_h,w_0), \overline{w}_h - w_0 \rangle \right){\phi}(x^\tau_h,{w}_0)\label{eq: 84}.
\end{flalign}
Note that \cref{eq: 84} implies
\begin{flalign}\label{eq: 85}
{\rm\Lambda}_h(w_0)\left( \overline{w}_h - w_0 \right) = \sum_{\tau\in\mcd}\widehat{V}_{h+1}(s^\tau_{h+1}){\phi}(x^\tau_h,{w}_0).
\end{flalign}
Adding the term $\sum_{\tau\in\mcd} \langle {\phi}(x^\tau_h,{w}_0), \widehat{w}_h-{w}_0 \rangle {\phi}(x^\tau_h,{w}_0)$ on both sides of \cref{eq: 83} yields
\begin{flalign}\label{eq: 86}
{\rm\Lambda}_h(w_0)\left( \widehat{w}_h - w_0 \right) &= \sum_{\tau\in\mcd} \widehat{V}_{h+1}(s^\tau_{h+1}){\phi}(x^\tau_h,\widehat{w}_h) \nonumber\\
& \quad + \sum_{\tau\in\mcd} \left[ \langle {\phi}(x^\tau_h,{w}_0), \widehat{w}_h-{w}_0 \rangle {\phi}(x^\tau_h,{w}_0) - f(x^\tau_h,\widehat{w}_h) {\phi}(x^\tau_h,\widehat{w}_h) \right].
\end{flalign}
Then, by subtracting \cref{eq: 85} from \cref{eq: 86}, we have
\begin{flalign}
{\rm\Lambda}_h(w_0)( \widehat{w}_h -  \overline{w}_h)&=\sum_{\tau\in\mcd} \widehat{V}_{h+1}(s^\tau_{h+1}) \left( {\phi}(x^\tau_h,\widehat{w}_h) - {\phi}(x^\tau_h,{w}_0) \right)\nonumber\\
&\quad + \sum_{\tau\in\mcd} \left[ \langle {\phi}(x^\tau_h,{w}_0), \widehat{w}_h-{w}_0 \rangle {\phi}(x^\tau_h,{w}_0) - f(x^\tau_h,\widehat{w}_h) {\phi}(x^\tau_h,\widehat{w}_h) \right],\label{eq: 87}
\end{flalign}
which implies
\begin{flalign}
\ltwo{{\rm\Lambda}_h(w_0)( \widehat{w}_h -  \overline{w}_h)}&=\sum_{\tau\in\mcd} \widehat{V}_{h+1}(s^\tau_{h+1}) \ltwo{{\phi}(x^\tau_h,\widehat{w}_h) - {\phi}(x^\tau_h,{w}_0)} \nonumber\\
&\quad + \sum_{\tau\in\mcd}  \ltwo{\langle {\phi}(x^\tau_h,{w}_0), \widehat{w}_h-{w}_0 \rangle {\phi}(x^\tau_h,{w}_0) - f(x^\tau_h,\widehat{w}_h) {\phi}(x^\tau_h,\widehat{w}_h)}.\label{eq: 88}
\end{flalign}
To bound the term $ \ltwo{\langle {\phi}(x^\tau_h,{w}_0), \widehat{w}_h-{w}_0 \rangle {\phi}(x^\tau_h,{w}_0) - f(x^\tau_h,\widehat{w}_h) {\phi}(x^\tau_h,\widehat{w}_h)}$, we proceed as follows
\begin{flalign*}
	&\langle {\phi}(x^\tau_h,{w}_0), \widehat{w}_h-{w}_0 \rangle {\phi}(x^\tau_h,{w}_0) - f(x^\tau_h,\widehat{w}_h) {\phi}(x^\tau_h,\widehat{w}_h)\nonumber\\
	& = \langle {\phi}(x^\tau_h,{w}_0), \widehat{w}_h-{w}_0 \rangle ({\phi}(x^\tau_h,{w}_0) - {\phi}(x^\tau_h,\widehat{w}_h)) - (\langle {\phi}(x^\tau_h,{w}_0), \widehat{w}_h-{w}_0 \rangle -  f(x^\tau_h,\widehat{w}_h)) {\phi}(x^\tau_h,\widehat{w}_h)
\end{flalign*}
which implies
\begin{flalign}
	&\ltwo{\langle {\phi}(x^\tau_h,{w}_0), \widehat{w}_h-{w}_0 \rangle {\phi}(x^\tau_h,{w}_0) - f(x^\tau_h,\widehat{w}_h) {\phi}(x^\tau_h,\widehat{w}_h)}\nonumber\\
	&\leq   \ltwo{{\phi}(x^\tau_h,{w}_0) } \ltwo{\widehat{w}_h-{w}_0}  \ltwo{{\phi}(x^\tau_h,{w}_0) - {\phi}(x^\tau_h,\widehat{w}_h)} \nonumber\\
	&\quad + \lone{\langle {\phi}(x^\tau_h,{w}_0), \widehat{w}_h-{w}_0 \rangle -  f(x^\tau_h,\widehat{w}_h)} \ltwo{{\phi}(x^\tau_h,\widehat{w}_h)}.\label{eq: 89}
\end{flalign}
According to \Cref{lemma4} and the fact that $\ltwo{\widehat{w}_h-w_0}\leq H\sqrt{N/\lambda_2}$, we have the followings hold for all $h\in[H]$ and $\tau\in\mcd$ with probability at least $1-N^{-2}H^{-4}$
\begin{flalign}
&\ltwo{\phi(x^\tau_h,w_0)} \leq C_\phi\quad\text{and}\quad	\ltwo{\phi(x^\tau_h,\widehat{w}_h)} \leq C_\phi,\label{eq: 90}\\
&\ltwo{{\phi} (x^\tau_h,w_0) - {\phi} (x^\tau_h,\widehat{w}_h)}\leq \mathcal{O}\left( C_\phi \left( \frac{H\sqrt{N/\lambda_2}}{\sqrt{m}}\right)^{1/3}\sqrt{\log m} \right),\label{eq: 91}\\
&\lone{\langle {\phi} (x^\tau_h,{w}_0), \widehat{w}_h-{w}_0 \rangle - f(x^\tau_h,\widehat{w}_h)}\leq \mathcal{O}\left(C_\phi \left(\frac{H^4N^2/\lambda^2_2}{\sqrt{m}}\right)^{1/3}\sqrt{\log m}\right).\label{eq: 92}
\end{flalign}
Substituting \cref{eq: 90}, \cref{eq: 91} and \cref{eq: 92} into \cref{eq: 89}, we can obtain
\begin{flalign}
	&\ltwo{\langle {\phi}(x^\tau_h,{w}_0), \widehat{w}_h-{w}_0 \rangle {\phi}(x^\tau_h,{w}_0) - f(x^\tau_h,\widehat{w}_h) {\phi}(x^\tau_h,\widehat{w}_h)}\nonumber\\
	&\leq (H\sqrt{N/\lambda_2}) \mathcal{O}\left( C^2_\phi \left( \frac{H\sqrt{N/\lambda_2}}{\sqrt{m}}\right)^{1/3}\sqrt{\log m} \right) + \mathcal{O}\left(C^2_\phi \left(\frac{H^4N^2/\lambda^2_2}{\sqrt{m}}\right)^{1/3}\sqrt{\log m}\right)\nonumber\\
	&\leq \mathcal{O}\left( \frac{C^2_\phi H^{4/3} N^{2/3}\sqrt{\log(m)}}{m^{1/6}\lambda_2^{2/3}} \right).\label{eq: 93}
\end{flalign}
Substituting \cref{eq: 93} into \cref{eq: 88}, we have the following holds with probability at least $1-N^{-2}H^{-4}$
\begin{flalign}
&\ltwo{{\rm\Lambda}_h(w_0)( \widehat{w}_h -  \overline{w}_h)}\nonumber\\
&\leq NH \cdot \mathcal{O}\left( C_\phi \left( \frac{H\sqrt{N/\lambda_2}}{\sqrt{m}}\right)^{1/3}\sqrt{\log m} \right) + N \cdot \mathcal{O}\left( \frac{C^2_\phi H^{4/3} N^{2/3}\sqrt{\log(m)}}{m^{1/6}\lambda_2^{2/3}} \right)\nonumber\\
&\leq \mathcal{O}\left( \frac{C^2_\phi H^{4/3} N^{5/3}\sqrt{\log(m)}}{m^{1/6}\lambda_2^{2/3}} \right)
\end{flalign}
where we use the fact that $\widehat{V}_{h+1}(s)\leq H$ for any $s\in\mcs$. We then proceed to bound $	\ltwo{\widehat{w}_h -  \overline{w}_h}$ as follows
\begin{flalign}
\ltwo{\widehat{w}_h -  \overline{w}_h}&= \ltwo{{\rm\Lambda}^{-1}(w_0){\rm\Lambda}(w_0)(\widehat{w}_h -  \overline{w}_h)}\nonumber\\
&\leq \ltwo{{\rm\Lambda}^{-1}(w_0)}  \ltwo{{\rm\Lambda}(w_0)(\widehat{w}_h -  \overline{w}_h)}\nonumber\\
&\leq \frac{1}{\lambda_2} \cdot \mathcal{O}\left( \frac{C^2_\phi H^{4/3} N^{5/3}\sqrt{\log(m)}}{m^{1/6}\lambda_2^{2/3}} \right)\nonumber\\
&\leq  \mathcal{O}\left( \frac{C^2_\phi H^{4/3} N^{5/3}\sqrt{\log(m)}}{m^{1/6}\lambda_2^{5/3}} \right).\label{eq: 94}
\end{flalign}
Substituting \cref{eq: 94} into \cref{eq: 82} yields
\begin{flalign}
(ii) \leq \mathcal{O}\left( \frac{C^3_\phi H^{4/3} N^{5/3}\sqrt{\log(m)}}{m^{1/6}\lambda_2^{5/3}} \right).\label{eq: 95}
\end{flalign}
Taking summation of the upper bounds of $(i)$ in \cref{eq: 81} and $(ii)$ in \cref{eq: 95}, respectively, we have the following holds for all $x\in\mcx$ with probability at least $1-N^{-2}H^{-4}$
\begin{flalign}
\lone{(\widehat{\mP}_h\widehat{V}_{h+1})(x) - \overline{\mP}_h\widehat{V}_{h+1})(x)} &\leq (i) + (ii) \nonumber\\
&\leq  \mathcal{O}\left(C_\phi \left(\frac{N^2H^4}{\lambda^2_2\sqrt{m}}\right)^{1/3}\sqrt{\log m}\right) + \mathcal{O}\left( \frac{C^3_\phi H^{4/3} N^{5/3}\sqrt{\log(m)}}{m^{1/6}\lambda_2^{5/3}} \right)\nonumber\\
&\leq \mathcal{O}\left( \frac{C^3_\phi H^{4/3} N^{5/3}\sqrt{\log(m)}}{m^{1/6}\lambda_2^{5/3}} \right).\label{eq: 96}
\end{flalign}

\textbf{Step III.} In this step, we show that the bonus term $b_{v,h}(\cdot,\widehat{w}_h)$ in \Cref{alg1} can be well approximated by $b_{v,h}(\cdot,{w}_0)$. By the definition of $b_{v,h}(\cdot,w)$, we have
\begin{flalign}
&\lone{ b_{v,h}(x,\widehat{w}_h) - b_{v,h}(x,{w}_0)} \nonumber\\
&\quad =\lone{\left[{\phi}_h(x,\widehat{w}_h)^\top {\rm\Lambda}^{-1}(\widehat{w}_h) {\phi}_h(x,\widehat{w}_h)\right]^{1/2} - \left[{\phi}_h(x,w_0)^\top {\rm\Lambda}^{-1}(w_0) {\phi}_h(x,w_0)\right]^{1/2}}\nonumber\\
&\quad \leq \lone{ {\phi}_h(x,\widehat{w}_h)^\top {\rm\Lambda}^{-1}(\widehat{w}_h) {\phi}_h(x,\widehat{w}_h) - {\phi}_h(x,w_0)^\top {\rm\Lambda}^{-1}({w}_0) {\phi}_h(x,{w}_0)}^{1/2},\label{eq: 104}
\end{flalign}
where the last inequality follows from the fact that $\lone{\sqrt{x}-\sqrt{y}}\leq \sqrt{\lone{x-y}}$. Following steps similar to those in \cref{eq: 33}, we can obtain
\begin{flalign}
	&\lone{ {\phi}_h(x,\widehat{w}_h)^\top {\rm\Lambda}^{-1}(\widehat{w}_h) {\phi}_h(x,\widehat{w}_h) - {\phi}_h(x,w_0)^\top {\rm\Lambda}^{-1}({w}_0) {\phi}_h(x,{w}_0)}\nonumber\\
	&\quad \leq \ltwo{{\phi}(x,\widehat{w}_h) - {\phi}(x,{w}_0)} \ltwo{{\rm\Lambda}^{-1}(\widehat{w}_h)} \ltwo{{\phi}(x,\widehat{w}_h)} \nonumber\\
	&\quad\quad+ \ltwo{{\phi}(x, w_0)} \ltwo{{\rm\Lambda}^{-1}(\widehat{w}_h)} \ltwo{{\rm\Lambda}(\widehat{w}_h) - {\rm\Lambda}({w}_0)} \ltwo{{\rm\Lambda}^{-1}({w}_0)} \ltwo{{\phi}(x,\widehat{w}_h)}\nonumber\\
	&\quad\quad + \ltwo{{\phi}(x,{w}_0)} \ltwo{{\rm\Lambda}^{-1}({w}_0)} \ltwo{{\phi}(x,\widehat{w}_h) - {\phi}(x,{w}_0)}\nonumber\\
	&\quad \leq \frac{1}{\lambda_2}\ltwo{{\phi}(x,\widehat{w}_h) - {\phi}(x,{w}_0)}  \ltwo{{\phi}(x,\widehat{w}_h)} + \frac{1}{\lambda^2_2} \ltwo{{\phi}(x,{w}_0)} \ltwo{{\rm\Lambda}(\widehat{w}_h) - {\rm\Lambda}({w}_0)} \ltwo{{\phi}(x,\widehat{w}_h)}\nonumber\\
	&\quad\quad + \frac{1}{\lambda_2}\ltwo{{\phi}(x,w_0)} \ltwo{{\phi}(x,\widehat{w}_h) - {\phi}(x,{w}_0)},\label{eq: 107}
\end{flalign}
where the last inequality follows from the fact that $\ltwo{{\rm\Lambda}({w})}\geq \lambda_2$ for any ${w}\in\mR^{2md}$. For ${\rm\Lambda}(\widehat{w}) - {\rm\Lambda}({w}_0)$, by following steps similar to those in \cref{eq: 36}, we can obtain
\begin{flalign}
&\ltwo{	{\rm\Lambda}(\widehat{w}_h) - {\rm\Lambda}({w}_0)}\nonumber\\
&\quad \leq \sum_{\tau\in\mcd} \left[\ltwo{{\phi}(x^\tau_h,\widehat{w}_h)} \ltwo{{\phi}(x^\tau_h,\widehat{w}_h) - {\phi}(x^\tau_h,{w}_0)} + \ltwo{{\phi}(x^\tau_h,\widehat{w}_h) - {\phi}(x^\tau_h,{w}_0)} \ltwo{{\phi}(x^\tau_h,{w}_0)} \right].\label{eq: 105}
\end{flalign}
Applying \Cref{lemma4} to \cref{eq: 105}, we have the followings hold with probability at least $1-N^{-2}H^{-4}$
\begin{flalign*}
\ltwo{{\phi}(x^\tau_h,\widehat{w}_h)} & \leq C_\phi,\nonumber\\
\ltwo{{\phi}(x^\tau_h,\widehat{w}_h) - {\phi}(x^\tau_h,{w}_0) } &\leq \mathcal{O}\left(   \frac{C_\phi H^{1/3} N^{1/6} \sqrt{\log m} }{m^{1/6}\lambda_2^{1/6}} \right).
\end{flalign*}
Substituting the above two inequalities into \cref{eq: 36} yields
\begin{flalign}
\ltwo{	{\rm\Lambda}(\widehat{w}_h) - {\rm\Lambda}({w}_0)} \leq \mathcal{O}\left(   \frac{C^2_\phi H^{1/3} N^{1/6} \sqrt{\log m} }{m^{1/6}\lambda_2^{1/6}} \right).\label{eq: 106}
\end{flalign}
Finally, combining \cref{eq: 106} and \cref{eq: 13} and \cref{eq: 14} in \Cref{lemma4}, the right hand side of \cref{eq: 107} can be bounded by
\begin{flalign}
&\lone{ {\phi}_h(x,\widehat{w}_h)^\top {\rm\Lambda}^{-1}(\widehat{w}_h) {\phi}_h(x,\widehat{w}_h) - {\phi}_h(x,w_0)^\top {\rm\Lambda}^{-1}({w}_0) {\phi}_h(x,{w}_0)}\nonumber\\
&\quad \leq \frac{1}{\lambda_2}\ltwo{{\phi}(x,\widehat{w}_h) - {\phi}(x,{w}_0)}  \ltwo{{\phi}(x,\widehat{w}_h)} + \frac{1}{\lambda^2_2} \ltwo{{\phi}(x,{w}_0)} \ltwo{{\rm\Lambda}(\widehat{w}_h) - {\rm\Lambda}({w}_0)} \ltwo{{\phi}(x,\widehat{w}_h)}\nonumber\\
&\quad\quad + \frac{1}{\lambda_2}\ltwo{{\phi}(x,w_0)} \ltwo{{\phi}(x,\widehat{w}_h) - {\phi}(x,{w}_0)}\nonumber\\
&\quad\leq \mathcal{O}\left(   \frac{C^2_\phi H^{1/3} N^{1/6} \sqrt{\log m} }{m^{1/6}\lambda_2^{7/6}} \right) + \mathcal{O}\left(   \frac{C^4_\phi H^{1/3} N^{1/6} \sqrt{\log m} }{m^{1/6}\lambda_2^{13/6}} \right).\nonumber
\end{flalign}
By \cref{eq: 104}, we have the following holds with probability at least $1-N^{-2}H^{-4}$
\begin{flalign}
&\lone{ b_{v,h}(x,\widehat{w}_h) - b_{v,h}(x,{w}_0)} \nonumber\\
&\leq\lone{ {\phi}_h(x,\widehat{w}_h)^\top {\rm\Lambda}^{-1}(\widehat{w}_h) {\phi}_h(x,\widehat{w}_h) - {\phi}_h(x,w_0)^\top {\rm\Lambda}^{-1}({w}_0) {\phi}_h(x,{w}_0)}^{1/2}\nonumber\\
&\leq \mathcal{O}\left(   \frac{C^2_\phi H^{1/6} N^{1/12} (\log m)^{1/4}}{m^{1/12}\lambda_2^{13/12}} \right).\label{eq: 108}
\end{flalign}

{\bf Step IV.} In Steps I and II, we show that $({\mP}_h\widehat{V}_{h+1})(\cdot)$ can be well approximated by a linear function $(\widetilde{\mP}_h\widehat{V}_{h+1})(\cdot)$ with the feature $\phi(\cdot,\theta_0)$, and $(\widehat{\mP}_h\widehat{V}_{h+1})(\cdot)$ can be well approximated by a linear function $(\overline{\mP}_h\widehat{V}_{h+1})(\cdot)$ with the feature $\phi(\cdot,\theta_0)$. In this step, we want to show that the difference between $({\mP}_h\widehat{V}_{h+1})(\cdot)$ and $(\widehat{\mP}_h\widehat{V}_{h+1})(\cdot)$ is approximately $\beta_2\cdot b_{v.h}(x,{\rm\Theta}_0)$ with an approximately chosen $\beta_2$.

Recall that $(\widetilde{\mP}_h\widehat{V}_{h+1})(\cdot)=\langle \phi(\cdot, w_0), \tilde{w}_h-w_0 \rangle $ and $ (\overline{\mP}_h\widehat{V}_{h+1})(\cdot) = \langle \phi(\cdot,w_0), \bar{w}_h - w_0 \rangle$. Consider the difference between $(\overline{\mP}_h\widehat{V}_{h+1})(\cdot)$ and $(\widetilde{\mP}_h\widehat{V}_{h+1})(\cdot)$. We have
\begin{flalign}
(\overline{\mP}_h\widehat{V}_{h+1})(x) -(\widetilde{\mP}_h\widehat{V}_{h+1})(x) = \langle \phi(x,w_0), \bar{w}_h - \tilde{w}_h \rangle,\label{eq: 97}
\end{flalign}
By \cref{eq: 84}, we have
\begin{flalign}
\overline{w} - w_0 = {\rm\Lambda}(w_0)^{-1}\sum_{\tau\in\mcd} \widehat{V}_{h+1}(s^\tau_{h+1}){\phi}(x^\tau_h,w_0).\label{eq: 98}
\end{flalign}
By the definition of ${\rm\Lambda}(w)$, we have
\begin{flalign}
\widetilde{w} - {w}_0 = {\rm\Lambda}({w}_0)^{-1}\left[ \lambda_2\left( \widetilde{w} - {w}_0 \right) + \left(\sum_{\tau\in\mcd}{\phi}(x^\tau_{h},w_0){\phi}(x^\tau_{h},w_0)^\top\right)\left( \widetilde{w} - {w}_0 \right)  \right].\label{eq: 99}
\end{flalign}
Subtracting \cref{eq: 99} from \cref{eq: 98}, we have
\begin{flalign}
\overline{w} - \widetilde{w} = -\lambda_2 {\rm\Lambda}(w_0)^{-1}\left( \widetilde{w} - w_0 \right) + {\rm\Lambda}({w}_0)^{-1} \sum_{\tau\in\mcd} {\phi}(x^\tau_h,w_0)\left[ \widehat{V}_{h+1}(s^\tau_{h+1}) - \langle {\phi}(x^\tau_h,w_0), \widetilde{w} - w_0 \rangle \right].\label{eq: 100}
\end{flalign}
Taking inter product of both sides of \cref{eq: 100} with vector ${\phi}(x^\tau_h,w_0)$ and using the fact that $(\widetilde{\mP}_h\widehat{V}_{h+1})(s^\tau_{h+1}) = \langle \phi(x^\tau_h, w_0), \tilde{w}_h - w_0 \rangle$, we have
\begin{flalign}
&\langle {\phi}_h(x^\tau_h,w_0), \overline{w} - \widetilde{w} \rangle\nonumber\\
&\quad=-\lambda_2 {\phi}_h(x^\tau_h,w_0)^\top{\rm\Lambda}(w_0)^{-1/2}{\rm\Lambda}({w}_0)^{-1/2}\left( \widetilde{w} - {w}_0 \right) \nonumber\\
&\quad\quad+ {\phi}_h(x^\tau_h,w_0)^\top{\rm\Lambda}(w_0)^{-1/2}{\rm\Lambda}(w_0)^{-1/2} \left(\sum_{\tau\in\mcd} {\phi}(x^\tau_h,w_0)\left(\widehat{V}_{h+1}(s^\tau_{h+1}) - (\mP_h\widehat{V}_{h+1})(x^\tau_{h})  \right) \right)\nonumber\\
&\quad\quad+ {\phi}_h(x,w_0)^\top{\rm\Lambda}(w_0)^{-1/2}{\rm\Lambda}(w_0)^{-1/2} \left(\sum_{\tau\in\mcd} {\phi}(x^\tau_h,w_0)\left(  (\mP_h\widehat{V}_{h+1})(s^\tau_{h+1}) - \langle \phi(x^\tau_h,w_0), \tilde{w}_h - w_0 \rangle \right) \right)\nonumber\\
&\quad=-\lambda_2 {\phi}_h(x^\tau_h,w_0)^\top{\rm\Lambda}(w_0)^{-1/2}{\rm\Lambda}({w}_0)^{-1/2}\left( \widetilde{w} - {w}_0 \right) \nonumber\\
&\quad\quad+ {\phi}_h(x^\tau_h,w_0)^\top{\rm\Lambda}(w_0)^{-1/2}{\rm\Lambda}(w_0)^{-1/2} \left(\sum_{\tau\in\mcd} {\phi}(x^\tau_h,w_0)\left(\overline{V}_{h+1}(s^\tau_{h+1}) - (\mP_h\overline{V}_{h+1})(x^\tau_{h})  \right) \right)\nonumber\\
&\quad\quad+ {\phi}_h(x^\tau_h,w_0)^\top{\rm\Lambda}(w_0)^{-1/2}{\rm\Lambda}(w_0)^{-1/2} \left(\sum_{\tau\in\mcd} {\phi}(x^\tau_h,w_0)\left(\Delta V_{h+1}(s^\tau_{h+1}) - (\mP_h\Delta{V}_{h+1})(x^\tau_{h})  \right) \right)\nonumber\\
&\quad\quad+ {\phi}_h(x,w_0)^\top{\rm\Lambda}(w_0)^{-1/2}{\rm\Lambda}(w_0)^{-1/2} \left(\sum_{\tau\in\mcd} {\phi}(x^\tau_h,w_0)\left(  (\mP_h\widehat{V}_{h+1})(x^\tau_{h}) - (\widetilde{\mP}_h\widehat{V}_{h+1})(x^\tau_{h}) \right) \right)\label{eq: 101},
\end{flalign}
where in the last equality we denote $\Delta V_{h}(s) \coloneqq \widehat{V}_h(s) - \overline{V}_h(s)$. By the definition of $\widehat{V}_h(\cdot)$ in \Cref{alg1} and $\overline{V}_h(\cdot)$ in \cref{eq: 103}, we have
\begin{flalign}
	&\lone{\widehat{V}_h(x) - \overline{V}_h(x)} \nonumber\\
	&\quad\leq \sup_{x\in\mcx}\lone{\widehat{Q}_h(x) - \overline{Q}_h(x)}\nonumber\\
	&\quad\leq \lone{f(x,\widehat{\theta}_h) - \langle \phi(x,\theta_0), \bar{\theta}_h - \theta_0 \rangle} + \lone{f(x,\widehat{w}_h) - \langle \phi(x,w_0), \overline{w}_h - w_0 \rangle} \nonumber\\
	&\quad\quad + \beta_1\lone{b_{r,h}(x,\widehat{\rm\Theta}) - b_{r,h}(x,{\rm\Theta}_0) } + \beta_2\lone{b_{v,h}(x,\widehat{w}) - b_{v,h}(x,w_0)}\nonumber\\
	&\quad\overset{(i)}{\leq} \mathcal{O}\left(C_\phi \left(\frac{H^4N^2/\lambda^2_1}{\sqrt{m}}\right)^{1/3}\sqrt{\log m}\right) + \mathcal{O}\left(C_\phi \left(\frac{H^4N^2/\lambda^2_2}{\sqrt{m}}\right)^{1/3}\sqrt{\log m}\right)\nonumber\\
	&\quad\quad + \beta_1\cdot  \mathcal{O}\left(   \frac{C^2_\phi N^{1/12} (\log m)^{1/4}}{m^{1/12}\lambda_1^{13/12}} \right) + \beta_2\cdot \mathcal{O}\left(   \frac{C^2_\phi H^{1/6} N^{1/12} (\log m)^{1/4}}{m^{1/12}\lambda_2^{13/12}} \right)\nonumber\\
	&\quad\overset{(ii)}{\leq} \mathcal{O}\left(C_\phi \left(\frac{H^4N^2}{\sqrt{m}}\right)^{1/3}\sqrt{\log m}\right) + \max\{H^{2/3}\beta_1,H^{1/6}\beta_2\}\cdot \mathcal{O}\left(   \frac{C^2_\phi N^{1/12} (\log m)^{1/4}}{m^{1/12}} \right),\label{eq: 111}
\end{flalign}
where $(i)$ follows from \cref{eq: 15} in \Cref{lemma4}, \cref{eq: 109} and \cref{eq: 108}, and $(ii)$ follows from the fact that $\lambda_1,\lambda_2>1$. Denoting
\begin{flalign*}
	\varepsilon_v = \mathcal{O}\left(C_\phi \left(\frac{H^4N^2}{\sqrt{m}}\right)^{1/3}\sqrt{\log m}\right) + \max\{H^{2/3}\beta_1,H^{1/6}\beta_2\}\cdot \mathcal{O}\left(   \frac{C^2_\phi N^{1/12} (\log m)^{1/4}}{m^{1/12}} \right),
\end{flalign*} 
we then have the following holds for all $h\in[H]$ and $s\in\mcs$
\begin{flalign*}
	\lone{\Delta V_h(s)} \leq \varepsilon_v.
\end{flalign*}
\Cref{eq: 101} together with \cref{eq: 111} imply
\begin{flalign}
&\lone{\langle {\phi}_h(x^\tau_h,w_0), \overline{w} - \widetilde{w} \rangle} \nonumber\\
&\leq \sqrt{\lambda_2} \ltwo{{\phi}_h(x^\tau_h,w_0)^\top{\rm\Lambda}(w_0)^{-1/2}} \ltwo{\widetilde{w} - w_0}  \nonumber\\
&\quad +  \ltwo{ {\phi}_h(x^\tau_h,w_0)^\top{\rm\Lambda}(w_0)^{-1/2}} \llambda{\sum_{\tau\in\mcd} {\phi}(x^\tau_h,w_0)\varepsilon_v(x^\tau_h)} \nonumber\\
&\quad +  \frac{2\varepsilon_v}{\sqrt{\lambda_2}} \ltwo{{\phi}_h(x^\tau_h,w_0)^\top{\rm\Lambda}(w_0)^{-1/2}} \sum_{\tau\in\mcd} \ltwo{{\phi}(x^\tau_h,w_0)} \nonumber\\
&\quad + \frac{1}{\sqrt{\lambda_2}} \ltwo{ {\phi}_h(x^\tau_h,w_0)^\top{\rm\Lambda}(w_0)^{-1/2}} \left(\sum_{\tau\in\mcd} \ltwo{{\phi}(x^\tau_h,{w}_0)} \lone{ (\mP_h\widehat{V}_{h+1})(x^\tau_{h}) - (\widetilde{\mP}_h\widehat{V}_{h+1})(x^\tau_{h}) } \right),\label{eq: 110}
\end{flalign}
where we denote $\varepsilon_v(x^\tau_{h}) \coloneqq \overline{V}_{h+1}(s^\tau_{h+1}) - (\mP_h\overline{V}_{h+1})(x^\tau_{h})$ and use the fact that $\ltwo{{\rm\Lambda}({w})^{-1/2}}\leq 1/\sqrt{\lambda_2}$ for any ${w}\in\mR^{2md}$. By the definition of $\widetilde{w}$ in Step I, we have
\begin{flalign}
\ltwo{\widetilde{w} - {w}_0} = \ltwo{ \ell_{v} } \leq A_2\sqrt{H/d}.\label{eq: 112}
\end{flalign}
By \Cref{lemma4} and \cref{eq: 47}, we have the followings hold with probability at least $1-N^{-2}H^{-4}$ over the randomness of initialization $w_0$
\begin{flalign}
\ltwo{{\phi}(x^\tau_h,w_0)} &\leq C_\phi,\label{eq: 113}\\
\lone{ (\mP_h\widehat{V}_{h+1})(x^\tau_{h}) - (\widetilde{\mP}_h\widehat{V}_{h+1})(x^\tau_{h}) } &\leq \frac{2(L_\sigma A_2 + C^2_\sigma A^2_2)\sqrt{\log N^2H^5}}{\sqrt{m}}.\label{eq: 114}
\end{flalign}

Substituting \cref{eq: 112}, \cref{eq: 113} and \cref{eq: 114} into \cref{eq: 110} and using the fact that $b_{v,h}(x,w_0) = \ltwo{{\phi}_h(x,{w}_0)^\top{\rm\Lambda}(w_0)^{-1/2}}$, we have
\begin{flalign}
\lone{\langle {\phi}_h(x^\tau_h,w_0), \overline{w} - \widetilde{w} \rangle} &\leq \Bigg(A_2\sqrt{\frac{\lambda_2H}{d}} + \frac{2(L_\sigma A_2 + C^2_\sigma A^2_2)C_\phi N H^{3/2}\sqrt{\log HN}}{\sqrt{\lambda_2m}} \nonumber\\
&\qquad\qquad + \llambda{\sum_{\tau\in\mcd} {\phi}(x^\tau_h,w_0)\varepsilon_v(x^\tau_h)} \Bigg) b_{v,h}(x^\tau_h,{w}_0)  \label{eq: 115}.
\end{flalign}
Given that the events in \cref{eq: 113} and \cref{eq: 114} occur, applying \cref{eq: 69} in \Cref{lemma7}, we have the following holds with probability at least $1-N^{-2}H^{-4}$
\begin{flalign}
&\llambda{\sum_{\tau\in\mcd} {\phi}(x^\tau_h,w_0)\varepsilon_v(x^\tau_h)}^2\nonumber\\
&\quad\leq 2H^2  \log\det(I+K^v_{N,h}/\lambda_2) + 2H^2 N(\lambda_2-1) + 4H^2\log(\mcN^v_{\epsilon,h}/\delta) + 8N^2C^2_\phi\epsilon^2/\lambda_2,\label{eq: 116}
\end{flalign}
where $K^v_{N,h}\in\mR^{N\times N}$ is the Gram matrix defined as
\begin{flalign*}
K^v_{N,h} = [K(x^{\tau_i}_h, x^{\tau_j}_h)]_{i,j\in[N]}\in\mR^{N\times N},
\end{flalign*}
and $\mcN^v_{\epsilon,h}$ is the cardinality of the following function class
\begin{flalign*}
	\mcv_h(x,&\,R_\theta, R_w, R_{\beta_1}, R_{\beta_2}, \lambda_1, \lambda_2 )=\{ \max_{a\in\mca}\{\overline{Q}_h(s,a)\}:\mcs\rightarrow [0,H]\,\, \nonumber\\
	&\text{with}\,\, \ltwo{\theta}\leq R_\theta, \ltwo{w}\leq R_w, \beta_1\in[0, R_{\beta_1}], \beta_2\in [0, R_{\beta_2}], \ltwo{\rm\Sigma}\geq \lambda_1, \ltwo{\rm\Lambda}\geq \lambda_2  \},
\end{flalign*}
where $R_\theta = H\sqrt{N/\lambda_1}$, $R_w = H\sqrt{N/\lambda_2}$ and 
\begin{flalign*}
\overline{Q}_h&(x ) = \min\{ \langle \phi(x,\theta_0), \theta - \theta_0 \rangle+  \langle \phi(x,w_0), w- w_0 \rangle \nonumber\\
&- \beta_1\cdot \sqrt{{\rm\Phi}_h(x,\theta_0)^\top {\rm\Sigma}^{-1}{\rm\Phi}_h(x,\theta_0) } - \beta_2\cdot \sqrt{\phi(x,w_0)^\top{\rm\Lambda}^{-1}\phi(x,w_0)  } , H \}^{+}.
\end{flalign*}
Combining \cref{eq: 115} and \cref{eq: 116}, defining $\mcN_{\epsilon}^v=\max_{h\in[H]}\{\mcN_{\epsilon,h}^v\}$ and letting 
\begin{flalign}
	\epsilon = \sqrt{\lambda_2 C_\epsilon}H/(2NC_\phi)\,\, \text{where}\,\,C_\epsilon\geq 1\quad\text{and}\quad\lambda_2 = 1+N^{-1},\label{eq: 197}
\end{flalign}
and $m$ be sufficiently large such that
\begin{flalign*}
\frac{2(L_\sigma A_2 + C^2_\sigma A^2_2)C_\phi N H^{3/2}\sqrt{\log HN}}{\sqrt{\lambda_2m}}\leq  A_2\sqrt{\frac{\lambda_2H}{d}},
\end{flalign*}
we have the following holds with probability at least $1-N^{-2}H^{-4}$
\begin{flalign}
&\lone{\langle {\phi}_h(x^\tau_h,w_0), \overline{w} - \widetilde{w} \rangle} \nonumber\\
&\leq \left( 2A_2\sqrt{\frac{\lambda_2H}{d}} + \sqrt{2H^2  \log\det\left(I+\frac{K^v_{N,h}}{\lambda_2}\right) + 3C_\epsilon H^2+ 8H^2\log(NH^2 \mcN_{\epsilon}^v )} \right) b_{v,h}(x,w_0)\nonumber\\
&\leq \underbrace{H\left(\frac{8 A^2_2\lambda_2}{d} +   4\max_{h\in[H]} \left\{\log\det\left(I+\frac{K^v_{N,h}}{\lambda_2}\right)\right\} + 6C_\epsilon + 16\log(NH^2 \mcN_{\epsilon}^v) \right)^{1/2}}_{\beta_2} b_{v,h}(x,w_0),\label{eq: 117}
\end{flalign}
where in the last inequality we use the fact that $a+b\leq\sqrt{2(a^2+b^2)}$.
Substituting \cref{eq: 117} into \cref{eq: 97}, we conclude that the following holds with probability at least $1-N^{-2}H^{-4}$
\begin{flalign}
\lone{(\overline{\mP}_h\widehat{V}_{h+1})(x) -(\widetilde{\mP}_h\widehat{V}_{h+1})(x) }\leq \beta_2\cdot b_{v,h}(x,w_0),\label{eq: 118}
\end{flalign}
where
\begin{flalign*}
\beta_2 = H\left(\frac{8 A^2_2\lambda_2}{d} +   4\max_{h\in[H]}\left\{\log\det\left(I+\frac{K^v_{N,h}}{\lambda_2}\right)\right\} + 22\log(NH^2 \mcN_{\epsilon}^v) \right)^{1/2}.
\end{flalign*}
Next, we proceed to bound the term $\lone{({\mP}_h\widehat{V}_{h+1})(x) -(\widehat{\mP}_h\widehat{V}_{h+1})(x)}$. By the triangle inequality, we have
\begin{flalign}
&\lone{({\mP}_h\widehat{V}_{h+1})(x) - (\widehat{\mP}_h\widehat{V}_{h+1})(x)} \nonumber\\
&= \lone{({\mP}_h\widehat{V}_{h+1})(x) -(\widetilde{\mP}_h\widehat{V}_{h+1})(x) + (\widetilde{\mP}_h\widehat{V}_{h+1})(x) - (\overline{\mP}_h\widehat{V}_{h+1})(x) + (\overline{\mP}_h\widehat{V}_{h+1})(x) - (\widehat{\mP}_h\widehat{V}_{h+1})(x) }\nonumber\\
&\leq \lone{({\mP}_h\widehat{V}_{h+1})(x) -(\widetilde{\mP}_h\widehat{V}_{h+1})(x)} + \lone{(\widetilde{\mP}_h\widehat{V}_{h+1})(x) - (\overline{\mP}_h\widehat{V}_{h+1})(x)} + \lone{(\overline{\mP}_h\widehat{V}_{h+1})(x) - (\widehat{\mP}_h\widehat{V}_{h+1})(x) }\nonumber\\
&\overset{(i)}{\leq} \frac{2(L_\sigma A_2 + C^2_\sigma A^2_2)\sqrt{\log(N^2H^5)}}{\sqrt{m}} + \beta_2\cdot b_{v,h}(x,w_0) +  \mathcal{O}\left( \frac{C^3_\phi H^{4/3} N^{5/3}\sqrt{\log(m)}}{m^{1/6}\lambda_2^{5/3}} \right)\nonumber\\
&\overset{(ii)}{\leq} \mathcal{O}\left( \frac{H^{4/3} N^{5/3}\sqrt{\log(N^2H^5m)}}{m^{1/6}} \right) + \beta_2\cdot b_{v,h}(x,w_0),
\end{flalign}
where $(i)$ follows from \cref{eq: 78}, \cref{eq: 96} and \cref{eq: 118} and $(ii)$ follows from the fact that $\lambda_2=1+1/N$ and $L_\sigma,C_\sigma, A_2, C_\phi = \mathcal{O}(1)$.

\subsection{Upper and Lower Bounds on Evaluate Error $\delta_h(\cdot)$}\label{uplowerdelta1}
By definition, we have the following holds with probability $1-2N^{-2}H^{-4}$
\begin{flalign}
	&\lone{(\widehat{\mB}_h\widehat{V}_{h+1})(x) - ({\mB}_h\widehat{V}_{h+1})(x) } \nonumber\\
	&= \lone{\widehat{R}_h(x) + (\widehat{\mP}_h\hV_{h+1})(x) - R_h(x) - ({\mP}_h\hV_{h+1})(x) }\nonumber\\
	&\leq \lone{\widehat{R}_h(x) - R_h(x)} + \lone{(\widehat{\mP}_h\hV_{h+1})(x) - ({\mP}_h\hV_{h+1})(x) }\nonumber\\
	&\overset{(i)}{\leq} \beta_1\cdot b_{r,h}(x,{\rm\Theta}_0) + \beta_2\cdot b_{v,h}(x,w_0) + \mathcal{O}\left( \frac{H^{4/3} N^{5/3}\sqrt{\log(N^2H^5m)}}{m^{1/6}} \right) \nonumber\\
	&\quad+ \mathcal{O}\left( \frac{H^{17/6} N^{5/3}\sqrt{\log(m)}}{m^{1/6}} \right)\nonumber\\
	&\leq \beta_1\cdot b_{r,h}(x,{\rm\Theta}_0) + \beta_2\cdot b_{v,h}(x,w_0) +  \mathcal{O}\left( \frac{H^{17/6} N^{5/3}\sqrt{\log(N^2H^5m)}}{m^{1/6}} \right), \label{eq: 120}
\end{flalign}
where $(i)$ follows from \cref{eq: 119} and \cref{eq: 120}. Moreover, by the triangle inequality, \cref{eq: 109} and \cref{eq: 108}, we have the following holds with probability $1-2N^{-2}H^{-4}$
\begin{flalign}
	&\beta_1\cdot b_{r,h}(x,{\rm\Theta}_0) + \beta_2\cdot b_{v,h}(x,w_0)\nonumber\\
	&\leq \beta_1\cdot b_{r,h}(x,\widehat{\rm\Theta}) + \beta_2\cdot b_{v,h}(x,\widehat{w}) + \beta_1\cdot \lone{b_{r,h}(x,\widehat{\rm\Theta})  -  b_{r,h}(x,{\rm\Theta}_0) } \nonumber\\
	&\quad+ \beta_2\cdot \lone{b_{r,h}(x,\widehat{w})  -  b_{r,h}(x,{w}_0) } \nonumber\\
	&\leq \beta_1\cdot b_{r,h}(x,\widehat{\rm\Theta}) + \beta_2\cdot b_{v,h}(x,\widehat{w}_h) \nonumber\\
	&\quad + \beta_1\cdot \mathcal{O}\left(   \frac{C^2_\phi H^{2/3} N^{1/12} (\log m)^{1/4}}{m^{1/12}\lambda_1^{13/12}} \right) + \beta_2\cdot \mathcal{O}\left(   \frac{C^2_\phi H^{1/6} N^{1/12} (\log m)^{1/4}}{m^{1/12}\lambda_2^{13/12}} \right)\nonumber\\
	&\overset{(i)}{\leq} \beta_1\cdot b_{r,h}(x,\widehat{\rm\Theta}) + \beta_2\cdot b_{v,h}(x,\widehat{w}_h) + \max\{\beta_1 H^{2/3}, \beta_2 H^{1/6} \} \mathcal{O}\left(   \frac{N^{1/12} (\log m)^{1/4}}{m^{1/12}} \right),\label{eq: 121}
\end{flalign}
where $(i)$ follows from the fact that $\lambda_1 = \lambda_2 = 1 + 1/N$ and $C_\phi = \mathcal{O}(1)$. Substituting \cref{eq: 121} into \cref{eq: 120}, we can obtain
\begin{flalign}
	&\lone{(\widehat{\mB}_h\widehat{V}_{h+1})(x) - ({\mB}_h\widehat{V}_{h+1})(x) }\nonumber\\
	&\leq  \beta_1\cdot b_{r,h}(x,\widehat{\rm\Theta}) + \beta_2\cdot b_{v,h}(x,\widehat{w}_h) \nonumber\\
	&\quad + \max\{\beta_1 H^{2/3}, \beta_2 H^{1/6} \} \mathcal{O}\left(   \frac{N^{1/12} (\log m)^{1/4}}{m^{1/12}} \right) +  \mathcal{O}\left( \frac{H^{17/6} N^{5/3}\sqrt{\log(N^2H^5m)}}{m^{1/6}} \right).\nonumber
\end{flalign}
Denoting
\begin{flalign*}
	\varepsilon_b = \max\{\beta_1 H^{2/3}, \beta_2 H^{1/6} \} \mathcal{O}\left(   \frac{N^{1/12} (\log m)^{1/4}}{m^{1/12}} \right) +  \mathcal{O}\left( \frac{H^{17/6} N^{5/3}\sqrt{\log(N^2H^5m)}}{m^{1/6}} \right),
\end{flalign*}
we have
\begin{flalign}
	\lone{(\widehat{\mB}_h\widehat{V}_{h+1})(x) - ({\mB}_h\widehat{V}_{h+1})(x) }\leq \beta_1\cdot b_{r,h}(x,\widehat{\rm\Theta}) + \beta_2\cdot b_{v,h}(x,\widehat{w}_h) + \varepsilon_b.\label{eq: 122}
\end{flalign}
Up to this point, we characterize the uncertainty of $(\widehat{\mB}_h\widehat{V}_{h+1})(\cdot)$. Next, we proceed to bound the suboptimality of \Cref{alg1}. Recalling the construction of $\widehat{Q}_h(x)$ in \Cref{alg1}, we have
\begin{flalign*}
	\widehat{Q}_h(\cdot) = \min\{ (\widehat{\mB}_h\widehat{V}_{h+1})(\cdot) - \beta_1\cdot b_{r,h}(\cdot,\widehat{\rm\Theta}) - \beta_2\cdot b_{v,h}(\cdot,\widehat{w}_h), H \}^{+}.
\end{flalign*}
If $(\widehat{\mB}_h\widehat{V}_{h+1})(x) < \beta_1\cdot b_{r,h}(x,\widehat{\rm\Theta}) + \beta_2\cdot b_{v,h}(x,\widehat{w}_h)$, we have
\begin{flalign*}
	\widehat{Q}_h(\cdot) = 0. 
\end{flalign*}
Note that $\widehat{V}_{h+1}(\cdot)$ is nonnegative. Recalling the definition of $\delta_h(x)$ in \cref{eq: 123}, we have
\begin{flalign*}
	\delta_h(x) = ({\mB}_h\widehat{V}_{h+1})(x) - \widehat{Q}_h(x) = ({\mB}_h\widehat{V}_{h+1})(x) > 0.
\end{flalign*}
Otherwise, if $(\widehat{\mB}_h\widehat{V}_{h+1})(x) > \beta_1\cdot b_{r,h}(x,\widehat{\rm\Theta}) + \beta_2\cdot b_{v,h}(x,\widehat{w}_h)$, we have
\begin{flalign*}
\widehat{Q}_h(x) &= \min\{ (\widehat{\mB}_h\widehat{V}_{h+1})(x) - \beta_1\cdot b_{r,h}(x,\widehat{\rm\Theta}) - \beta_2\cdot b_{v,h}(x,\widehat{w}_h), H \}^{+} \nonumber\\
&\leq (\widehat{\mB}_h\widehat{V}_{h+1})(x) - \beta_1\cdot b_{r,h}(x,\widehat{\rm\Theta}) - \beta_2\cdot b_{v,h}(x,\widehat{w}_h),
\end{flalign*}
which implies that
\begin{flalign*}
	\delta_h(x) &\geq ({\mB}_h\widehat{V}_{h+1})(x) - \left[(\widehat{\mB}_h\widehat{V}_{h+1})(x) - \beta_1\cdot b_{r,h}(x,\widehat{\rm\Theta}) - \beta_2\cdot b_{v,h}(x,\widehat{w}_h)\right]\nonumber\\
	&= \left[({\mB}_h\widehat{V}_{h+1})(x) - (\widehat{\mB}_h\widehat{V}_{h+1})(x)\right] + \beta_1\cdot b_{r,h}(x,\widehat{\rm\Theta}) + \beta_2\cdot b_{v,h}(x,\widehat{w}_h).
\end{flalign*}
Note that \cref{eq: 122} implies the followings hold with probability $1-2N^{-2}H^{-4}$
\begin{flalign}
	({\mB}_h\widehat{V}_{h+1})(x) - (\widehat{\mB}_h\widehat{V}_{h+1})(x) &\geq -\beta_1\cdot b_{r,h}(x,\widehat{\rm\Theta}) - \beta_2\cdot b_{v,h}(x,\widehat{w}_h) -\varepsilon_b.\label{eq: 125}\\
	({\mB}_h\widehat{V}_{h+1})(x) - (\widehat{\mB}_h\widehat{V}_{h+1})(x) &\leq \beta_1\cdot b_{r,h}(x,\widehat{\rm\Theta}) + \beta_2\cdot b_{v,h}(x,\widehat{w}_h) + \varepsilon_b.\label{eq: 126}
\end{flalign}
As a result, we have the following holds with probability $1-2N^{-2}H^{-4}$ 
\begin{flalign}
	\delta_h(x) &\geq -\varepsilon_b \label{eq: 124}.
\end{flalign} 

It remains to establish the upper bound of $\delta_h(x)$. Considering the event in \cref{eq: 126} occurs, we have
\begin{flalign*}
	&(\widehat{\mB}_h\widehat{V}_{h+1})(\cdot) - \beta_1\cdot b_{r,h}(\cdot,\widehat{\rm\Theta}) - \beta_2\cdot b_{v,h}(\cdot,\widehat{w}_h) \nonumber\\
	&\leq \left[({\mB}_h\widehat{V}_{h+1})(x) +  \beta_1\cdot b_{r,h}(x,\widehat{\rm\Theta}) + \beta_2\cdot b_{v,h}(x,\widehat{w}_h) + \varepsilon_b\right] - \beta_1\cdot b_{r,h}(\cdot,\widehat{\rm\Theta}) - \beta_2\cdot b_{v,h}(\cdot,\widehat{w}_h)\nonumber\\
	&= ({\mB}_h\widehat{V}_{h+1})(x) + \varepsilon_b \leq H + \varepsilon_b,
\end{flalign*}
where the last inequality follows from the fact that $R_h(x) \leq 1$ and $\widehat{V}_{h+1}(s)\leq H$ for all $x\in\mcx$ and $s\in\mcs$. Hence, we have
\begin{flalign}
	\widehat{Q}_h(x) &= \min\{ (\widehat{\mB}_h\widehat{V}_{h+1})(x) - \beta_1\cdot b_{r,h}(x,\widehat{\rm\Theta}) - \beta_2\cdot b_{v,h}(x,\widehat{w}_h), H \}^{+}\nonumber\\
	&\geq \min\{ (\widehat{\mB}_h\widehat{V}_{h+1})(x) - \beta_1\cdot b_{r,h}(x,\widehat{\rm\Theta}) - \beta_2\cdot b_{v,h}(x,\widehat{w}_h) - \varepsilon_b, H \}^{+}\nonumber\\
	&=\max\{ (\widehat{\mB}_h\widehat{V}_{h+1})(x) - \beta_1\cdot b_{r,h}(x,\widehat{\rm\Theta}) - \beta_2\cdot b_{v,h}(x,\widehat{w}_h) - \varepsilon_b, 0 \}\nonumber\\
	&\geq (\widehat{\mB}_h\widehat{V}_{h+1})(x) - \beta_1\cdot b_{r,h}(x,\widehat{\rm\Theta}) - \beta_2\cdot b_{v,h}(x,\widehat{w}_h) - \varepsilon_b,
\end{flalign}
which by definition of $\delta_h(x)$ implies
\begin{flalign}
	\delta_h(x) &= ({\mB}_h\widehat{V}_{h+1})(x) - \widehat{Q}_h(x)\nonumber\\
	&\leq ({\mB}_h\widehat{V}_{h+1})(x) - (\widehat{\mB}_h\widehat{V}_{h+1})(x) + \beta_1\cdot b_{r,h}(x,\widehat{\rm\Theta}) + \beta_2\cdot b_{v,h}(x,\widehat{w}_h) + \varepsilon_b \nonumber\\
	&\leq 2\left[\beta_1\cdot b_{r,h}(x,\widehat{\rm\Theta}) + \beta_2\cdot b_{v,h}(x,\widehat{w}_h) + \varepsilon_b\right],\label{eq: 127}
\end{flalign}
where the last inequality follows from \cref{eq: 126}. Combining \cref{eq: 124} and \cref{eq: 127}, with probability $1-2N^{-2}H^{-4}$, we have
\begin{flalign*}
	-\varepsilon_b \leq \delta_h(x) \leq  2\left[\beta_1\cdot b_{r,h}(x,\widehat{\rm\Theta}) + \beta_2\cdot b_{v,h}(x,\widehat{w}_h) + \varepsilon_b\right], \quad\forall x\in\mcx,\quad \forall h\in[H],
\end{flalign*}
which completes the proof.

\section{Proof of \Cref{lemma8}}\label{pfpenaltysummation}
For $\sum_{h=1}^{H}b_{r,h}(x,\widehat{\rm\Theta})$, we have the following holds with probability $1-N^{-2}H^{-4}$
\begin{flalign}
	\sum_{h=1}^{H}b_{r,h}(x,\widehat{\rm\Theta})  & \leq  \sum_{h=1}^{H}b_{r,h}(x,{\rm\Theta}_0) + \sum_{h=1}^{H}\lone{b_{r,h}(x,\widehat{\rm\Theta}) - b_{r,h}(x,{\rm\Theta}_0)}\nonumber\\
	&\overset{(i)}{\leq} \sum_{h=1}^{H}b_{r,h}(x,{\rm\Theta}_0) + \mathcal{O}\left(   \frac{H^{5/3} N^{1/12} (\log m)^{1/4}}{m^{1/12}} \right),\label{eq: 137}
\end{flalign}
where $(i)$ follows from \cref{eq: 109}. We next proceed to bound the term $\sum_{h=1}^{H}b_{r,h}(x,{\rm\Theta}_0)$. Recall that in \Cref{ass2} we define $\overline{M}({\rm\Theta}_0) = \mE_{\mu}\left[ {\rm\Phi}(\tau,{\rm\Theta}_0){\rm\Phi}(\tau,{\rm\Theta}_0)^\top \right]$. For all $\tau\in\mcd$, we define the following random matrix $\widehat{M}({\rm\Theta}_0)$
\begin{flalign}
	\widehat{M}({\rm\Theta}_0) = \sum_{\tau\in\mcd} A_\tau({\rm\Theta}_0),\quad\text{where}\quad A_\tau({\rm\Theta}_0) = {\rm\Phi}(\tau,{\rm\Theta}_0){\rm\Phi}(\tau,{\rm\Theta}_0)^\top - \overline{M}({\rm\Theta}_0).\label{eq: 130}
\end{flalign}
Note that \cref{eq: 23} implies $\ltwo{{\rm\Phi}(\tau,{\rm\Theta}_0)}\leq C_\phi\sqrt{H}$. By Jensen's inequality, we have
\begin{flalign}
	\ltwo{\overline{M}({\rm\Theta}_0)}\leq \mE_{\mu}\left[ \ltwo{{\rm\Phi}(\tau,{\rm\Theta}_0){\rm\Phi}(\tau,{\rm\Theta}_0)^\top} \right]\leq C^2_\phi H.\label{eq: 131}
\end{flalign}
For any vector $v\in\mR^{2mdH}$ with $\ltwo{v}=1$, we have
\begin{flalign}
	\ltwo{A_\tau({\rm\Theta}_0)v} &\leq \ltwo{{\rm\Phi}(\tau,{\rm\Theta}_0){\rm\Phi}(\tau,{\rm\Theta}_0)^\top v} + \ltwo{\overline{M}({\rm\Theta}_0) v}\nonumber\\
	& \leq \ltwo{{\rm\Phi}(\tau,{\rm\Theta}_0){\rm\Phi}(\tau,{\rm\Theta}_0)^\top }\ltwo{v} + \ltwo{\overline{M}({\rm\Theta}_0)} \ltwo{v}\nonumber\\
	&\leq 2C^2_\phi H\ltwo{v} = 2C^2_\phi H,\nonumber
\end{flalign}
which implies
\begin{flalign}
	\ltwo{A_\tau({\rm\Theta}_0)}\leq 2C^2_\phi H\,\,\text{and}\,\, \ltwo{A_\tau({\rm\Theta}_0) A_\tau({\rm\Theta}_0)^\top} \leq \ltwo{A_\tau({\rm\Theta}_0)} \ltwo{A_\tau({\rm\Theta}_0)^\top} \leq 4C^4_\phi H^2.\label{eq: 132}
\end{flalign}
Since $\{ A_\tau({\rm\Theta}_0)\}_{\tau\in\mcd}$ are i.i.d. and $\mE[A_\tau({\rm\Theta}_0)] = 0$ for all $\tau$, we have
\begin{flalign}
	\ltwo{E_{\mu}[\widehat{M}({\rm\Theta}_0) \widehat{M}({\rm\Theta}_0)^\top ]} &= \ltwo{ \sum_{\tau\in\mcd}E_{\mu} \left[ A_\tau({\rm\Theta}_0)A_\tau({\rm\Theta}_0)^\top \right] }\nonumber\\
	&=N\cdot \ltwo{E_{\mu} \left[ A_{\tau_1}({\rm\Theta}_0)A_{\tau_1}({\rm\Theta}_0)^\top \right]}\nonumber\\
	&\overset{(i)}{\leq} N\cdot E_{\mu} \left[ \ltwo{A_{\tau_1}({\rm\Theta}_0)A_{\tau_1}({\rm\Theta}_0)^\top} \right]\nonumber\\
	&\leq 4C^4_\phi H^2 N,\nonumber
\end{flalign}
where $(i)$ follows from Jensen's inequality. Similarly, we can also obtain
\begin{flalign}
	\ltwo{E_{\mu}[\widehat{M}({\rm\Theta}_0)^\top \widehat{M}({\rm\Theta}_0) ]} \leq 4C^4_\phi H^2 N.\nonumber
\end{flalign}
Applying \Cref{lemma9} to $\widehat{M}({\rm\Theta}_0)$, for any fixed $h\in[H]$ and any $\xi_1>0$, we have
\begin{flalign}
{\rm P}\left(\ltwo{\widehat{M}({\rm\Theta}_0)}\geq \xi_1\right)\leq 4mdH\cdot\exp\left( -  \frac{\xi_1^2/2}{4C^4_\phi H^2N + 2C^2_\phi H/3 \cdot \xi_1} \right).\nonumber
\end{flalign}
For any $\delta_1\in(0,1)$, let
\begin{flalign*}
	\xi_1 = C^2_\phi H \sqrt{10N \log\left(\frac{4mdH}{\delta_1}\right)}\quad\text{and}\quad N\geq \frac{40}{9}\log\left( \frac{4mdH}{\delta_1} \right).
\end{flalign*}
Then, we have
\begin{flalign}
	{\rm P}\left(\ltwo{\widehat{M}({\rm\Theta}_0)}\geq \xi_1\right)&\leq 4mdH\cdot\exp\left( -  \frac{\xi_1^2/2}{4C^4_\phi H^2N + 2C^2_\phi H/3 \cdot \xi_1} \right) \nonumber\\
	&\quad \leq 4mdH\cdot\exp\left( -  \frac{\xi_1^2}{10C^4_\phi H^2N } \right)=\delta_1,\nonumber
\end{flalign}
which implies that the following holds with probability at least $1-\delta_1$ taken with respect to the randomness of $\mcd$
\begin{flalign}
	\ltwo{\widehat{M}({\rm\Theta}_0)/N} &= \ltwo{\frac{1}{N} \sum_{\tau\in\mcd} {\rm\Phi}(\tau,{\rm\Theta}_0){\rm\Phi}(\tau,{\rm\Theta}_0)^\top - \overline{M}({\rm\Theta}_0) }\nonumber\\
	&\leq C^2_\phi H \sqrt{\frac{10}{N} \log\left(\frac{4mdH}{\delta_1}\right)}.\label{eq: 133}
\end{flalign}
By the definition of ${\rm\Sigma}({\rm\Theta}_0)$, we have
\begin{flalign}
	\widehat{M}({\rm\Theta}_0) = \left({\rm\Sigma}({\rm\Theta}_0) - \lambda_1\cdot I_{2mdH} \right) - N\cdot \overline{M}({\rm\Theta}_0).\label{eq: 134}
\end{flalign}
By \Cref{ass2}, there exists an absolute constant $C_\sigma>0$ such that $\lambda_{\min}(\overline{M}({\rm\Theta}_0))\geq C_\sigma$, which implies that $\ltwo{\overline{M}({\rm\Theta}_0)^{-1}}\leq 1/C_\sigma$. Letting $N$ be sufficiently large such that
\begin{flalign*}
	N\geq \max\left\{\frac{40C^4_\phi H^2}{C^2_\sigma}, \frac{40}{9} \right\} \log\left(\frac{4mdH}{\delta_1}\right)
\end{flalign*}
and combining \cref{eq: 133} and \cref{eq: 134}, we have
\begin{flalign}
	\lambda_{\min}({\rm\Sigma}({\rm\Theta}_0)/N) &= \lambda_{\min}(\overline{M}({\rm\Theta}_0) + \widehat{M}({\rm\Theta}_0)/N + \lambda_1/N \cdot I_{2mdH})\nonumber\\
	&\geq \lambda_{\min}(\overline{M}({\rm\Theta}_0)) - \ltwo{\widehat{M}({\rm\Theta}_0)/N}\nonumber\\
	&\geq C_\sigma - C^2_\phi H \sqrt{\frac{10}{N} \log\left(\frac{4mdH}{\delta}\right)}\nonumber\\
	&\geq C_\sigma/2.\nonumber
\end{flalign}
Hence, the following holds with probability $1-\delta_1$ with respect to randomness of $\mcd$
\begin{flalign*}
	\ltwo{{\rm\Sigma}({\rm\Theta}_0)^{-1}}\leq (N\cdot \lambda_{\min}({\rm\Sigma}({\rm\Theta}_0)/N))^{-1}\leq \frac{2}{NC_\sigma},
\end{flalign*}
which implies the following holds for all $x\in\mcx$ and $h\in[H]$
\begin{flalign}
	b_{r,h}(x,{\rm\Theta}_0) &= \sqrt{{\rm\Phi}_h(x,{\rm\Theta}_0)^\top {\rm\Sigma}^{-1}({\rm\Theta}_0) {\rm\Phi}_h(x,{\rm\Theta}_0)}\nonumber\\
	&\leq \ltwo{{\rm\Phi}_h(x,{\rm\Theta}_0)}\cdot \ltwo{{\rm\Sigma}^{-1}({\rm\Theta}_0)}^{1/2} \leq \frac{\sqrt{2} C_\phi}{\sqrt{C_\sigma}\sqrt{N}},\label{eq: 135}
\end{flalign}
where we use the fact that $\ltwo{{\rm\Phi}_h(x,{\rm\Theta}_0)} = \ltwo{{\phi}(x^\tau_h,\theta_0)}\leq C_\phi$. Substituting \cref{eq: 135} into \cref{eq: 137}, we have
\begin{flalign}
	\sum_{h=1}^{H}b_{r,h}(x,\widehat{\rm\Theta}) \leq  \frac{\sqrt{2} H C_\phi}{\sqrt{C_\sigma}\sqrt{N}} + \mathcal{O}\left(   \frac{H^{5/3} N^{1/12} (\log m)^{1/4}}{m^{1/12}} \right).\label{eq: 136}
\end{flalign}
Next, we proceed to bound the term $\sum_{h=1}^{H}b_{v,h}(x,\widehat{w}_h)$. According to \cref{eq: 108}, we have the following holds with probability at least $1-N^{-2}H^{-4}$
\begin{flalign}
\sum_{h=1}^{H}b_{v,h}(x,\widehat{w}_h)  & \leq  \sum_{h=1}^{H}b_{v,h}(x,{w}_0) + \sum_{h=1}^{H}\lone{b_{v,h}(x,\widehat{w}_h) - b_{v,h}(x,{w}_0)}\nonumber\\
&\overset{(i)}{\leq} \sum_{h=1}^{H}b_{v,h}(x,{w}_0) + \mathcal{O}\left(   \frac{ H^{7/6} N^{1/12} (\log m)^{1/4}}{m^{1/12}} \right).\label{eq: 138}
\end{flalign}

We then proceed to bound the summation of the penalty terms $\sum_{h=1}^{H}b_{v,h}(x,{w}_0)$. Recall that in \Cref{ass2} we define $\overline{m}_h({w}_0) = \mE_{\mu}\left[\phi(x^\tau_h,w_0)\phi(x^\tau_h,w_0)^\top \right]$. For all $h\in[H]$ and $\tau\in\mcd$, we define the following random matrix $\widehat{m}({w}_0)$
\begin{flalign}
\widehat{m}_h({w}_0) = \sum_{\tau\in\mcd} B^\tau_h({w}_0),\quad\text{where}\quad B_h^\tau({w}_0) = {\phi}(x^\tau_h,{w}_0){\phi}(x^\tau_h,{w}_0)^\top - \overline{m}_h({w}_0).\label{eq: 139}
\end{flalign}
Note that \cref{eq: 23} implies $\ltwo{{\phi}(x^\tau_h,{w}_0)}\leq C_\phi$. By Jensen's inequality, we have
\begin{flalign}
\ltwo{\overline{m}_h({w}_0)}\leq \mE_{\mu}\left[ \ltwo{{\phi}(x^\tau_h,{w}_0){\phi}(x^\tau_h,{w}_0)^\top} \right]\leq C^2_\phi.\label{eq: 140}
\end{flalign}
For any vector $v\in\mR^{2md}$ with $\ltwo{v}=1$, we have
\begin{flalign}
\ltwo{B^\tau_h({w}_0)v} &\leq \ltwo{{\phi}(x^\tau_h,{w}_0){\phi}(x^\tau_h,{w}_0)^\top v} + \ltwo{\overline{m}_h({w}_0) v}\nonumber\\
& \leq \ltwo{{\phi}(x^\tau_h,{w}_0){\phi}(x^\tau_h,{w}_0)^\top }\ltwo{v} + \ltwo{\overline{m}_h({w}_0)} \ltwo{v}\nonumber\\
&\leq 2C^2_\phi \ltwo{v}= 2C^2_\phi,\nonumber
\end{flalign}
which implies
\begin{flalign}
\ltwo{B^\tau_h({w}_0)}\leq 2C^2_\phi \,\,\text{and}\,\, \ltwo{B^\tau_h({w}_0) B^\tau_h({w}_0)^\top} \leq \ltwo{B^\tau_h({w}_0)} \ltwo{B^\tau_h({w}_0)^\top} \leq 4C^4_\phi.\label{eq: 141}
\end{flalign}
Since $\{ B^\tau_h({w}_0)\}_{\tau\in\mcd}$ are i.i.d. and $\mE[B^\tau_h({w}_0)] = 0$ for all $\tau$, we have
\begin{flalign}
\ltwo{E_{\mu}[\overline{m}_h({w}_0) \overline{m}_h({w}_0)^\top ]} &= \ltwo{ \sum_{\tau\in\mcd}E_{\mu} \left[ B^\tau_h({w}_0) B^\tau_h({w}_0)^\top \right] }\nonumber\\
&=N\cdot \ltwo{E_{\mu} \left[ B^{\tau_1}_h({w}_0)B^{\tau_1}_h({w}_0)^\top \right]}\nonumber\\
&\overset{(i)}{\leq} N\cdot E_{\mu} \left[ \ltwo{B^{\tau_1}_h({w}_0)B^{\tau_1}_h({w}_0)^\top} \right]\nonumber\\
&\leq 4C^4_\phi N,\nonumber
\end{flalign}
where $(i)$ follows from Jensen's inequality. Similarly, we can also obtain
\begin{flalign}
\ltwo{E_{\mu}[\overline{m}_h({w}_0)^\top \overline{m}_h({w}_0) ]} \leq 4C^4_\phi N.\nonumber
\end{flalign}
Applying \Cref{lemma9} to $\widehat{m}_h({w}_0)$, for any fixed $h\in[H]$ and any $\xi_2>0$, we have
\begin{flalign}
{\rm P}\left(\ltwo{\widehat{m}_h({w}_0)}\geq \xi_2\right)\leq 4md\cdot\exp\left( -  \frac{\xi_2^2/2}{4C^4_\phi N + 2C^2_\phi /3 \cdot \xi_2} \right).\nonumber
\end{flalign}
For any $\delta_2\in(0,1)$, let
\begin{flalign*}
\xi_2 = C^2_\phi \sqrt{10N \log\left(\frac{4mdH}{\delta_2}\right)}\quad\text{and}\quad N\geq \frac{40}{9}\log\left( \frac{4mdH}{\delta_2} \right).
\end{flalign*}
Then, we have
\begin{flalign}
{\rm P}\left(\ltwo{\widehat{m}_h({w}_0)}\geq \xi_2\right)&\leq 4md\cdot\exp\left( -  \frac{\xi^2_2/2}{4C^4_\phi N + 2C^2_\phi /3 \cdot \xi_2} \right) \nonumber\\
&\quad \leq 4md \cdot\exp\left( -  \frac{\xi^2_2}{10C^4_\phi N } \right)=\frac{\delta_2}{H},\nonumber
\end{flalign}
which implies that we have the following holds with probability at least $1-\delta_2/H$ taken with respect to the randomness of $\mcd$
\begin{flalign}
\ltwo{\widehat{m}_h({w}_0)/N} &= \ltwo{\frac{1}{N} \sum_{\tau\in\mcd} {\phi}(x^\tau_h,{w}_0){\phi}(x^\tau_h,{w}_0)^\top - \overline{m}_h({w}_0) }\nonumber\\
&\leq C^2_\phi  \sqrt{\frac{10}{N} \log\left(\frac{4mdH}{\delta_2}\right)}.\label{eq: 142}
\end{flalign}
By the definition of ${\rm\Lambda}_h({w}_0)$, we have
\begin{flalign}
\widehat{m}_h({w}_0) = \left({\rm\Lambda}_h({w}_0) - \lambda_2\cdot I_{2md} \right) - N\cdot \overline{m}_h({w}_0).\label{eq: 143}
\end{flalign}
By \Cref{ass2}, there exists an absolute constant $C_\varsigma>0$ such that $\lambda_{\min}(\overline{m}_h({\rm\Theta}_0))\geq C_\varsigma$, which implies that $\ltwo{\overline{m}({w}_0)^{-1}}\leq 1/C_\varsigma$. Letting $N$ be sufficiently large such that
\begin{flalign*}
N\geq \max\left\{\frac{40C^4_\phi }{C^2_\varsigma}, \frac{40}{9} \right\} \log\left(\frac{4mdH}{\delta_2}\right)
\end{flalign*}
and combining \cref{eq: 142} and \cref{eq: 143}, we have
\begin{flalign}
\lambda_{\min}({\rm\Lambda}_h({w}_0)/N) &= \lambda_{\min}(\overline{m}({w}_0) + \widehat{m}({w}_0)/N + \lambda_1/N \cdot I_{2md})\nonumber\\
&\geq \lambda_{\min}(\overline{m}({w}_0)) - \ltwo{\widehat{m}({w}_0)/N}\nonumber\\
&\geq C_\varsigma - C^2_\phi H \sqrt{\frac{10}{N} \log\left(\frac{4mdH}{\delta_2}\right)}\nonumber\\
&\geq C_\varsigma/2.\nonumber
\end{flalign}
Hence, the following holds with probability $1-\delta_2/H$ with respect to randomness of $\mcd$
\begin{flalign}
\ltwo{{\rm\Lambda}_h({w}_0)^{-1}}\leq (N\cdot \lambda_{\min}({\rm\Lambda}_h({w}_0)/N))^{-1}\leq \frac{2}{NC_\varsigma}.\label{eq: 144}
\end{flalign}
Taking union bound of \cref{eq: 144} over $[H]$, we have the following holds for all $x\in\mcx$ and $h\in[H]$ with probability $1-\delta_2$
\begin{flalign}
b_{v,h}(x,{w}_0) &= \sqrt{{\phi}_h(x,{w}_0)^\top {\rm\Lambda}_h^{-1}({w}_0) {\phi}_h(x,{w}_0)}\nonumber\\
&\leq \ltwo{{\phi}_h(x,{w}_0)}\cdot \ltwo{{\rm\Lambda}_h({w}_0)^{-1}}^{1/2} \leq \frac{\sqrt{2} C_\phi}{\sqrt{C_\varsigma}\sqrt{N}},\label{eq: 145}
\end{flalign}

where we use the fact that $ \ltwo{{\phi}(x^\tau_h,\theta_0)}\leq C_\phi$. Substituting \cref{eq: 145} into \cref{eq: 138}, we have
\begin{flalign}
\sum_{h=1}^{H}b_{v,h}(x,\widehat{w}) \leq  \frac{\sqrt{2} H C_\phi}{\sqrt{C_\varsigma}\sqrt{N}} + \mathcal{O}\left(   \frac{ H^{7/6} N^{1/12} (\log m)^{1/4}}{m^{1/12}} \right).\label{eq: 146}
\end{flalign}
Finally, letting $\delta_1=N^{-2}H^{-4}/2$ and $\delta_2=N^{-2}H^{-4}/2$ and combining \cref{eq: 136} and \cref{eq: 146}, we have the following holds with probability $1-N^{-2}H^{-4}$
\begin{flalign}
&\beta_1\cdot\sum_{h=1}^{H}b_{r,h}(x,\widehat{\rm\Theta}) + \beta_2\cdot \sum_{h=1}^{H}b_{v,h}(x,\widehat{w})  \nonumber\\
&\quad \leq  \left( \frac{\beta_1}{\sqrt{C_\sigma}} + \frac{\beta_2}{\sqrt{C_\varsigma}} \right)\frac{\sqrt{2} H C_\phi}{\sqrt{N}} + \max\{\beta_1 H^{5/3},\beta_2 H^{7/6}\}\cdot \mathcal{O}\left(   \frac{ N^{1/12} (\log m)^{1/4}}{m^{1/12}} \right),\nonumber
\end{flalign}
which completes the proof.

\section{Proof of \Cref{lemma11}}\label{pflinearmdp}
Similarly to the proof of \Cref{lemma2}, we first bound the uncertainty of the estimated reward $\widehat{R}_h(\cdot)$ in \cref{eq: 173} and then bound the uncertainty of the estimated transition value function $(\widehat{\mP}_h\widehat{V}_{h+1})(\cdot)$ in \cref{eq: 175}.

\subsection{Uncertainty of Estimated Reward $\widehat{R}_h(\cdot)$}\label{subsc: linear1}
Following steps similar to those in the proof of Lemma B.1 in \cite{jin2021pessimism}, we can obtain
\begin{flalign}
	\ltwo{{\rm\Theta}^*}\leq H\sqrt{dH}\quad\text{and}\quad  \ltwo{\widehat{\rm\Theta}}\leq H\sqrt{dHN/\lambda_1}.\label{eq: 181}
\end{flalign} 
For simplicity, we denote $r(\tau) = \sum_{h\in[H]} r(x^\tau_h)$, $R(\tau) = \sum_{h\in[H]} R(x^\tau_h)$ and $\varepsilon(\tau) = R(\tau) - r(\tau)$. Consider the estimation error $R_h(\cdot) - \widehat{R}_h(\cdot)$. We have
\begin{flalign}
	&R_h(x) - \widehat{R}_h(x) \nonumber\\
	&= \langle \phi(x), \theta^*_h - \widehat{\theta}_h \rangle\nonumber\\
	&=\langle {\rm\Phi}_h(x), {\rm\Theta}^* - \widehat{\rm\Theta} \rangle\nonumber\\
	&=\langle {\rm\Phi}_h(x), {\rm\Theta}^*\rangle - {\rm\Phi}_h(x)^\top {\rm\Sigma}^{-1}\left( \sum_{\tau\in\mcd} {\rm\Phi}(\tau) r(\tau) \right)\nonumber\\
	&=\langle {\rm\Phi}_h(x), {\rm\Theta}^*\rangle - {\rm\Phi}_h(x)^\top {\rm\Sigma}^{-1}\left( \sum_{\tau\in\mcd} {\rm\Phi}(\tau) {\rm\Phi}(\tau)^\top {\rm\Theta}^* \right) + {\rm\Phi}_h(x)^\top {\rm\Sigma}^{-1}\left( \sum_{\tau\in\mcd} {\rm\Phi}(\tau) \varepsilon(\tau) \right)\nonumber\\
	&=\langle {\rm\Phi}_h(x), {\rm\Theta}^*\rangle - {\rm\Phi}_h(x)^\top {\rm\Sigma}^{-1}\left( {\rm\Sigma} - \lambda_1 \cdot I_{dH} \right){\rm\Theta}^* + {\rm\Phi}_h(x)^\top {\rm\Sigma}^{-1}\left( \sum_{\tau\in\mcd} {\rm\Phi}(\tau) \varepsilon(\tau) \right)\nonumber\\
	&= -\lambda_1\cdot {\rm\Phi}_h(x)^\top {\rm\Sigma}^{-1} {\rm\Theta}^* + {\rm\Phi}_h(x)^\top {\rm\Sigma}^{-1}\left( \sum_{\tau\in\mcd} {\rm\Phi}(\tau) \varepsilon(\tau) \right).\label{eq: 179}
\end{flalign}
Applying the triangle inequality to \cref{eq: 179}, we have
\begin{flalign}
	\lone{R_h(x) - \widehat{R}_h(x)}\leq \underbrace{\lambda_1\cdot \lone{{\rm\Phi}_h(x)^\top {\rm\Sigma}^{-1} {\rm\Theta}^*}}_{(i)} + \underbrace{\lone{{\rm\Phi}_h(x)^\top {\rm\Sigma}^{-1}\left( \sum_{\tau\in\mcd} {\rm\Phi}(\tau) \varepsilon(\tau) \right)}}_{(ii)}.\label{eq: 180}
\end{flalign}
We then proceed to bound $(i)$ and $(ii)$ separately. For $(i)$, we have
\begin{flalign}
	(i) = \lambda_1\cdot \lone{{\rm\Phi}_h(x)^\top {\rm\Sigma}^{-1/2}  {\rm\Sigma}^{-1/2} {\rm\Theta}^*} \leq \lambda_1 \lsigmab{{\rm\Phi}_h(x)}\lsigmab{{\rm\Theta}^*}\overset{(i.1)}{\leq} H\sqrt{dH\lambda_1}\lsigmab{{\rm\Phi}_h(x)},\label{eq: 182}
\end{flalign}
where $(i.1)$ follows from \cref{eq: 181} and the following inequality
\begin{flalign*}
	\lsigmab{{\rm\Theta}^*} = \sqrt{{\rm\Theta}^{*\top} {\rm\Sigma}^{-1} {\rm\Theta}^*} \leq \ltwo{{\rm\Sigma}^{-1}}^{1/2}\ltwo{ {\rm\Theta}^* } \leq H\sqrt{dH/\lambda_1}.
\end{flalign*}
For $(ii)$, we have
\begin{flalign}
	(ii)&= \lone{{\rm\Phi}_h(x)^\top {\rm\Sigma}^{-1/2} {\rm\Sigma}^{-1/2} \left( \sum_{\tau\in\mcd} {\rm\Phi}(\tau) \varepsilon(\tau) \right)}\leq \underbrace{\lsigmab{\sum_{\tau\in\mcd} {\rm\Phi}(\tau) \varepsilon(\tau)}}_{(iii)} \cdot  \lsigmab{{\rm\Phi}_h(x)}.\label{eq: 183}
\end{flalign}
Following steps similar to those in \cref{eq: 72} and Lemma B.2 in \cite{jin2021pessimism}, we have the following holds with probability at least $1-\delta$
\begin{flalign*}
	(iii)\leq H\cdot\sqrt{2\log(1/\delta) + dH\cdot \log(1+N/\lambda_1)},
\end{flalign*}
which implies
\begin{flalign}
	(ii) \leq H\sqrt{2\log(1/\delta) + dH\cdot \log(1+N/\lambda_1)} \cdot \lsigmab{{\rm\Phi}_h(x)}.\label{eq: 184}
\end{flalign}
Recalling that $b_{r,h}(x) = \lsigmab{{\rm\Phi}_h(x)}$ and substituting \cref{eq: 184} and \cref{eq: 182} into \cref{eq: 180}, we can obtain
\begin{flalign}
		\lone{R_h(x) - \widehat{R}_h(x)} \leq R_{\beta_1}\cdot b_{r,h}(x),\label{eq: 185}
\end{flalign}
where $R_{\beta_1}$ is an absolute constant satisfying
\begin{flalign*}
	R_{\beta_1} \geq H \left(\sqrt{dH\lambda_1} + \sqrt{2\log(1/\delta) + dH\cdot \log(1+N/\lambda_1)}\right).
\end{flalign*}
Letting $\lambda_1=1$ and $C_{\beta_1}>0$ be a sufficiently large constant, we can verify that $R_{\beta_1} = C_{\beta_1} H\sqrt{dH\log(N/\delta)}$ satisfies the above inequality.

\subsection{Uncertainty of Estimated Transition Value Function $(\widehat{\mP}_h\widehat{V}_{h+1})(\cdot)$}\label{subsc: linear2}
Following steps similar to those in the proof of Lemma B.1 in \cite{jin2021pessimism}, we can obtain
\begin{flalign}
\ltwo{{w}^*}\leq H\sqrt{d}\quad\text{and}\quad  \ltwo{\widehat{w}}\leq H\sqrt{dN/\lambda_2}.\label{eq: 186}
\end{flalign} 
Consider the estimation error $({\mP}_h\widehat{V}_{h+1})(\cdot) - (\widehat{\mP}_h\widehat{V}_{h+1})(\cdot)$. For simplicity, we define $\varepsilon_v(x) = ({\mP}_h\widehat{V}_{h+1})(x) - (\widehat{\mP}_h\widehat{V}_{h+1})(x)$ for all $x\in\mcx$. Following steps similar to those in \cref{eq: 179}, we can obtain
\begin{flalign}
	&({\mP}_h\widehat{V}_{h+1})(x) - (\widehat{\mP}_h\widehat{V}_{h+1})(x)\nonumber\\
	&\quad\leq -\lambda_2\cdot {\phi}(x)^\top {\rm\Lambda}_h^{-1} {w}^*_h + {\phi}(x)^\top {\rm\Lambda}^{-1}_h\left( \sum_{\tau\in\mcd} {\phi}(x^\tau_h) \varepsilon_v(x^\tau_h) \right).\label{eq: 187}
\end{flalign}
Applying the triangle inequality to \cref{eq: 187}, we have
\begin{flalign}
	&\lone{({\mP}_h\widehat{V}_{h+1})(x) - (\widehat{\mP}_h\widehat{V}_{h+1})(x)}\nonumber\\
	&\quad\leq \underbrace{\lambda_2\cdot \lone{{\phi}(x)^\top {\rm\Lambda}^{-1}_h {w}^*_h}}_{(i)} + \underbrace{\lone{{\phi}(x)^\top {\rm\Lambda}^{-1}_h\left( \sum_{\tau\in\mcd} {\phi}(x^\tau_h) \varepsilon_v(x^\tau_h) \right)}}_{(ii)}.\label{eq: 188}
\end{flalign}
Following steps similar to those in \cref{eq: 180}, we can obtain
\begin{flalign}
	(i) \leq H\sqrt{d\lambda_2} \llambdabh{\phi(x)}.
\end{flalign}
For $(ii)$, we have
\begin{flalign}
	(ii) &= \lone{{\phi}(x)^\top {\rm\Lambda}_h^{-1/2}{\rm\Lambda}_h^{-1/2}\left( \sum_{\tau\in\mcd} {\phi}(x^\tau_h) \varepsilon_v(x^\tau_h) \right)}\leq \underbrace{\llambdabh{\sum_{\tau\in\mcd} {\phi}(x^\tau_h) \varepsilon_v(x^\tau_h)}}_{(iii)} \llambdabh{\phi(x)}. \label{eq: 189}
\end{flalign}
We then proceed to upper bound the term $(iii)$. Following steps similar to those in \cref{eq: 116} and Lemma B.2 in \cite{jin2021pessimism}, we have the following holds with probability at least $1-\delta$
\begin{flalign}
	(iii) &\leq 2H\cdot  \sqrt{\log(H \cdot \mcN^v_{\epsilon,h}/\delta) + d\cdot\log(1+N/\lambda_2) + 8\epsilon^2N/\lambda_2} \cdot  \llambdabh{\phi(x)},\nonumber\\
	&\leq R_{\beta_2} \llambdabh{\phi(x)}\label{eq: 190}
\end{flalign}
where $R_{\beta_2}$ is an absolute constant satisfying
\begin{flalign}
	R_{\beta_2} \geq 2H\cdot  \sqrt{\log(H \cdot \mcN^v_{\epsilon,h}/\delta) + d\cdot\log(1+N/\lambda_2) + 8\epsilon^2N^2/\lambda_2},\label{eq: 194}
\end{flalign}
and $\mcN^v_{\epsilon,h}$ is the cardinality of the following function class
\begin{flalign*}
\mcv_h(x,&\,R_\theta, R_w, R_{\beta_1}, R_{\beta_2}, \lambda_1, \lambda_2 )=\{ \max_{a\in\mca}\{\overline{Q}_h(s,a)\}:\mcs\rightarrow [0,H]\,\, \nonumber\\
&\text{with}\,\, \ltwo{\rm\Theta}\leq R_\theta, \ltwo{w}\leq R_w, \beta_1\in[0, R_{\beta_1}], \beta_2\in [0, R_{\beta_2}], \ltwo{\rm\Sigma}\geq \lambda_1, \ltwo{\rm\Lambda}\geq \lambda_2  \},
\end{flalign*}
where $R_\theta = H\sqrt{dHN/\lambda_1}$, $R_w = H\sqrt{dN/\lambda_2}$ and 
\begin{flalign*}
\overline{Q}_h&(x ) = \min\{ \langle {\rm\Phi}_h(x), {\rm\Theta} \rangle+  \langle \phi(x), w \rangle \nonumber\\
&- \beta_1\cdot \sqrt{{\rm\Phi}_h(x)^\top {\rm\Sigma}^{-1}{\rm\Phi}_h(x) } - \beta_2\cdot \sqrt{\phi(x)^\top{\rm\Lambda}^{-1}\phi(x)  } , H-h+1 \}^{+}.
\end{flalign*}
Then, following steps similar to those in \Cref{pfcorollary1}, we have
\begin{flalign}
&\lone{\max_{a\in\mca}\{\overline{Q}_h(s,a,\theta,w,\beta_1,\beta_2,{\rm\Sigma}, {\rm\Lambda} )\} - \max_{a\in\mca}\{\overline{Q}_h(s,a,\theta^\prime,w^\prime,\beta_1^\prime,\beta_2^\prime,{\rm\Sigma}^\prime, {\rm\Lambda}^\prime )\}}\nonumber\\
&\leq \max_{a\in\mca}\lone{\langle {\rm\Phi}_h(x), {\rm\Theta} - {\rm\Theta}^\prime  \rangle} +   \max_{a\in\mca}\lone{\langle \phi(x), w- w^\prime \rangle } + \frac{1}{\sqrt{\lambda_1}} \lone{\beta_1-\beta^\prime_1} + \frac{1}{\sqrt{\lambda_2}}\lone{\beta_2-\beta^\prime_2}\nonumber\\
&\quad + R_{\beta_1} \max_{a\in\mca} \lone{ \lsigmab{{\rm\Phi}_h(x)} - \lsigmabp{{\rm\Phi}_h(x)} } + R_{\beta_2} \max_{a\in\mca} \lone{ \llambdab{\phi(x)} - \llambdabp{\phi(x)}  }\nonumber\\
&\overset{(i)}{\leq} \ltwo{{\rm\Theta} - {\rm\Theta}^\prime} + \ltwo{w-w^\prime} + \lone{\beta_1-\beta_1^\prime} + \lone{\beta_2-\beta^\prime_2} \nonumber\\
&\quad + R_{\beta_1}\sqrt{\lF{ {\rm\Sigma}^{-1} - {\rm\Sigma}^{\prime-1} }} + R_{\beta_2}\sqrt{\lF{ {\rm\Lambda}^{-1} - {\rm\Lambda}^{\prime-1} }},\label{eq: 191}
\end{flalign}
where $(i)$ follows from the fact that $\ltwo{\phi(x)}\leq 1$ and $\lambda_1,\lambda_2\geq 1$.
Following arguments similar to those used to obtain \cref{eq: 160} and applying Lemma 8.6 in \cite{agarwal2019reinforcement}, we have
\begin{flalign}
\log \mcN^v_{\epsilon,h}& \overset{(i)}{\leq}  \mcN(\epsilon/6, \mR^{dH}, R_\theta)+ \mcN(\epsilon/6, \mR^{d}, R_w) + \mcN(\epsilon/6, R_{\beta_1}) + \mcN(\epsilon/6, R_{\beta_2})\nonumber\\
&\quad + \mcN(\epsilon^2/(36R^2_{\beta_1}), \mcF, \sqrt{dH}/\lambda_1) + \mcN(\epsilon^2/(36R^2_{\beta_2}), \mcF, \sqrt{d}/\lambda_2)\nonumber\\
&\overset{(ii)}{\leq} dH\log(1+12 R_\theta/\epsilon) + d\log(1+12 R_w/\epsilon) + \log(1+12R_{\beta_1}/\epsilon) + \log(1+12R_{\beta_2}/\epsilon) \nonumber\\
&\quad + d^2H^2\log(1 + 36R_{\beta_1}^2\sqrt{dH}/\epsilon^2) + d^2\log(1 + 36R_{\beta_2}^2\sqrt{d}/\epsilon^2)\nonumber\\
&\overset{(iii)}{\leq}dH\log(1+12 H\sqrt{dHN}/\epsilon) + d\log(1+12 H\sqrt{dN}/\epsilon) \nonumber\\
&\quad + \log(1+12 C_{\beta_1} H\sqrt{dH\log(N/\delta)}/\epsilon) + \log(1+12R_{\beta_2}/\epsilon) \nonumber\\
&\quad + d^2H^2\log(1 + 36 C^2_{\beta_1} dH^3 \sqrt{dH} \log(N/\delta)/\epsilon^2) + d^2\log(1 + 36R_{\beta_2}^2\sqrt{d}/\epsilon^2)\nonumber\\
&\overset{(iv)}{\lesssim} C_1 d^2H^2\log( d^{3/2}H^{7/2} N^{1/2}/\epsilon^2 ) + C_2 d^2\log(R^2_{\beta_2}\sqrt{d}/\epsilon^2),\label{eq: 192}
\end{flalign}
where in $(i)$ we use $\mcN(\epsilon, \mR^{d}, B)$ to denote the $\epsilon$-covering of ball with radius $B$ in the space $\mR^d$, $\mcN(\epsilon, B)$ to denote the $\epsilon$-covering of interval $[0, B]$, and $\mcN(\epsilon, \mcF, B)$ to denote the $\epsilon$-covering of the function class $\mcF = \{ M: \lF{M}\leq B  \}$, $(ii)$ follows from Lemma. 8.6 in \cite{agarwal2019reinforcement}, $(iii)$ follows from the definition of $R_\theta$, $R_w$ and $R_{\beta_1}$, and in $(iv)$ we let $C_1$ and $C_2$ be sufficiently large and waive the $\log(\log(\cdot))$ term. 

Substituting \cref{eq: 192} into \cref{eq: 190}, we can obtain
\begin{flalign}
	&2H\cdot  \sqrt{\log(H \cdot \mcN^v_{\epsilon,h}/\delta) + d\cdot\log(1+N/\lambda_2) + 8\epsilon^2N/\lambda_2}\nonumber\\
	&\quad\leq 2H\cdot \left( \sqrt{\log(H/\delta)} +  \sqrt{\log\mcN^v_{\epsilon,h}} + \sqrt{d\cdot\log(1+N)} + \sqrt{8\epsilon^2N^2} \right) \nonumber\\
	&\quad\leq 2H\cdot \Big( \sqrt{\log(H/\delta)} +  \sqrt{C_1 d^2H^2\log( d^{3/2}H^{7/2} N^{1/2}/\epsilon^2 )} + \sqrt{C_2 d^2\log(R^2_{\beta_2}\sqrt{d}/\epsilon^2)} \nonumber\\
	&\qquad + \sqrt{d\cdot\log(1+N)} + \sqrt{8\epsilon^2N^2} \Big).\label{eq: 193}
\end{flalign}
Letting $\epsilon = (dH)^{1/4}/N$, we can see that when $R_{\beta_2} = C_{\beta_2} dH^2\sqrt{\log(dH^3N^{5/2}/\delta)}$, where $C_{\beta_2}$ is a sufficiently large constant, we have
\begin{flalign}
	R_{\beta_2} \geq \text{R.H.S of \cref{eq: 193}},\nonumber
\end{flalign}
which satisfies the inequality in \cref{eq: 194}.

\subsection{Upper and Lower Bounds on Evaluation Error $\delta_h(\cdot)$}
Using the properties that we obtained from \Cref{subsc: linear1} \& \ref{subsc: linear2} and following steps similar to those in \Cref{uplowerdelta1}, we can obtain
\begin{flalign*}
0 \leq \delta_h(x) \leq  2\left[\beta_1\cdot b_{r,h}(x) + \beta_2\cdot b_{v,h}(x)\right], \quad\forall x\in\mcx,\quad \forall h\in[H],
\end{flalign*}
where $\beta_1=R_{\beta_1} = C_{\beta_1} H\sqrt{dH\log(N/\delta)}$ and $R_{\beta_2} = C_{\beta_2} dH^2\sqrt{\log(dH^3N^{5/2}/\delta)}$.

\section{Supporting Lemmas for Overparameterized Neural Networks}\label{sc: supplemma}
The following lemma shows that an infinite-width neural network can be well-approximated by a finite-width neural network.
\begin{lemma}[Approximation by Finite Sum]\label{lemma3}
	Let $g(x)=\int_{\mR^d}\sigma^\prime(w^\top x)x^\top \ell(w)dp(w)\in\mathcal{F}_{g_1,g_2}$. Then for any $\epsilon>0$, with probability at least $1-\epsilon$ over $w_1,\cdots,w_m$ drawn i.i.d. from $N(0,I_d/d)$, there exist $\ell_1,\cdots,\ell_m$ where $\ell_i\in\mR^d$ and $\ltwo{\ell_i}\leq g_2/\sqrt{dm}$ for all $i\in[m]$ such that the function $\widehat{g}(x)=(1/\sqrt{m})\sum_{i=1}^{m}\sigma^\prime(w_i^\top x)x^\top\ell_i$ satisfies
	\begin{flalign*}
	\sup_x\lone{g(x)-\widehat{g}(x)}\leq \frac{2L_\sigma g_2}{\sqrt{m}}+ \frac{\sqrt{2}C_\sigma^2g^2_2}{\sqrt{m}}\sqrt{\log\left( \frac{1}{\delta} \right)}
	\end{flalign*}
	with probability at least $1-\delta$.
\end{lemma}
\begin{proof}
	The proof of \Cref{lemma3} follows from the proof of Proposition C.1 in \cite{gao2019convergence} with some modifications. 
	In \Cref{lemma3} we consider a different distribution of $w_i$ and upper bound on $\ltwo{\ell_i}$ from those in \cite{gao2019convergence}. 
	First, we define the following random variable
	\begin{flalign*}
	a(w_1,\cdots,w_m) = \sup_x\lone{g(x)-\widehat{g}(x)}.
	\end{flalign*}
	Then, we proceed to show that $a(\cdot)$ is robust to the perturbation of one of its arguments. Let $\ell_i=\ell(w_i)/(\sqrt{dm}p(w_i))$. For $w_1,\cdots,w_m$ and $\tilde{w}_i$ ($1\leq i\leq m$), we have
	\begin{flalign*}
	&\quad\lone{a(w_1,\cdots,w_m) - a(w_1,\cdots,\tilde{w}_i,\cdots,w_m)}\nonumber\\
	&=\frac{1}{\sqrt{dm}}\lone{\sigma^\prime(w_i^\top x)x^\top\ell_i - \sigma^\prime(\tilde{w}_i^\top x)x^\top\ell_i}\nonumber\\
	&=\frac{1}{\sqrt{d}m}\lone{ \frac{\sigma^\prime(w_i^\top x)x^\top\ell(w_i)}{p(w_i)} - \frac{\sigma^\prime(\tilde{w}_i^\top x)x^\top\ell(\tilde{w}_i)}{p(\tilde{w}_i)} }\nonumber\\
	&\leq \frac{1}{\sqrt{d}m}\sup_{x\in\mcx}\lone{ \frac{\sigma^\prime(w_i^\top x)x^\top\ell(w_i)}{p(w_i)} - \frac{\sigma^\prime(\tilde{w}_i^\top x)x^\top\ell(\tilde{w}_i)}{p(\tilde{w}_i)} }\nonumber\\
	&\leq \frac{1}{\sqrt{d}m}\sup_{x\in\mcx} \left(\lone{ \frac{\sigma^\prime(w_i^\top x)x^\top\ell(w_i)}{p(w_i)}} + \lone{\frac{\sigma^\prime(\tilde{w}_i^\top x)x^\top\ell(\tilde{w}_i)}{p(\tilde{w}_i)} }\right)\nonumber\\
	&\leq \frac{1}{\sqrt{d}m}\sup_{x\in\mcx} \left( \ltwo{\sigma^\prime(w_i^\top x)x} \ltwo{ \frac{\ell(w_i)}{p(w_i)}} +  \ltwo{\sigma^\prime(\tilde{w}_i^\top x)x}\ltwo{ \frac{\ell(\tilde{w}_i)}{p(\tilde{w}_i)}} \right)\nonumber\\
	&\leq \frac{2C_\sigma g_2}{\sqrt{d}m} = \zeta,
	\end{flalign*}
	where the last inequality follows from the facts that $\ltwo{x}=1$, $\lone{\sigma^\prime(\cdot)}\leq C_\sigma$ and $\sup_x\ltwo{\ell(w)/p(w)}\leq g_2$. Then, we proceed to bound the expectation of $a(\cdot)$. Note that our choice of $\ell_i$ ensures that $\sqrt{d}\cdot \mE_{w_1,\cdots,w_m}\widehat{g}_h(\cdot) = g(\cdot)$. By symmetrization, we have
	\begin{flalign*}
	\mE a&=\sqrt{d}\cdot \mE\sup_{x\in\mcx}\lone{\widehat{g}(x)-\mE\widehat{g}(x)}\nonumber\\
	&\leq \frac{2\sqrt{d}}{\sqrt{m}}\cdot \mE_{w,\varepsilon}\sup_{x\in\mcx}\lone{\sum_{i=1}^{m}\varepsilon_i \sigma^\prime(w_i^\top x)x^\top\ell_i},
	\end{flalign*}
	where $\{\varepsilon_i\}_{i\in[m]}$ are a sequence of Rademacher random variables. Since $\lone{x^\top\ell_i}\leq \ltwo{\ell_i}\leq g_2/\sqrt{m}$ and $\sigma^\prime(\cdot)$ is $L_\sigma$-Lipschitz, we have that the function $b(\cdot) = \sigma^\prime(\cdot)x^\top\ell_i$ is $(L_\sigma g_2/\sqrt{m})$-Lipschitz. We then proceed as follows
	\begin{flalign*}
	\mE a&\leq \frac{2\sqrt{d}}{\sqrt{m}}\cdot \mE_{w,\varepsilon}\sup_{x\in\mcx}\lone{\sum_{i=1}^{m}\varepsilon_i \sigma^\prime(w_i^\top x)x^\top\ell_i}\nonumber\\
	&\overset{(i)}{\leq}  \frac{2\sqrt{d}L_\sigma g_2}{m} \cdot \mE_{w,\varepsilon}\sup_{x\in\mcx} \lone{ \left(\sum_{i=1}^{m} \varepsilon_i w_i\right)^\top x}\nonumber\\
	&\overset{(ii)}{\leq}  \frac{2\sqrt{d}L_\sigma g_2}{m} \cdot \mE_{w} \ltwo{\sum_{i=1}^{m} \varepsilon_i w_i}\nonumber\\
	&\overset{(iii)}{\leq}  \frac{2\sqrt{d}L_\sigma g_2}{\sqrt{m}} \cdot \sqrt{\mE_{w\sim N(0,I_d/d)} \ltwo{w}^2}\nonumber\\
	& = \frac{2L_\sigma g_2}{\sqrt{m}},
	\end{flalign*}
	where $(i)$ follows from Talagrand's Lemma (Lemma 5.7) in \cite{mohri2018foundations}, $(ii)$ follows from the fact that $\ltwo{x}=1$ for all $x\in\mcx$ and Cauchy-Schwartz inequality and $(iii)$ follows from Jensen's inequality. Then, applying McDiarmid's inequality, we can obtain
	\begin{flalign*}
	{\rm P}\left(a \geq \frac{2L_\sigma g_2}{\sqrt{m}}+ \epsilon\right)\leq {\rm P}(a \geq \mE a+ \epsilon)\leq \exp\left( -\frac{2\epsilon^2}{m\zeta^2} \right) = \exp\left( -\frac{m\epsilon^2}{2C^2_\sigma g^2_2} \right).
	\end{flalign*}
	Letting $\epsilon=\frac{\sqrt{2}C_\sigma^2g^2_2}{\sqrt{m}}\sqrt{\log\left( \frac{1}{\delta} \right)}$, we have
	\begin{flalign*}
	{\rm P}\left(a \geq \frac{2L_\sigma g_2}{\sqrt{m}}+ \frac{\sqrt{2}C_\sigma^2g^2_2}{\sqrt{m}}\sqrt{\log\left( \frac{1}{\delta} \right)} \right) \leq \delta,
	\end{flalign*}
	which completes the proof.
\end{proof}

The following lemma bounds the perturbed gradient and value of local linearization of overparameterized neural networks around the initialization, which is provided as Lemma C.2 in \cite{yang2020function}.
\begin{lemma}\label{lemma4}
	Consider the overparameterized neural network defined in \Cref{subsc: overpnn}. Consider any fixed input $x\in\mcx$. Let $R\leq c\sqrt{m}/(\log m)^3$ for some sufficiently small constant $c$. Then, with probability at least $1-m^{-2}$ over the random initialization, we have for any $w\in \mcB(w_0, R)$, where $\mcB(w_0, R)$ denotes the Euclidean ball centred at $w_0$ with radius $R$, the followings hold
	\begin{flalign}
	\ltwo{\phi(x,w)} &\leq C_\phi,\label{eq: 13}\\
	\ltwo{\phi(x,w) - \phi(x,w_0)} &\leq \mathcal{O}\left( C_\phi \left( \frac{R}{\sqrt{m}}\right)^{1/3}\sqrt{\log m} \right),\label{eq: 14}\\
	\lone{f(x,w) - \langle \phi(x,w_0)^\top(w-w_0) \rangle} &\leq \mathcal{O}\left(C_\phi \left(\frac{R^4}{\sqrt{m}}\right)^{1/3}\sqrt{\log m}\right),\label{eq: 15}
	\end{flalign}
	where $C_\phi= \mathcal{O}(1)$ is a constant independent from $m$ and $d$.
\end{lemma}
\begin{proof}
	Please see Lemma C.2 in \cite{yang2020function} for a detailed proof, which is based on Lemma F.1, F.2 in \cite{cai2019neural}, Lemma A.5, A.6 in \cite{gao2019convergence} and Theorem 1 in \cite{allen2019convergence}.
\end{proof}

\section{Supporting Lemmas for RKHS}\label{sc: rkhs}
In this section, we provide some useful lemmas for general RKHS. Consider a variable space $\mcx$. Given a mapping $\phi(\cdot):\mcx\rightarrow\mR^d$, we can assign a feature vector $\phi(x)\in\mR^d$ for each $x\in\mcx$. We further define a kernel function $K(\cdot,\cdot):\mcx\times\mcx\rightarrow\mR$ as $K(x,x^\prime)=\phi(x)^\top\phi(x^\prime)$ for any $x,x^\prime\in\mcx$. Let $\mcH$ be a RKHS defined on $\mcx$ with the kernel function $K(\cdot,\cdot)$. Let $\langle \cdot,\cdot \rangle_{\mcH}:\mcH\times\mcH\rightarrow \mR$ and $\lH{\cdot}:\mcH\rightarrow \mR$ denote the inner product and RKHS norm on $\mcH$, respectively. Since $\mcH$ is a RKHS, there exists a feature mapping $\psi(\cdot):\mcx\rightarrow\mcH$, such that $f(x) = \langle f(\cdot), \psi(x) \rangle_{\mcH}$ for all $f\in\mcH$ and all $x\in\mcx$. Moreover, for any $x,x^\prime\in\mcx$ we have $K(x,x^\prime) = \langle \psi(x), \psi(x^\prime) \rangle_{\mcH}$. Without loss of generality, we further assume $\ltwo{\phi(x)}\leq C_\phi$ and $\lH{\psi(x)}\leq C_\psi$ for all $x\in\mcx$. 

Let $\mcl^2(\mcx)$ be the space of square-integrable functions on $\mcx$ with respect to the Lebesgue measure and let $\langle\cdot,\cdot\rangle_{\mcl^2}$ be the inner product on $\mcl^2(\mcx)$. The kernel function $K(\cdot,\cdot)$ induces an integral operator $T_K:\mcl^2(\mcx)\rightarrow \mcl^2(\mcx)$ defined as
\begin{flalign}
	T_Kf(z) = \int_{\mcx} K(x,x^\prime)\cdot f(x^\prime) dx^\prime,\quad\forall f\in \mcl^2(\mcx).\label{eq: 147}
\end{flalign}

Consider the kernel function $K(\cdot,\cdot)$ of the RHKS $\mcH$. Let $\{ x_i\}_{i=1}^{\infty}\subset \mcx$ be a discrete time stochastic process that is adapted to a filtration $\{ \mcF_t \}_{i=0}^{\infty}$, i.e., $x_i$ is $\mcF_{i-1}$ measurable for all $i\geq 1$. We define the Gram matrix $K_N\in\mR^{N\times N}$ and function $k_N(\cdot):\mcx\rightarrow\mR^{N}$ as
\begin{flalign}
	K_N=[K(x_i,x_j)]_{i,j\in[N]}\in\mR^{N\times N},\,\, k_N(x)=[K(x_1,x),\cdots,K(x_N,x)]^\top\in\mR^{N}.\label{eq: 51}
\end{flalign}
Note that $K_N$ and $k_N(x)$ can also be expressed as
\begin{flalign*}
	K_N= {\rm\Phi}{\rm\Phi}^\top={\rm\Psi}{\rm\Psi}^\top\in\mR^{N\times N},\quad\text{and}\quad k_N(x)= {\rm\Phi}\phi(x) = {\rm\Psi}\psi(x) \in\mR^{N\times 1},
\end{flalign*}
where ${\rm\Phi}=[\phi(x_1),\cdots,\phi(x_N)]^\top\in\mR^{N\times d}$ and ${\rm\Psi}=[\psi(x_1),\cdots,\psi(x_N)]^\top\in\mR^{N\times \infty}$.
Given a regularization parameter $\lambda>1$, we define the matrix ${\rm\Omega}_N$ based on ${\rm\Phi}$ and an operator ${\rm\Upsilon}_N$ in RKHS $\mcH$ based on ${\rm\Psi}$ as
\begin{flalign}
	{\rm\Omega}_N={\rm\Phi}^\top{\rm\Phi}+ \lambda\cdot I_d,\quad\text{and}\quad {\rm\Upsilon}_N ={\rm\Psi}^\top{\rm\Psi} + \lambda\cdot I_\mcH. \label{eq: 52}
\end{flalign}
We next provide some fundamental properties for the RKHS $\mcH$.
\begin{lemma}\label{lemma5}
	For any $x\in\mcx$, considering $K_N$, $k_N(\cdot)$, ${\rm\Omega}_N$ and ${\rm\Upsilon}_N$ defined in \cref{eq: 51} and \cref{eq: 52}, we have the followings hold
	\begin{flalign}
		&{\rm\Phi}^\top(K_N+I_N)^{-1} = {\rm\Omega}_N^{-1}{\rm\Phi}^\top,\label{eq: 64}\\
		&{\rm\Psi}^\top (K_N + I_N)^{-1} = {\rm\Upsilon}_N^{-1}{\rm\Psi}^\top,\label{eq: 65}\\
		&\phi(x)^\top {\rm\Omega}^{-1}_N\phi(x) \overset{(i)}{=} \frac{1}{\lambda}\left[ K(x,x) - k_N(x)^\top (K_N + \lambda\cdot I_N)^{-1} k_N(x) \right] \overset{(ii)}{=} \psi(x)^\top {\rm\Upsilon}_N^{-1} \psi(x).\label{eq: 56}
	\end{flalign}
\end{lemma}
\begin{proof}
	The result in \Cref{lemma5} can be obtained from steps spread out in \cite{yang2020function}. We provide a detailed proof here for completeness. 
	
	We first proceed to prove \cref{eq: 64} and $(i)$ in \cref{eq: 56}. According to the definition of ${\rm\Sigma}_N$, we have
	\begin{flalign}
		{\rm\Omega}_N{\rm\Phi}^\top={\rm\Phi}^\top{\rm\Phi}{\rm\Phi}^\top+\lambda{\rm\Phi}^\top = {\rm\Phi}^\top({\rm\Phi}{\rm\Phi}^\top + \lambda I_N) = {\rm\Phi}^\top (K_N+I_N).\nonumber
	\end{flalign}
	Multiplying ${\rm\Omega}_N^{-1}$ on both sides of the above equality yields
	\begin{flalign*}
		{\rm\Phi}^\top = {\rm\Omega}_N^{-1}{\rm\Phi}^\top (K_N+I_N),
	\end{flalign*}
	which implies \cref{eq: 64} as follows
	\begin{flalign}
		{\rm\Phi}^\top(K_N+I_N)^{-1} = {\rm\Omega}_N^{-1}{\rm\Phi}^\top.\label{eq: 53}
	\end{flalign}
	We next proceed as follows
	\begin{flalign}
		\phi(x) &= {\rm\Omega}_N^{-1}{\rm\Omega}_N \phi(x)\nonumber\\
		&={\rm\Omega}_N^{-1} ({\rm\Phi}^\top{\rm\Phi} + \lambda\cdot I_d) \phi(x)\nonumber\\
		&=({\rm\Omega}_N^{-1} {\rm\Phi}^\top){\rm\Phi}\phi(x) + \lambda{\rm\Omega}_N^{-1}\phi(x)\nonumber\\
		&\overset{(i)}{=}{\rm\Phi}^\top(K_N+I_N)^{-1}{\rm\Phi}\phi(x) + \lambda{\rm\Omega}_N^{-1}\phi(x),\label{eq: 54}
	\end{flalign}
	where $(i)$ follows from \cref{eq: 53}. Taking inter product with $\phi(x)$ on both sides of \cref{eq: 54} yields
	\begin{flalign}
		K(x,x) &= \phi(x)^\top\phi(x) =  \phi(x)^\top{\rm\Phi}^\top(K_N+I_N)^{-1}{\rm\Phi}\phi(x) + \lambda\phi(x)^\top{\rm\Omega}_N^{-1}\phi(x)\nonumber\\
		&=k_N(x)^\top(K_N+I_N)^{-1}k_N(x)+ \lambda\phi(x)^\top{\rm\Omega}_N^{-1}\phi(x),\nonumber
	\end{flalign}
	which implies
	\begin{flalign}
		\phi(x)^\top{\rm\Omega}_N^{-1}\phi(x) = \frac{1}{\lambda}\left[ K(x,x) - k_N(x)^\top(K_N+I_N)^{-1}k_N(x) \right].\label{eq: 55}
	\end{flalign}
	We next proceed to prove \cref{eq: 65} and $(ii)$ in \cref{eq: 56}. According to the definition of ${\rm\Upsilon}_N$, we have
	\begin{flalign}
		{\rm\Upsilon}_N{\rm\Psi}^\top = {\rm\Psi}^\top{\rm\Psi}{\rm\Psi}^\top + \lambda {\rm\Psi}^\top = {\rm\Psi}^\top({\rm\Psi}{\rm\Psi}^\top + I_N) = {\rm\Psi}^\top(K_N + I_N).\label{eq: 60}
	\end{flalign}
	Multiplying ${\rm\Upsilon}_N^{-1}$ on both sides of the above equality yields
	\begin{flalign*}
		{\rm\Psi}^\top = {\rm\Upsilon}_N^{-1}{\rm\Psi}^\top(K_N + I_N),
	\end{flalign*}
	which further implies \cref{eq: 65} as follows
	\begin{flalign}
	{\rm\Psi}^\top (K_N + I_N)^{-1} = {\rm\Upsilon}_N^{-1}{\rm\Psi}^\top.\label{eq: 57}
	\end{flalign}
	We next proceed as follows
	\begin{flalign}
	\psi(x) &= {\rm\Upsilon}_N^{-1}{\rm\Upsilon}_N \psi(x)\nonumber\\
	&={\rm\Upsilon}_N^{-1} ({\rm\Psi}^\top{\rm\Psi} + \lambda\cdot I_\mcH) \psi(x)\nonumber\\
	&=({\rm\Upsilon}_N^{-1} {\rm\Psi}^\top){\rm\Psi}\psi(x) + \lambda{\rm\Upsilon}_N^{-1}\psi(x)\nonumber\\
	&\overset{(i)}{=}{\rm\Psi}^\top(K_N+I_N)^{-1}{\rm\Psi}\psi(x) + \lambda{\rm\Upsilon}_N^{-1}\psi(x),\label{eq: 58}
	\end{flalign}
	where $(i)$ follows from \cref{eq: 57}. Taking inter product with $\psi(x)$ on both sides of \cref{eq: 58} yields
	\begin{flalign}
	K(x,x) &= \langle \psi(x), \psi(x) \rangle_{\mcH} =  \psi(x)^\top{\rm\Psi}^\top(K_N+I_N)^{-1}{\rm\Psi}\psi(x) + \lambda\psi(x)^\top{\rm\Upsilon}_N^{-1}\psi(x)\nonumber\\
	&=k_N(x)^\top(K_N+I_N)^{-1}k_N(x)+ \lambda\psi(x)^\top{\rm\Upsilon}_N^{-1}\psi(x),\nonumber
	\end{flalign}
	which implies
	\begin{flalign}
		\psi(x)^\top{\rm\Upsilon}_N^{-1}\psi(x) = \frac{1}{\lambda}\left[ K(x,x) - k_N(x)^\top(K_N+I_N)^{-1}k_N(x) \right].\label{eq: 59}
	\end{flalign}
	Combining \cref{eq: 60} and \cref{eq: 59} completes the proof.
\end{proof}

The following two lemmas characterize the concentration property of self-normalized processes.

\begin{lemma}[Concentration of Self-Normalized Process in RKHS \cite{chowdhury2017kernelized}]\label{lemma6}
	Let $\{ \varepsilon_i \}_{i=1}^{\infty}$ be a real-valued stochastic process such that (i) $\epsilon_i\in\mcF_t$ and (ii) $\epsilon_i$ is zero-mean and $\sigma$-sub-Gaussian conditioned on $\mcF_{i-1}$ satisfying
	\begin{flalign}
	\mE\left[ \varepsilon_i|\mcF_{i-1} \right]=0,\qquad \mE\left[ e^{\kappa\varepsilon_i}\leq e^{\kappa^2\sigma^2/2} |\mcF_{i-1} \right],\qquad\forall \kappa\in\mR.\label{eq: 61}
	\end{flalign}
	Moreover, for any $t\geq2$, let $E_N=[\varepsilon_1,\cdots,\varepsilon_{N-1}]^\top\in\mR^{N-1}$.
	For any $\eta>0$ and any $\delta\in(0,1)$, with probability at least $1-\delta$, we have the following holds simultaneously for all $N\geq 1$:
	\begin{flalign*}
	E^\top_N \left[ (K_N + \eta\cdot I_{N-1})^{-1} + I_{N-1} \right]^{-1} E_N \leq \sigma^2\cdot\log\text{det}[(1+\eta)\cdot I_{N+1} + K_N] + 2\sigma^2\cdot\log(1/\delta).
	\end{flalign*}
	Moreover, if $K_N$ is positive definite for all $N\geq 2$ with probability one, then the above inequality also holds with $\eta=0$.
\end{lemma}

\begin{lemma}\label{lemma7}
	Let $\mathcal{G}\subset \{G: \mcx\rightarrow[0,C_g]\}$ be a class of bounded functions on $\mcx$. Let $\mcG_\epsilon\subset \mcG$ be the minimal $\epsilon$-cover of $\mcG$ such that $\mcN_\epsilon=\lone{\mcG_\epsilon}$. Then for any $\delta\in(0,1)$, with probability at least $1-\delta$, we have
	\begin{flalign}
	&\sup_{G\in\mathcal{G}}\lomega{\sum_{i=1}^{N} \phi(x_i)  \left( G(x_i) - \mE\left[G(x_i)|\mcF_{i-1}\right] \right) }^2 \nonumber\\
	&\quad\leq 2C^2_g  \log\det(I+K_N/\lambda) + 2C^2_g N(\lambda-1) + 4C^2_g\log(\mcN_\epsilon/\delta) + 8N^2C^2_\phi\epsilon^2/\lambda.\label{eq: 69}
	\end{flalign}
	Moreover, if $G(\cdot)$ does not depend on $\{x_i\}_{i\in[N]}$, we have
	\begin{flalign}
	&\lomega{\sum_{i=1}^{N} \phi(x_i)  \left( G(x_i) - \mE\left[G(x_i)|\mcF_{i-1}\right] \right) }^2 \nonumber\\
	&\quad\leq C^2_g  \log\det(I+K_N/\lambda) + C^2_g N(\lambda-1) + 2C^2_g\log(1/\delta).\label{eq: 70}
	\end{flalign}
\end{lemma}
\begin{proof}
	The proof is adapted but different from the proof of Lemma E.2 in \cite{yang2020function}. We first proceed to prove \cref{eq: 69} and will show that \cref{eq: 70} can be obtained as a by-product of proving \cref{eq: 69}. For any $G\in\mcG$, there exists a function $G^\prime$ in $\mcG_\epsilon$ such that $\sup_{x\in\mcx}\lone{G(x) - G^\prime(x)}\leq \epsilon$. Denote ${\rm\Delta}_G(x) = G(x) -G^\prime(x)$. We have the following holds
	\begin{flalign}
		&\lomega{\sum_{i=1}^{N} \phi(x_i)  \left( G(x_i) - \mE\left[G(x_i)|\mcF_{i-1}\right] \right) }^2\nonumber\\
		&\quad\leq 2\lomega{\sum_{i=1}^{N} \phi(x_i)  \left( G^\prime(x_i) - \mE\left[G^\prime(x_i)|\mcF_{i-1}\right] \right) }^2 + 2 \lomega{\sum_{i=1}^{N} \phi(x_i)  \left( {\rm\Delta}_G(x_i) - \mE\left[{\rm\Delta}_G(x_i)|\mcF_{i-1}\right] \right) }^2.\label{eq: 62}
	\end{flalign}
	For the second term on the right hand side of \cref{eq: 62}, we have
	\begin{flalign}
		\lomega{\sum_{i=1}^{N} \phi(x_i)  \left( {\rm\Delta}_G(x_i) - \mE\left[{\rm\Delta}_G(x_i)|\mcF_{i-1}\right] \right) }^2\leq N^2 C^2_\phi \cdot(2\epsilon)^2/\lambda = 4N^2 C^2_\phi\epsilon^2/\lambda.\label{eq: 63}
	\end{flalign}
	To bound the first term on the right hand side of \cref{eq: 62}, we apply \Cref{lemma4} to $G^\prime(x_i) - \mE\left[G^\prime(x_i)|\mcF_{i-1}\right]$. We fix $G^\prime\in\mcG$ and let $\varepsilon_i = G^\prime(x_i) - \mE\left[G^\prime(x_i)|\mcF_{i-1}\right]$ and $E_N=[\varepsilon_1,\cdots,\varepsilon_{N-1}]^\top\in\mR^{N-1}$. Using this notation, we have
	\begin{flalign}
		&\lomega{\sum_{i=1}^{N} \phi(x_i)  \left( G^\prime(x_i) - \mE\left[G^\prime(x_i)|\mcF_{i-1}\right] \right) }^2 = \lomega{\sum_{i=1}^{N} \phi(x_i)  \varepsilon_i }^2 = \lomega{{\rm\Phi}^\top E_N}\nonumber\\
		&=E_N^\top {\rm\Phi}{\rm\Omega}_N^{-1}{\rm\Phi}^\top E_N \overset{(i)}{=} E_N^\top {\rm\Phi}{\rm\Phi}^\top(K_N+\lambda I_N)^{-1}E_N =  E_N^\top K_N(K_N+\lambda I_N)^{-1}E_N\nonumber\\
		&\overset{(ii)}{\leq} E_N^\top (K_N + (\lambda-1)I_N)(K_N+\lambda I_N)^{-1}E_N\nonumber\\
		&=E_N^\top (K_N + (\lambda-1)I_N)[I_N + (K_N+(\lambda - 1) I_N)]^{-1}E_N\nonumber\\
		&=E_N^\top [(K_N+(\lambda - 1) I_N)^{-1} + I_N]E_N\label{eq: 66}
	\end{flalign}
	where $(i)$ follows from \cref{eq: 64} in \Cref{lemma5} and $(ii)$ follows from the fact that $\lambda>1$ and $K_N + \lambda I_N$ is positive definite. Note that each entry of $E_N$ is bounded by $C_g$ in absolute value. Applying \Cref{lemma6} to \cref{eq: 66} and taking a union bound over $\mcG_\epsilon$, for any $0<\delta<1$, we have the following holds with probability at least $1-\delta$
	\begin{flalign}
		&\sup_{G^\prime\in\mcG_\epsilon}\lomega{\sum_{i=1}^{N} \phi(x_i)  \left( G^\prime(x_i) - \mE\left[G^\prime(x_i)|\mcF_{i-1}\right] \right) }^2 \nonumber\\
		&\quad\leq C^2_g \log\text{det}[(1+\eta)I+K_N] + 2C^2_g\log(\mcN_\epsilon/\delta).\label{eq: 67}
	\end{flalign}
	Moreover, note that $(1+\eta)I + K_N=[I+(1+\eta)^{-1}K_N][(1+\eta)I]$, which implies
	\begin{flalign}
		\log\det[(1+\eta)I + K_N] &= \log\det[I+(1+\eta)^{-1}K_N] + N\log(1+\eta)\nonumber\\
		&\leq \log\det[I+(1+\eta)^{-1}K_N] + N\eta.\label{eq: 68}
	\end{flalign}
	Combining \cref{eq: 62}, \cref{eq: 63}, \cref{eq: 66}, \cref{eq: 67} and \cref{eq: 68} and letting $\eta=\lambda-1$, we have the following holds with probability $1-\delta$
	\begin{flalign*}
	&\lomega{\sum_{i=1}^{N} \phi(x_i)  \left( G(x_i) - \mE\left[G(x_i)|\mcF_{i-1}\right] \right) }^2\nonumber\\
	&\quad\leq 2C^2_g  \log\det(I+K_N/\lambda) + 2C^2_g N(\lambda-1) + 4C^2_g\log(\mcN_\epsilon/\delta) + 8N^2C^2_\phi\epsilon^2/\lambda,
	\end{flalign*}
	which completes the proof of \cref{eq: 69}. To prove \cref{eq: 70} we do not need to go through the "$\epsilon$-cover" argument since $G(\cdot)$ is independent from $\{x_i\}_{i\in N}$. We can directly apply \Cref{lemma6} and then follow steps similar to those in \cref{eq: 68} to obtain \cref{eq: 70}.
\end{proof}
For any integer $N$ and $\lambda>0$, we define the maximal information gain associated with the RKHS $\mathcal{H}$ as
\begin{flalign*}
	{\rm\Gamma}_K(N,\lambda) = \sup_{\mcd\subset \mcx}\{1/2\cdot \log\det(I_d + \lambda^{-1}\cdot K_N) \},
\end{flalign*}
where the supremum is taken over all discrete subset $\mcd$ of $\mcx$ with the cardinality no more than $N$.
\begin{lemma}[Finite Spectrum/Effective Dimension Property]\label{lemma10}
	Let $\{\sigma_j\}_{j\geq 1}$ be the eigenvalues of $T_K$ defined in \cref{eq: 147} in the descending order.
	Let $\lambda\in[c_1,c_2]$ with $c_1$ and $c_2$ being absolute constants. If $\sigma_j=0$ for all $j\geq D+1$, where $D$ is a positive integer. Then, we have ${\rm\Gamma}_K(N,\lambda) = C_K\cdot D\cdot \log N$,
	where $C_K$ is an absolute constant that depends on $C_1$, $C_2$, $c_1$, $c_2$ and $C_\phi$.
\end{lemma}
\begin{proof}
	See the proof of Lemma D.5 in \cite{yang2020function} for a detailed proof.
\end{proof}

\section{Other Useful Lemmas}
\begin{lemma}[Matrix Bernstein Inequality \cite{tropp2015introduction}]\label{lemma9}
	Suppose that $\{A_i\}_{i=1}^{N}$ are independent and centered random matrices in $\mR^{d_1\times d_2}$, that is, $\mE[A_i]=0$ for all $i\in[N]$. Also, suppose $\ltwo{A_i}\leq C_A$ for all $i\in[n]$. Let $Z = \sum_{i=1}^{N} A_i$ and 
	\begin{flalign*}
		v(Z) = \max\left\{ \ltwo{\mE\left[ ZZ^\top \right]}, \ltwo{\mE\left[ Z^\top Z \right]} \right\}.
	\end{flalign*}
	For all $\xi\geq 0$, we have
	\begin{flalign*}
		{\rm P}(\ltwo{Z}\geq \xi)\leq (d_1 + d_2)\cdot\exp\left( -  \frac{\xi^2/2}{v(Z) + C_A/3 \cdot \xi} \right).
	\end{flalign*}
\end{lemma}
\begin{proof}
	See Theorem 1.6.2 in \cite{tropp2015introduction} for a detailed proof.
\end{proof}